\documentclass{article}

\usepackage[main,final]{neurips_2026}
\usepackage[utf8]{inputenc}
\usepackage[T1]{fontenc}
\usepackage{hyperref}
\usepackage{url}
\usepackage{booktabs}
\usepackage{pgfplots}
\usepackage{amsfonts}
\usepackage{mathtools}
\usepackage{nicefrac}
\usepackage{microtype}
\usepackage{xcolor}
\usepackage{graphicx}
\usepackage{amsmath}
\usepackage{amssymb}
\usepackage{enumitem}
\usepackage{pifont}
\usepackage{colortbl, multirow}
\usepackage{pgfplots}\usepackage{pgfplotstable}
\usepackage{tikz}
\usepackage{algorithm}
\usepackage{algpseudocode}
\usepackage{bm}
\usepackage{caption}
\usepackage{subcaption}     
\usepackage{wrapfig}  

\pgfplotsset{compat=1.18}
\usepackage{soul}
\usepackage[table]{xcolor}
\definecolor{highlight}{RGB}{245,245,200}  
\sethlcolor{highlight}
\usepackage{tablefootnote}

\definecolor{nored}{RGB}{190, 0, 0}    
\definecolor{yesgreen}{RGB}{0, 150, 0} 
\newcommand{\cmark}{\textcolor{yesgreen}{\ding{51}}}
\newcommand{\xmark}{\textcolor{nored}{\ding{55}}}
\newcommand{\softmask}[1]{\textcolor{brown!35!gray}{#1}}

\usepackage{xparse}

\NewDocumentEnvironment{proposition*}{o}
{%
  \par\medskip
  \noindent\textbf{Proposition\IfValueT{#1}{~#1}.}\itshape
}
{%
  \par\medskip
}

\usepackage[utf8]{inputenc} 
\usepackage[T1]{fontenc}    
\usepackage{hyperref}       
\usepackage{url}            
\usepackage{booktabs}       
\usepackage{amsfonts}       
\usepackage{nicefrac}       
\usepackage{microtype}      
\usepackage{xcolor}   
\usepackage{amsthm}

\newtheorem{proposition}{Proposition}
\usepackage{algorithm}
\usepackage{algpseudocode}

\newtheorem*{theorem*}{Theorem}
\usepackage{graphicx}

\newcommand{\ub}[1]{\textcolor{gray!100}{#1}}

\usepackage{fontawesome5}
\definecolor{linkblue}{RGB}{10,20,160}
\newcommand{\globeicon}{\raisebox{-0.17em}{\begin{tikzpicture}[x=1.05em,y=1.05em,line width=0.05em]
  \shade[ball color=cyan!45!blue] (0,0) circle (0.5);
  \begin{scope}[white]
    \draw (0,-0.5) -- (0,0.5);
    \draw (-0.5,0) -- (0.5,0);
    \draw (0,0) ellipse (0.23 and 0.5);
    \draw (-0.433,0.25) -- (0.433,0.25);
    \draw (-0.433,-0.25) -- (0.433,-0.25);
  \end{scope}
  \draw[blue!55!black] (0,0) circle (0.5);
\end{tikzpicture}}}
\newcommand{\githubicon}{\raisebox{-0.1em}{\textcolor{black}{\large\faGithub}}}
\newcommand{\projectlinks}{%
  \href{https://mainaksingha01.github.io/DistMoE}{\globeicon~\textcolor{linkblue}{\textbf{Website}}}%
  \qquad
  \href{https://github.com/mainaksingha01/DistMoE}{\githubicon~\textcolor{linkblue}{\textbf{Code}}}%
}

\title{\textsc{DistMoE}: Private-data Rehearsal-free Routing in Mixture-of-Experts for Distributed Instruction Tuning}

\usepackage{xspace}
\DeclareRobustCommand{\method}{\textsc{DistMoE}\xspace}

\author{
\textbf{Mainak Singha}$^{1}$ \quad
\textbf{Niccol\`o Biondi}$^{1}$ \quad
\textbf{Elisa Ricci}$^{1,2}$ \quad
\textbf{Subhankar Roy}$^{3}$
\\[1.2ex]
$^{1}$University of Trento, Italy
\qquad
$^{2}$Fondazione Bruno Kessler, Italy
\\
$^{3}$University of Bergamo, Italy
\\[1.2ex]
\small
\texttt{
\{mainak.singha, niccolo.biondi, e.ricci\}@unitn.it, subhankar.roy@unibg.it
}
\\[1.8ex]
\normalsize
\projectlinks
}

\begin{document}

\maketitle
\begin{abstract} 
\label{sec:abstract}
Multimodal Large Language Models (MLLMs) have shown strong multimodal instruction-following ability, but adapting them to diverse visual-language domains typically assumes centralized data access and costly joint training. This is restrictive when data is distributed across private, domain-specific, or permission-limited clients. To this end, we propose \method, a mixture-of-experts (MoE) approach for distributed visual instruction tuning. In each layer of the language decoder it augments the public feedforward network (FFN) with a client-specific private FFN expert, with the goal to acquire domain-specific knowledge. However, independent expert training causes the private FFNs to learn representation of different scales and magnitudes, making merging the experts difficult. To reduce client-specific drift, we introduce a public-anchored expert composition stage that updates only routers and lightweight private projection adapters on a mix of \textit{local} client data and public data, via an isotropic regularization loss, therefore making it \emph{cross-client rehearsal-free composition}. During inference, \method performs modular routing over public and private experts, enabling token-wise domain composition without explicit domain labels. Experiments across diverse visual-language benchmarks show that \method enables flexible expert reuse, effective domain adaptation, and competitive performance while preserving modular control over client-specific knowledge.
\end{abstract}

    
\section{Introduction}
\label{sec:introduction}

\begin{figure}[t]
    \centering
        \includegraphics[width=0.95\linewidth]{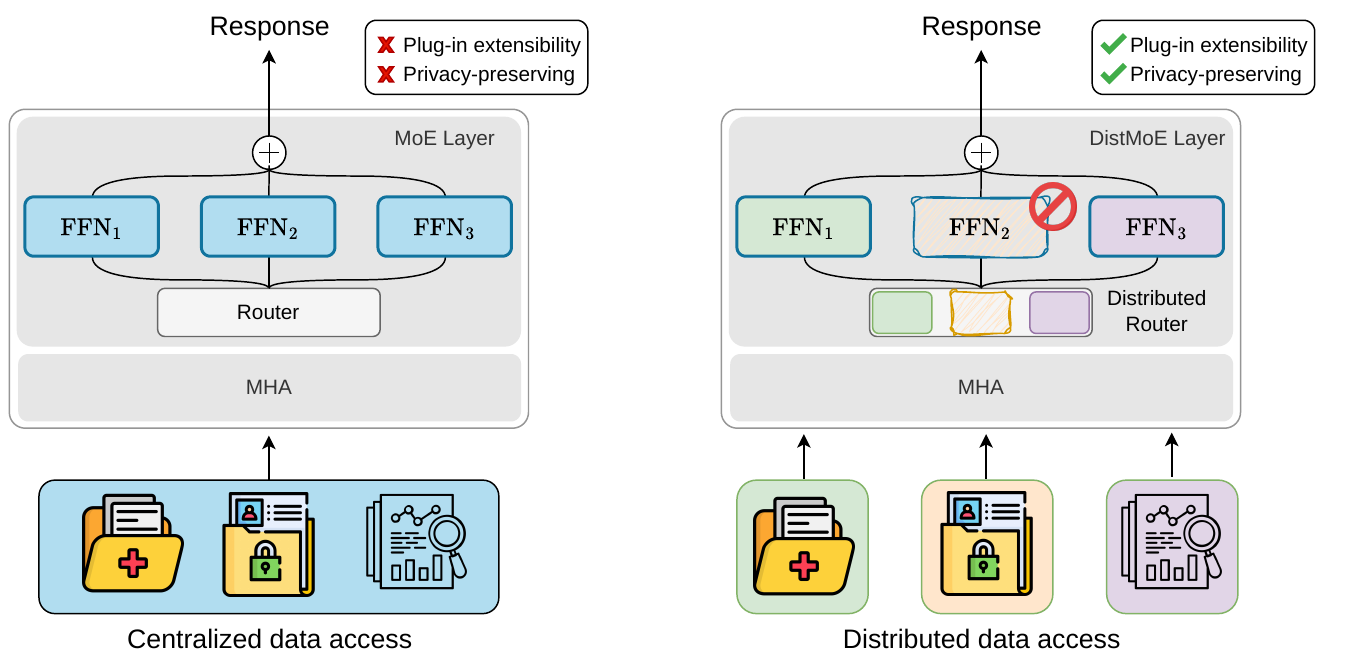}
    \caption{
Centralized MoE tuning vs. \method. \textbf{Left:} Centralized MoE methods \cite{moellava} rely on joint access to domain-specific data. \textbf{Right:} \method independently trains domain-specific private experts with distributed data access and composes them via a distributed router. This enables experts to be flexibly \emph{added}, \emph{removed}, or \emph{reused} at inference time depending on the application needs}
    \label{fig:teaser}
    \vspace{-0.4cm}
\end{figure}

Multimodal Large Language Models (MLLMs) \cite{caffagni2024revolution} have excelled in visual instruction following and reasoning by aligning visual encoders with Large Language Models (LLMs). These gains are often driven by scaling model size \cite{chen2024internvl,li2025llava} and augmenting visual instruction tuning data \cite{zhao2023svit,chen2024sharegpt4v}. Although impressive, such training paradigms present two challenges from training and deployment perspective: (i) a centralized access to all the training data, and (ii) high computation and memory costs during deployment. Thus, it calls for visual instruction tuning approaches that can be trained in a distributed manner and that result in models with a lower computational footprint during inference.

Mixture-of-Experts (MoE) approach \cite{sparsemoe,moe3,mixtral} presents a promising solution towards this goal, as opposed to dense models, it allows to scale models while activating only a subset of the parameters (or \textit{experts)} both during training and inference \cite{moellava}, as well as enables distributed training \cite{sukhbaatar2024branchtrainmix}. For example, MoE-LLaVA~\cite{moellava}, MoE-based visual instruction tuning approach, have shown that sparse expert routing can improve MLLM efficiency while maintaining competitive visual understanding and reduced hallucination with fewer activated parameters than dense baselines. Despite the advantage, MoE-LLaVA assumes centralized access to training data and requires joint training of experts and routers (see Fig. \ref{fig:teaser}). This assumption is restrictive, as visual instruction tuning data can be privately maintained \cite{cui2024biomedical}. Moreover, once such data is utilized to jointly train a MoE model in a centralized setup, it becomes difficult to control, remove, and selectively disable its influence during inference.

A set of works \cite{sukhbaatar2024branchtrainmix,flexolmo,morrison2026train} takes an important step towards flexible data use in LLMs, by independently training experts on private \textit{domain-specific} data slices (e.g., coding, maths, news) and merging them into a unified model through the MoE approach. In detail, a router is trained that learns to combine the domain-specific experts during inference. However, it still requires a \textit{rehearsal} of a subset of the private data from all domains to learn the router, which can still be a bottleneck in extremely privacy-sensitive applications. While FlexOlmo \cite{flexolmo} proposes a technique to merge the experts without further tuning the router, independent expert tuning can cause the experts to diverge making routing scores during inference not comparable. This issue can further be exacerbated in visual instruction tuning as it has been shown that a model can learn to ignore visual information \cite{zhou2026learning}. 


In this work, we propose Distributed MoE (\method), a \textit{cross-client rehearsal-free composition} MoE approach \cite{sparsemoe} for distributed visual instruction tuning (see Fig. \ref{fig:teaser}). To be noted, \emph{cross-client rehearsal-free composition} refers that each client may reuse its own local data, but expert composition never requires pooling, sharing, or jointly revisiting private data from multiple clients. The novelty of \method lies 
in how independently trained experts are calibrated before composition
In detail, following FlexOlmo, we first independently train \textit{private} experts, one per domain, by initializing it from a \textit{public} model, to acquire domain-specific knowledge. However, different from FlexOlmo, after expert-training, we introduce a lightweight and privacy-preserving approach to ensure that the private expert remains anchored to the public expert that is shared across all the clients. This is enforced by imposing a Gaussian structure \cite{lejepa} on the residuals between the public and private feature representations in the intermediate layers of the Transformer. In addition, we calibrate the router using a mix of public and strictly client-specific private data. As a result, the representation and routing scores across domains become comparable, making merging straightforward during inference. The native MoE architecture endows flexibility to \method by enabling computation-efficient sparse routing and seamless exclusion of sensitive domain-specific knowledge during inference. Here, efficiency refers to a fixed expert-activation budget relative to dense all-expert execution, not to a speedup over FlexOlmo (Sec.~\ref{sec:ablations}). Likewise, \method provides \emph{data locality} rather than formal privacy guarantees, and excluding an expert disables its influence at inference.


Our \textbf{contributions} are: \textit{i)} We introduce \method, a MoE architecture-based approach, for distributed visual instruction tuning in MLLMs. \textit{ii)} To facilitate seamless composition of knowledge across experts, we introduce a novel \textit{cross-client rehearsal-free composition} technique that calibrates private experts and the associated router weights without relying on joint access to private data from all clients.  \textit{iii)}  We conducted experiments on several multimodal evaluation benchmarks to evaluate in-domain and cross-domain generalization. Our results demonstrate the effectiveness and flexibility of \method.



\section{Related Works}
\label{sec:related_works}

\noindent\textbf{Visual instruction tuning}, pioneered by LLaVA \cite{llava}, has emerged as a successful recipe for building multimodal LLMs from pre-trained models for addressing tasks such as visual question answering (VQA). It aligns visual encoder (e.g., CLIP \cite{radford2021learning}) with LLM decoder (e.g., Vicuna \cite{chiang2023vicuna}) using GPT generated instruction-following data \cite{caffagni2024revolution}. However, methods like LLaVA relies on the assumption that all the instruction tuning data is available simultaneously. Instead, in real-world applications, introduction of new tasks and domains requires the model to learn dynamically over time. This has shifted the research focus towards Continual Visual Instruction Tuning (CVIT) \cite{chen2024coin,pmlr-v267-chen25n, Zhao_2026_CVPR, song2026kss, he2026continual, zhu2024model}, where the goal is to learn continually while avoiding catastrophic forgetting by employing regularization techniques. A characteristic assumption made in CVIT is that tasks arrive sequentially and therefore, the employed regularization techniques are specifically designed towards striking a balance between stability and plasticity. We argue that as training becomes more and more decentralized on heterogeneous hardware \cite{douillard2023diloco}, enforcing regularization while learning a new task is limiting and impractical. Moreover, as argued in the work \cite{nikandrou2022task}, very often task boundaries cannot be clearly demarcated in vision-language tasks. Thus, we seek a more general and flexible framework that can enable \textit{distributed} visual instruction tuning, where knowledge can be acquired independently and consolidated post-hoc without explicitly accounting for task type or arrival order.


\noindent\textbf{Model merging} provides a practical way to combine independently trained models into a unified system \cite{yang2026model} without requiring access to raw training data or expensive computation. Existing methods mainly include weight averaging or output ensembling. Weight averaging approaches, such as Model Soups \cite{wortsman2022model}, combine models fine-tuned from the same initialization, while output ensembling methods, such as Branch-Train-Merge (BTM) \cite{li2022branchtrainmerge}, merge predictions from models trained independently in different domains. While model merging inherently offers more flexibility than regularization-based continual learning approaches, model merging is prone to interference between parameters of constituent models, resulting in performance drops \cite{yadav2023ties}. To resolve this, additional steps are typically involved, such as selective merging or post-hoc fine-tuning. These additional steps introduce challenges in a purely distributed setup, where data is locally maintained.


\noindent\textbf{Mixture-of-Experts} (MoE) tuning, originally designed to scale deep networks \cite{sparsemoe} without increasing computation, dynamically selects feedforward networks or \textit{experts} using Top-K routing scheme that activates only a fraction of the total number of parameters. It has been successfully applied to scale LLM training \cite{mixtral, agarwal2025gpt}, vision-language tasks \cite{moe1,moe2,moe3,moellava} and privacy-preserving learning \cite{sukhbaatar2024branchtrainmix,flexolmo,morrison2026train}. For instance, MoE-LLaVA~\cite{moellava} extends MoE tuning to visual instruction tuning through training experts on different domains, but requires joint access to the union of all expert-training datasets. Similar to our setup, more privacy-friendly approaches, such as Branch-Train-Mix (BTX) \cite{sukhbaatar2024branchtrainmix}, FlexOlmo \cite{flexolmo} and BAR \cite{morrison2026train}, were proposed that first independently train models and later combine them using MoE routing. However, all these approaches are designed for LLMs and assume access to \textit{rehearsal} training data to learn the routing mechanism. In contrast, \method is designed for distributed visual instruction tuning and \textit{composes experts without cross-client rehearsal}.

\section{Distributed Visual Instruction Tuning}
\label{sec:dist-vit}
\vspace{-0.1cm}

\subsection{Problem Formulation}
\label{subsec:problem_setup}

Let $\mathcal{D}= \{(I_i, q_i, y_i) \mid i \in \mathbb{N}\}$ denote a dataset of image, a textual query, and the target answer tuples. In visual instruction tuning (e.g., LLaVA \cite{llava}), the goal is to jointly train and/or fine-tune pre-trained LLM, vision transformer (ViT) \cite{radford2021learning} and a multimodal projector to predict the target $y$ given image and query $(I, q)$ pair at the input. The prediction is defined as:
\begin{equation}
\hat{y} = f_{\theta}(I, q) = f_{\theta_{\text{llm}}}\!\left(f_{\theta_{\text{mm}}}(f_{\theta_{\text{vis}}}(I)), q\right),
\end{equation}
where $f_{\theta_{\text{vis}}}$ is the vision encoder that processes the input image $I$, $f_{\theta_{\text{mm}}}$ is the multimodal projector, and $f_{\theta_{\text{llm}}}$ is the LLM decoder; and $\theta = \{\theta_{\text{vis}}, \theta_{\text{mm}}, \theta_{\text{llm}}\}$ are the parameters of the model $f_\theta$.

Similar to \cite{flexolmo}, in distributed visual instruction tuning, we assume that the instruction tuning dataset $\mathcal{D} = \{\mathcal{D}_1, \cdots, \mathcal{D}_M\}$ is  a collection of locally maintained datasets, spread among $M$ clients. Each client $m$ is considered to possess \textit{domain-specific} instruction tuning dataset $\mathcal{D}_m$ that is private. In addition, we assume access to a public dataset $\mathcal{D}_\text{pub}$ that is available to all clients. Our objective is to obtain a single model $f_\theta$ that is constructed by combining the knowledge of the clients, each independently trained on $\mathcal{D}_m$, without requiring joint access to the private datasets.



\subsection{Preliminaries}
\label{sec:prelim}


The idea of MoE \cite{sparsemoe,mixtral}, originally designed to scale a LLM's capacity without a proportional increase in computation, has been extended for visual instruction tuning in MoE-LLaVA \cite{moellava}. The main goal remains the same: constructing a sparse multimodal model with high learning capacity but with a constant computational cost. We note that this setup in MoE-LLaVA lends well while learning from diverse instruction tuning datasets as one Feed Forward Network (FFN) or \textit{expert}, in the LLM decoder, can specialize on each domain-specific dataset. Then, as in standard MoE, a learned router in the decoder can redirect the input tokens (both visual and text) to relevant domain experts.

More formally, let $\mathcal{V} = [v_1, \cdots, v_P]$ denote the sequence of $P$ tokens corresponding to the image $I$ from the multimodal projector $f_{\theta_{\text{mm}}}$, and $\mathcal{T} = [t_1, \cdots, t_N]$ be the sequence of $N$ text tokens corresponding to the input query $q$. These tokens are concatenated and fed to the LLM decoder, which comprise of stacked multi-head self-attention (MSA) and FFN blocks. MoE-LLaVA additionally consists of a MoE layer, comprised of multiple expert FFNs and a router. Denoting the expert FFNs as $\mathcal{E} = [e_1, \cdots, e_M]$; and the router, which is a linear layer, predicts the probability of each token being assigned to the $i-$th expert as:
\begin{equation}
\label{eqn:moe-softmax}
    \alpha_i = \text{softmax}(\mathbf{x})_i = \frac{\exp (g(\mathbf{x})_i)}{\sum^{M}_{j} \exp (g(\mathbf{x})_j)},
\end{equation}
where $\mathbf{x}$ denotes the intermediate representation of a token;  $g(\mathbf{x})_i = \mathbf{x}^\intercal \mathbf{r}_i$ and $\mathbf{r}_i$ are the logit score and router weight corresponding to the $i-$th FFN, respectively. For each token in a sequence, the top$-k$ highest probabilities are selected, which are then used to compute a weighted-sum representation as:
\begin{equation}
\label{eqn:moe}
    \mathbf{o} = \text{MoE}(\mathbf{x}) = \sum^{k}_{i=1} \alpha_i \cdot \mathcal{E}_i(\mathbf{x}),
\end{equation}
where $\mathbf{o}$ is the output of the MoE layer. MoE-LLaVA trains both the experts and the router with next-token prediction and router balancing loss, in a \textit{centralized} fashion, by having joint access to all the domain-specific datasets. However, joint access presents as a bottleneck in the distributed visual instruction tuning task. While in theory, each expert FFN $\mathcal{E}_m$ can be trained independently on $\mathcal{D}_m$, and later merged as in FlexOlmo \cite{flexolmo}, it still requires rehearsal of the private data to learn the router.

\begin{figure}[t]
    \centering
    \includegraphics[width=\linewidth]{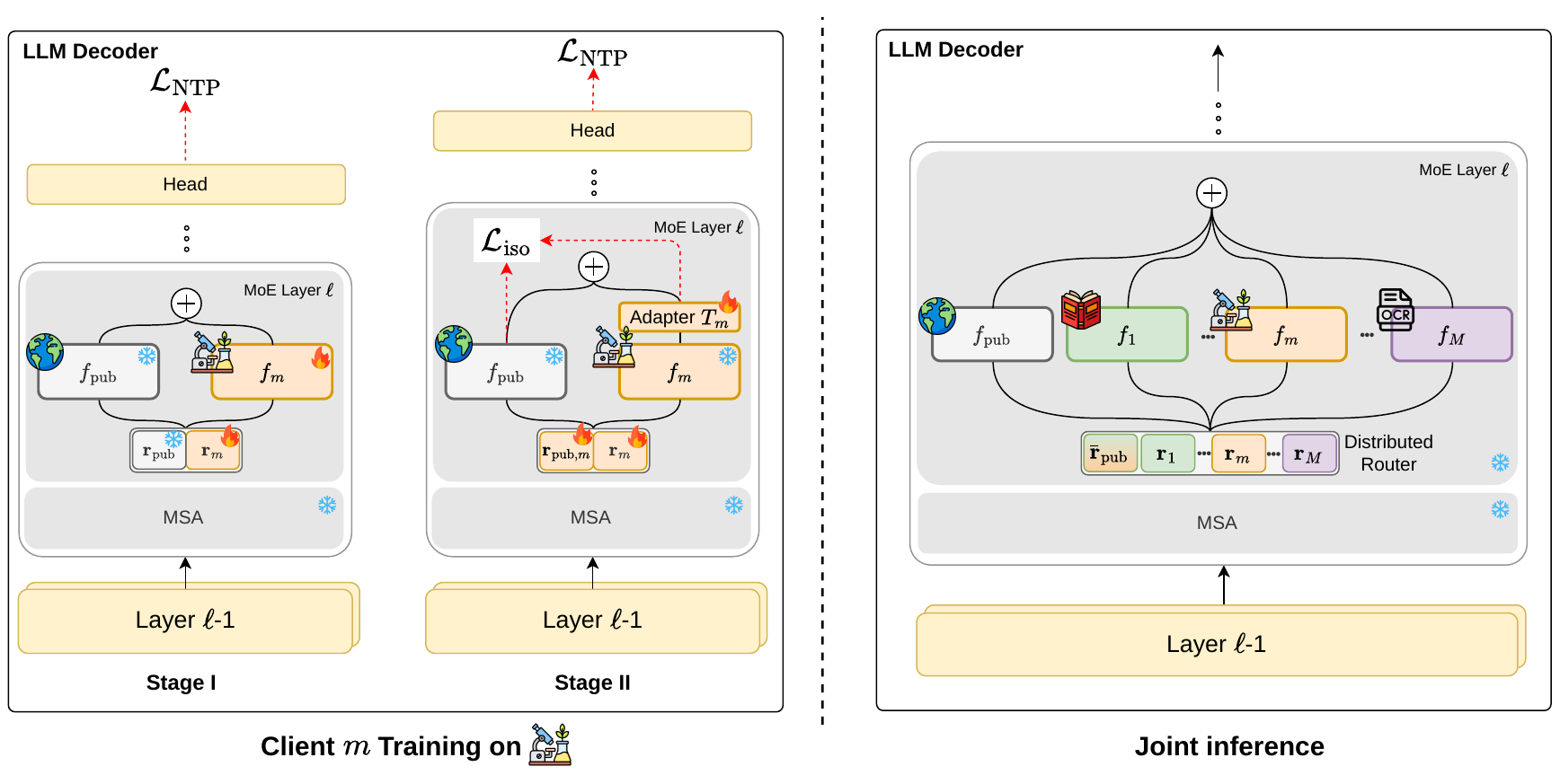}
    \caption{
    Overview of \method.
\textbf{Left}. In Stage-I each client independently trains a private FFN by initializing from a shared public FFN. 
In Stage-II we freeze the experts, and recalibrate routing using isotropic residual regularization and next-token prediction loss utilizing a mix of public and private data. 
\textbf{Right.} At inference, \method composes the shared public expert and calibrated private experts through sparse routing without explicit domain labels.
}
    \label{fig:main_arch_training}
    \vspace{-0.4cm}
\end{figure}

\section[\method: Distributed MoE]{\method: Private-data Rehearsal-free Routing in MoE MLLMs}
\label{sec:distmoe}

To support a more modular and privacy-preserving approach to visual instruction tuning, we propose \method, based on MoE architecture. The key idea behind \method is to train the experts on private data in a novel way, so that during inference the router does not need any additional training on the combined private data. We split the training phase in two stages: (i) In Stage-I (see Sec. \ref{sec: stage1}) we acquire domain-specific knowledge by training a "two-expert" MoE network, where we initialize each private expert $\mathcal{E}_m$ from a public FFN $\mathcal{E}_\text{pub}$ trained on generic public data. (ii) In Stage-II (see Sec. \ref{sec: stage2}) we calibrate routing of tokens, so that during inference, the experts can simply be merged to behave like a standard MoE router trained on all the domain-specific data jointly. Fig.~\ref{fig:main_arch_training} and Algorithm~\ref{alg:distmoe} in Appendix~\ref{sec:algorithm} summarizes the resulting training and inference procedure.

\vspace{-0.1cm}
\subsection{Training}
\label{sec: train-algo}
\vspace{-0.1cm}

Standard MoE training, as in MoE-LLaVA, trains experts on domain-specific datasets jointly. In the distributed setup, this is infeasible, and training experts in isolation can cause them to diverge from each other, making merging during inference difficult. While FlexOlmo proposes to use the public $\mathcal{E}_\text{pub}$ as an \textit{anchor} to induce a coordination effect among private experts, we observe that private experts can still drift from the anchor, and thus from one another. To address this, as shown in Fig.~\ref{fig:main_arch_training}, \method separates \emph{expert knowledge acquisition} (i.e., Stage-I) from \emph{expert composition} (i.e., Stage-II). Unless otherwise stated, private experts are trained within each layer of the decoder.



\subsubsection{Stage-I: Expert knowledge acquisition}
\label{sec: stage1}

Following FlexOlmo, we dedicate Stage-I to accomplish two goals: (i) acquire domain-specific knowledge on a private dataset $\mathcal{D}_m$, and (ii) prepare the experts to coordinate with each other. More specifically, we construct a MoE model with two experts: a public FFN $\mathcal{E}_\text{pub}$ and a private FFN $\mathcal{E}_\text{m}$ initialized from pre-trained $\mathcal{E}_\text{pub}$. During training, we freeze the parameters of the attention layers and $\mathcal{E}_\text{pub}$, and update only the private expert and its corresponding router parameters using the next-token prediction loss $\mathcal{L}_{\mathrm{NTP}}$ on $\mathcal{D}_m$. As a result, for a token with hidden state $\mathbf{x}$, the router learns to predict the probability of being assigned to the public or private expert. The MoE layer output is given as:
\begin{equation}
\label{eqn: two-moe-out}
    \mathbf{o}
    =
    \alpha_{\mathrm{pub}}
    \mathcal{E}_{\mathrm{pub}}(\mathbf{x})
    +
    \alpha_{m}
    \mathcal{E}_{m}(\mathbf{x}),
\end{equation}
where $\alpha_{\mathrm{pub}}$ and $\alpha_{m}$ are the probabilities of being assigned to the public and private experts, respectively. Intuitively, the public expert $\mathcal{E}_\text{pub}$ serves as a shared anchor for all experts, and the two-way router learns a mechanism that enables their integration into a unified MoE architecture.




While Stage-I leads to knowledge acquisition, empirically, we found that this approach introduced by FlexOlmo still falls short of producing optimal routers for joint inference across all domains. We hypothesize that this occurs for two reasons. \textit{First}, the public router is never trained on public data, and thus, cannot sufficiently distinguish between information sources. \textit{Second}, Eq. (\ref{eqn: two-moe-out}) cannot guarantee that the outputs $\mathbf{o}$ from different clients will intrinsically have compatible magnitudes or directions. As a result of this, it makes the private router scores across clients not comparable.



\vspace{-0.1cm}
\subsubsection{Stage-II: Expert composition}
\label{sec: stage2}
\vspace{-0.1cm}

In Stage-II we address the mismatch of router scores across clients. While a straightforward approach is to collect a small subset of private data from each client to fine-tune the merged router \cite{flexolmo,sukhbaatar2024branchtrainmix}, we do so \textit{without} collecting private data from other clients. We introduce a lightweight adapter $T_m(\cdot)$ on top of the private expert $\mathcal{E}_m$ with the aim to encourage the output of the private expert to be comparable with that of the public expert $\mathcal{E}_\text{pub}$. We intentionally choose the adapter to be \textit{linear} because here the goal is not to acquire new knowledge. In addition, we unlock the parameters of the two-way router in a client $m$ to recalibrate the routing between the public and private expert. All other parameters of the model are kept frozen, making this stage computationally lighter than Stage-I.

Let us denote $\mathbf{h}_{\mathrm{pub}}=\mathcal{E}_{\mathrm{pub}}(\mathbf{x})$ and 
$\widetilde{\mathbf{h}}_{m}=T_m(\mathcal{E}_{m}(\mathbf{x}))$ be the representation of the public expert and the projected representation of the private expert, respectively. We define the \textit{client residual} as:
\begin{equation}
    \boldsymbol{\Delta}_m(\mathbf{x})
    =
    \widetilde{\mathbf{h}}_{m}
    -
    \mathbf{h}_{\mathrm{pub}}.
\label{eq:client_residual}
\end{equation}

The client residual measures how a private expert output deviates from the shared public expert after linear projection. Ideally, we want the residuals to have a comparable scale across clients. Inspired by LeJEPA~\cite{lejepa} that imposes isotropic Gaussian structure on learned representations, we impose a common residual distribution across clients by forcing the residual in each client to follow an isotropic Gaussian distribution. This leads the private experts trained on different datasets produce outputs with comparable residual scale around the same shared public expert. Formally, let $\mathcal{B}_m$ denote the token-level residuals in a mini-batch of client $m$. For sampled unit directions $\{\mathbf{a}_s\}_{s=1}^{S}$, we define
\begin{equation}
    \mathcal{L}_{\mathrm{iso}}
    =
    \frac{1}{S}
    \sum_{s=1}^{S}
    \mathrm{EP}
    \left(
        \{\mathbf{a}_s^\top \boldsymbol{\Delta}_{m}(\mathbf{x}) : \mathbf{x}\in\mathcal{B}_m\},
        \mathcal{N}(0,1)
    \right),
\label{eq:iso_loss}
\end{equation}
where $\mathrm{EP}(\cdot,\mathcal{N}(0,1))$ is the Epps--Pulley discrepancy between the projected residuals and a standard Gaussian\footnote{Given $\{u_i\}_{i=1}^{B}$, 
$\mathrm{EP}=B\int |\widehat{\varphi}(t)-\varphi_0(t)|^2\varphi_0(t)\,dt$, where 
$\widehat{\varphi}(t)=\frac{1}{B}\sum_{i=1}^{B}\exp(\mathrm{i}tu_i)$ and 
$\varphi_0(t)=\exp(-t^2/2)$.}. Because all clients are calibrated against the same public expert and isotropic Gaussian target, \method makes the private experts more comparable for distributed MoE composition.



In addition to imposing Gaussian structure on the residuals, we recalibrate the router by fine-tuning the router parameters corresponding to both the public and private expert. As we formalize later in Proposition~\ref{prop:residual_bound}, this is crucial because inference-time expert influence is controlled by both the client-residual magnitude and recalibrated routing scores. Formally, the Stage-II objective is
\begin{equation}
\label{eq:stage2_objective}
    \mathcal{L}_{\mathrm{Stage\text{-}II}}
    =
    \mathcal{L}_{\mathrm{NTP}}
    +
    \lambda
    \mathcal{L}_{\mathrm{iso}},
\end{equation}
where both $\mathcal{L}_{\mathrm{NTP}}$ and $\mathcal{L}_{\mathrm{iso}}$ are optimized on a mixed dataset $\mathcal{D}^{\mathrm{mix}}_m = \mathcal{D}_m \cup \mathcal{D}_{\mathrm{pub}}$, where $\mathcal{D}_\text{pub}$ is small subset of the publicly available dataset used for training the public model; and $\lambda$ is a hyperparameter for balancing the two losses. 
Since, $\mathcal{D}^{\mathrm{mix}}_m$ contains no private data from other clients, expert composition is cross-client rehearsal-free, unlike \cite{flexolmo,sukhbaatar2024branchtrainmix,morrison2026train} which fine-tune the merged router on pooled private data. Each client reuses only its own local data $\mathcal{D}_m$.
Moreover, $\mathcal{D}_{\mathrm{pub}}$ is optional; with $\mathcal{D}_{\mathrm{pub}}=\emptyset$, Stage-II calibrates on $\mathcal{D}_m$ alone (Fig.~\ref{fig:stage2_public_data}).


\vspace{-0.1cm}
\subsection{Inference}
\label{subsec:inference}
\vspace{-0.1cm}

During inference, we employ a standard MoE approach to combine the expert models $\mathcal{E}_m$ along with the public expert $\mathcal{E}_\text{pub}$. After Stage-II, each client $m$ contributes a calibrated private expert $\mathcal{E}_m$, router weights $\mathbf{r}_m$ and bias $b_{m}$, and a linear adapter $T_m(\cdot)$, while the shared public expert remains common across clients. Note that as the public expert is shared among clients, each client possesses their own public router weights $\mathbf{r}_{\mathrm{pub},m}$. Therefore, we average the public-route parameters learned during Stage-II as $\bar{\mathbf{r}}_{\mathrm{pub}}=\frac{1}{M}\sum_{m=1}^{M}\mathbf{r}_{\mathrm{pub},m}$. However, since the softmax in Eq.~(\ref{eqn:moe-softmax}) is invariant to a common shift of its logits, each client's public logit is only defined up to a client-specific offset (a \emph{gauge}), and averaging $\mathbf{r}_{\mathrm{pub},m}$ does not in general preserve the public--private balance learned by each client. We therefore also introduce a gauge-fixed composition rule that only changes inference. The gauge-invariant quantity is the private-vs-public log-odds $\ell_m(\mathbf{x}) = g_m(\mathbf{x}) - g_{\mathrm{pub},m}(\mathbf{x}) = \log(\alpha_m/\alpha_{\mathrm{pub}})$ learned by the two-way router of client $m$, where $g_{\mathrm{pub},m}$ and $g_m$ are its public and private logits. We fix the gauge by anchoring the public logit at zero,
\begin{equation}
    z_{\mathrm{pub}}(\mathbf{x}) = 0, \qquad z_m(\mathbf{x}) = \ell_m(\mathbf{x}), \quad m = 1,\dots,M,
\label{eq:gauge_fixed}
\end{equation}
followed by the same top-$K$ softmax as in Eq.~(\ref{eqn:moe-softmax}), which requires neither retraining nor additional data. For FlexOlmo, the public router is identical across clients, so Eq.~(\ref{eq:gauge_fixed}) subtracts a common term from all logits and leaves its routing unchanged. We evaluate this rule in Sec.~\ref{sec:ablations}. As our method is based on the MoE approach, it also supports different types of routing during inference, such as sparse top-$K$ routing (as noted in Eq.~(\ref{eqn:moe})) and soft routing involving all experts.

\subsection{Discussion} 
Next, we discuss why Stage-II can alleviate the issue highlighted in Sec. \ref{sec: stage1}.
Following the MoE routing rule in Eq.~\eqref{eqn:moe-softmax}, we denote by 
$\alpha_{\mathrm{pub}}$ and $\{\alpha_m\}_{m=1}^{M}$ the routing weights assigned to the public expert and the $M$ private experts, respectively. 
Under sparse top-$K$ routing, non-selected routes are assigned zero weight. 
The routed output is
\vspace{-0.3cm}
\begin{equation}
    \mathbf{o}
    =
    \alpha_{\mathrm{pub}}
    \mathbf{h}_{\mathrm{pub}}
    +
    \sum_{m=1}^{M}
    \alpha_m
    \widetilde{\mathbf{h}}_{m}.
\label{eq:distmoe_inference_output}
\end{equation}
Using the client residual 
$\boldsymbol{\Delta}_{m}
=
\widetilde{\mathbf{h}}_{m}
-
\mathbf{h}_{\mathrm{pub}}$, 
and the normalization 
$\alpha_{\mathrm{pub}}+\sum_{m=1}^{M}\alpha_m=1$, 
Eq.~\eqref{eq:distmoe_inference_output} can be rewritten as
\begin{equation}
    \mathbf{o}
    =
    \mathbf{h}_{\mathrm{pub}}
    +
    \sum_{m=1}^{M}
    \alpha_m
    \boldsymbol{\Delta}_{m}.
\label{eq:residual_decomposition}
\end{equation}
The algebraic derivation of Eq.~\eqref{eq:residual_decomposition} is provided in Appendix~\ref{app:residual_bound}. 
This equation shows why Stage-II regularizes the client residuals rather than the full expert outputs. 
At inference, the private experts do not contribute as isolated outputs; they contribute through their residuals relative to the shared public expert. 
Therefore, the effect of expert $m$ is determined by two factors: the routing weight 
$\alpha_{m}$, which is learned by the router, and the residual 
$\boldsymbol{\Delta}_{m}$, which determines the direction and magnitude of the representation shift introduced by that expert. 
If two clients receive similar routing weights but produce residuals with very different norms, the larger-residual expert can dominate the mixture regardless of whether it is semantically more relevant.

\begin{proposition}
\label{prop:residual_bound}
For the routed output in Eq.~\eqref{eq:residual_decomposition}, the deviation from the shared public expert is bounded by the routing-weighted residual magnitudes:
\vspace{-0.3cm}
\begin{equation}
    \left\|
        \mathbf{o}
        -
        \mathbf{h}_{\mathrm{pub}}
    \right\|_{2}
    \leq
    \sum_{m=1}^{M}
    \alpha_m
    \left\|
        \boldsymbol{\Delta}_{m}
    \right\|_{2}.
\label{eq:residual_bound}
\end{equation}
\end{proposition}
\vspace{-0.2cm}
The proof follows from Eq.~\eqref{eq:residual_decomposition}, the triangle inequality, and the non-negativity of the routing weights; the full proof is given in Appendix~\ref{app:residual_bound}. 
Proposition~\ref{prop:residual_bound} motivates calibrating both factors: the adapter and $\mathcal{L}_iso$ act on $\boldsymbol{\Delta}_m$, the router on $\alpha_m$.

The router learns how much mass to assign to each private expert, but the actual representation shift also depends on 
$\|\boldsymbol{\Delta}_{m}\|_2$. 
If different clients produce residuals with incompatible scales, the mixture can be influenced by representation magnitude rather than by the learned routing weights alone. 
By regularizing client residuals around the same public expert and the same isotropic Gaussian target, Stage-II makes these residuals more comparable while preserving client specialization. 
This supports cross-client rehearsal-free composition of independently trained experts in the final distributed MoE model.
\vspace{-0.1cm}
\section{Experimental Results}
\label{sec:experiments}
\vspace{-0.1cm}

\noindent\textbf{Experimental setup.}
We train \method with five client domains: COCO~\cite{lin2014microsoft}, GQA~\cite{hudson2019gqa}, OCRVQA~\cite{8978122}, VQA$^{\mathrm{T}}$: TextVQA~\cite{singh2019towards}, and Visual Genome (VG)~\cite{krishna2017visual}, treating each dataset as a separate private client. 
We initialize the shared public model from the MoE-LLaVA stage-2~\cite{moellava} checkpoint and keep the vision encoder and multimodal connector frozen. 
In Stage-II, each client calibrates routers/adapters using its private data plus a fixed 19k public subset; this subset is a design choice rather than a requirement (Fig.~\ref{fig:stage2_public_data}).
We instantiate \method with Qwen-1.8B~\cite{bai2023qwen}, and Phi-2.7B~\cite{javaheripi2023phi2}, and StableLM-1.6B~\cite{bellagente2024stable} language backbones; all variants use CLIP ViT-L/14~\cite{radford2021learning} as the frozen vision encoder. 
We evaluate on image QA benchmarks, including VQA$^{\mathrm{v2}}$~\cite{goyal2017making}, GQA, SQA$^{\mathrm{I}}$: ScienceQA-IMG~\cite{lu2022learn}, and VQA$^{\mathrm{T}}$, as well as general MLLM benchmarks, including POPE~\cite{li2023evaluating}, MME~\cite{fu2026mme}, SEED-Bench~\cite{li2024seed}, LLaVA$^{\mathrm{W}}$: LLaVA-Bench (in-the-Wild)~\cite{llava}, and MM-Vet~\cite{pmlr-v235-yu24o}. 
This evaluation covers both in-domain performance on clients' training datasets and out-of-domain generalization on different multimodal benchmarks. 
Additional implementation details are provided in Appendix~\ref{sec:additional_experimental_details}.

\noindent\textbf{Compared methods.}
The Public model denotes the MoE-LLaVA stage-2 checkpoint used to initialize the MLLM before client-specific expert adaptation. 
Offline reference models, such as Qwen-VL, LLaVA-Phi, and MoE-LLaVA, are included as centralized reference points. They represent the upperbound of our distributed approach. 
We report two MoE-LLaVA references: the original results reported by MoE-LLaVA~\cite{moellava} and a variant trained with joint access to all client datasets. 
We compare \method against Model Soup, weighted Model Soup, BTM, and BTX using their standard model-merging or expert-composition procedures. 
We also report the FlexOlmo-style routing baselines, denoted as FlexOlmo-LoRA and FlexOlmo, which correspond to Stage-I expert acquisition without the Stage-II public-anchored router and adapter calibration. 
Following \cite{flexolmo}, we additionally report FlexOlmo-RT, which fine-tunes the merged router after Stage-I using 10k private samples per client dataset plus the same 19k public subset used in our Stage-II. 
This rehearsal-based baseline serves as an upper-bound reference rather than a cross-client rehearsal-free composition distributed baseline.
Finally, we include two variants of \method: a variant that adapts one private FFN expert per client, and a parameter-efficient LoRA variant that replaces each full private FFN with multiple lightweight LoRA experts. 
For the LoRA variants, we use $R=16$ LoRA experts per client with rank $r=16$ and scaling $\alpha=32$. 
Unless otherwise stated, we use top-$k$ routing with $k=4R$ for LoRA variants and $k=4$ for all other distributed methods; the effect of varying $R$ is analyzed in Appendix~\ref{sec:lora}.

\begin{table*}[t]
\vspace{-0.1cm}
\centering
\caption{Main results for \method with Qwen-1.8B and Phi-2.7B. 
MoE-LLaVA$^\dagger$ denotes the original results reported in~\cite{moellava}, while MoE-LLaVA denotes the model trained with joint access to the datasets used in our distributed setting. FlexOlmo-RT denotes FlexOlmo with Router Training~\cite{flexolmo}. Results are from a single seed run; three-seed results are reported in Tab.~\ref{tab:seeds}.
}
\label{tab:main_all_backbones}
\vspace{-0.1cm}
\setlength{\tabcolsep}{4pt}
\footnotesize
\resizebox{\textwidth}{!}{%
\begin{tabular}{ll cccc>{\hspace{8pt}} ccccc}
\toprule
\textbf{LLM} & \textbf{Method}
& \multicolumn{4}{c}{\textbf{Image Question Answering}}
& \multicolumn{5}{c}{\textbf{Benchmark Toolkit}} \\
\cmidrule(lr){3-6} \cmidrule(lr){7-11}
& & VQA$^{\mathrm{v2}}$ & GQA & SQA$^{\mathrm{I}}$ & VQA$^{\mathrm{T}}$
& POPE & MME & SEED & LLaVA$^{\mathrm{W}}$ & MM-Vet \\
\midrule

\multicolumn{1}{l}{} & \multicolumn{10}{l}{\textbf{\textit{Offline reference models (upperbound)}}} \\
\cellcolor{white}\multirow{12}{*}{Qwen-1.8B}
& \ub{Qwen-VL}
& \ub{78.80} & \ub{59.30} & \ub{67.10} & \ub{63.80}
& \ub{--} & \ub{--} & \ub{--} & \ub{--} & \ub{--} \\

& \ub{MoE-LLaVA$^\dagger$}
& \ub{76.20} & \ub{61.50} & \ub{63.10} & \ub{48.00}
& \ub{87.00} & \ub{1291.60} & \ub{--} & \ub{88.70} & \ub{25.30} \\

& \ub{MoE-LLaVA}
& \ub{72.81} & \ub{57.29} & \ub{63.22} & \ub{44.15}
& \ub{89.32} & \ub{1230.16} & \ub{57.08} & \ub{62.20} & \ub{15.60} \\

& \ub{FlexOlmo-RT}
& \ub{72.36} & \ub{56.66} & \ub{62.17} & \ub{45.19}
& \ub{87.85} & \ub{1205.85} & \ub{56.51} & \ub{64.00} & \ub{19.70} \\

\cmidrule(lr){2-11}
& Public model
& 20.96 & 13.75 & 58.50 & 17.72
& 71.17 & 402.54 & 51.17 & 71.00 & 18.60 \\
& Model Soup (avg)
& 68.92 & 53.18 & 59.98 & 40.83 & 84.55 & 1154.80 & 53.82 & 50.40 & 14.60 \\
& Model Soup (wtd)
& 69.18 & 53.46 & 60.21 & 41.18 & 84.84 & 1160.35 & 54.06 & 51.10 & 14.70 \\
& BTM
& 69.73 & 53.81 & 60.66 & 41.87 & 85.46 & 1167.92 & 54.31 & 52.00 & 14.80 \\
& BTX
& 69.46 & 54.16 & 60.43 & 41.52 & 85.15 & 1174.63 & 54.63 & 52.90 & 14.90 \\

& FlexOlmo-LoRA
& 63.45 & 50.28 & 59.15 & 34.46 & 85.35 & 938.50 & 53.20 & 72.00 & 21.40 \\
& FlexOlmo
& 70.01 & 54.26 & 60.85 & 41.70 & 85.98 & 1176.35 & 54.91 & 52.20 & 15.00 \\
\rowcolor{yesgreen!8}
\cellcolor{white}& \method-LoRA
& 65.80 & 50.51 & 59.30 & 33.69 & 85.68 & 995.61 & 53.34 & \textbf{73.50} & \textbf{24.40} \\
\rowcolor{yesgreen!8}
\cellcolor{white}& \method
& \textbf{70.27} & \textbf{55.35} & \textbf{61.03} & \textbf{42.66}
& \textbf{86.38} & \textbf{1188.56} & \textbf{55.11} & 59.90 & 15.10 \\

\midrule
\midrule

\multicolumn{1}{l}{} & \multicolumn{10}{l}{\textbf{\textit{Offline reference models (upperbound)}}} \\
\cellcolor{white}\multirow{12}{*}{Phi-2.7B}
& \ub{LLaVA-Phi}
& \ub{71.40} & \ub{--} & \ub{68.40} & \ub{48.60}
& \ub{85.00} & \ub{1335.10} & \ub{--} & \ub{--} & \ub{28.90} \\

& \ub{MoE-LLaVA}$^\dagger$
& \ub{77.60} & \ub{61.40} & \ub{68.50} & \ub{51.40}
& \ub{85.70} & \ub{1423.00} & \ub{--} & \ub{94.10} & \ub{34.30} \\

& \ub{MoE-LLaVA}
& \ub{74.99} & \ub{56.69} & \ub{68.37} & \ub{46.52}
& \ub{85.65} & \ub{1331.75} & \ub{60.84} & \ub{76.20} & \ub{20.20} \\

\cmidrule(lr){2-11}
& Public model
& 42.87 & 32.80 & 69.26 & 31.82
& 80.48 & 862.53 & 59.36 & 73.30 & 23.20 \\

& Model Soup (avg)
& 64.89 & 51.67 & 59.28 & 39.57 & 74.17 & 1054.37 & 53.38 & 42.40 & 15.70 \\

& Model Soup (wtd)
& 65.92 & 52.34 & 59.89 & 39.76 & 74.43 & 1067.26 & 54.05 & 43.20 & 14.40 \\

& BTM
& 66.53 & 52.56 & 60.56 & 40.76 & 72.80 & 1106.64 & 52.82 & 41.70 & 16.30 \\

& BTX
& 67.80 & 52.67 & 60.24 & 40.53 & 73.62 & 1096.48 & 51.43 & 44.40 & 17.00 \\

& FlexOlmo-LoRA
& 42.87 & 32.78 & 69.41 & 31.74 & 80.53 & 862.53 & 59.39 & 72.10 & 22.30 \\

& FlexOlmo
& 70.03 & 54.84 & 62.77 & 41.63 & 74.98 & 1193.71 & 58.79 & 52.00 & 18.30 \\

\rowcolor{yesgreen!8}
\cellcolor{white}& \method-LoRA
& 42.89 & 32.84 & \textbf{69.52} & 31.77
& \textbf{80.84} & 862.60 & 59.38 & \textbf{75.30} & \textbf{23.80} \\

\rowcolor{yesgreen!8}
\cellcolor{white}& \method
& \textbf{71.31} & \textbf{56.53} & 66.34 & \textbf{41.87}
& 80.50 & \textbf{1291.41} & \textbf{59.96} & 55.90 & 19.80 \\

\bottomrule
\end{tabular}%
}
\end{table*}

\begin{figure*}[t]
    \centering





    \begin{subfigure}[t]{0.48\textwidth}
        \centering
        \includegraphics[width=\linewidth]{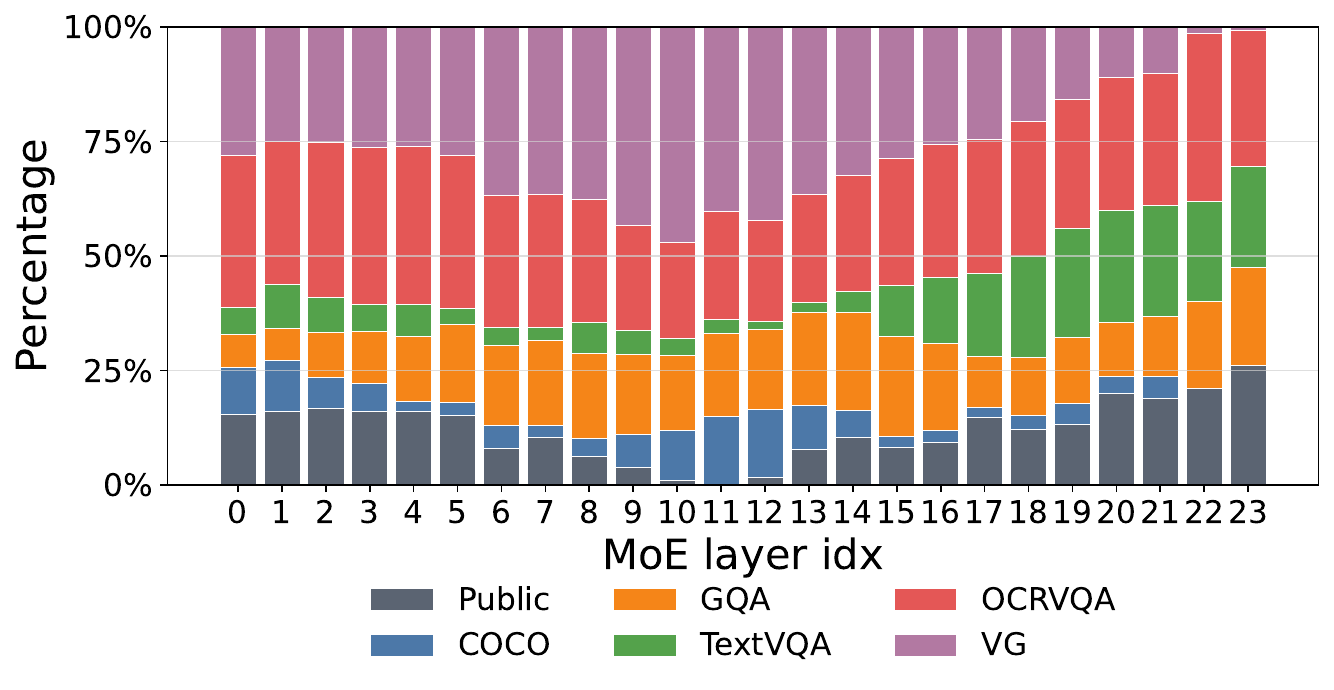}
        \caption{FlexOlmo (perception score: 1176.35).}
        \label{fig:mme_distmoe_routing_text}
    \end{subfigure}
    \hfill
    \begin{subfigure}[t]{0.48\textwidth}
        \centering
        \includegraphics[width=\linewidth]{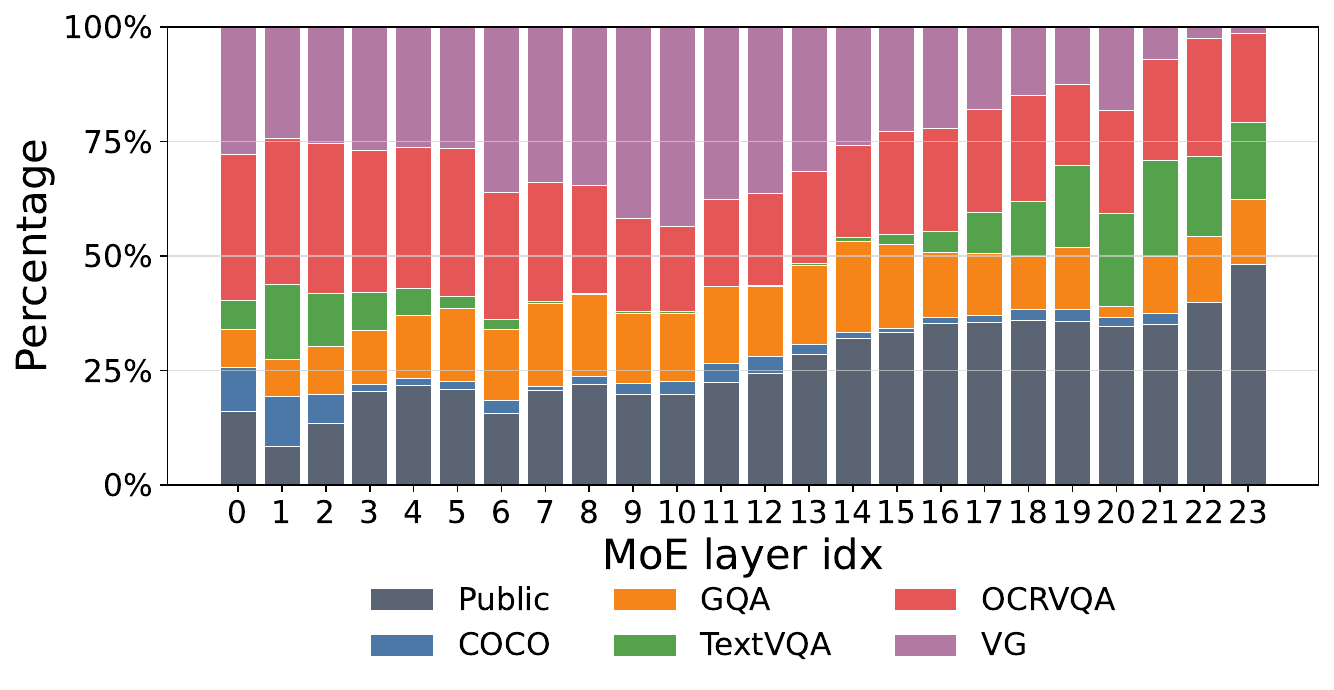}
        \caption{\method (perception score: 1188.56).}
        \label{fig:mme_distmoe_routing_image}
    \end{subfigure}

    \caption{Top-4 routing mass on MME with Qwen-1.8B, computed as the aggregate normalized router weight assigned to each expert. 
    Compared with FlexOlmo, \method assigns more routing mass to the public expert while activating private experts, supporting the effect of Stage-II.}
    \label{fig:routing_mme}
    \vspace{-0.3cm}
\end{figure*}

\vspace{-0.1cm}
\subsection{Results}
\vspace{-0.1cm}

Tab.~\ref{tab:main_all_backbones} shows that \method achieves strong performance across both image question answering and instruction-following MLLM benchmarks. 
Across Qwen-1.8B and Phi-2.7B, \method generally provides the strongest results on image QA benchmarks, indicating that public-anchored routing and residual calibration improve the reuse of independently trained visual-language experts without centralized training. 
\method-LoRA is also highly competitive despite its lightweight adaptation, and performs especially well on instruction-following benchmarks such as LLaVA$^{\mathrm{W}}$ and MM-Vet.
FlexOlmo-RT improves over rehearsal-free FlexOlmo by fine-tuning the merged router with pooled private rehearsal data, confirming the value of cross-client router supervision but violating our cross-client rehearsal-free composition setting. 
Compared with rehearsal-free FlexOlmo, \method benefits from Stage-II public-private calibration, suggesting that flexible expert composition for MLLMs requires calibration mechanisms tailored to multimodal representations rather than a direct transfer of text-only routing.
Fig.~\ref{fig:routing_mme} further illustrates this behavior.
FlexOlmo assigns relatively less routing mass to the public expert and relies more heavily on private experts, which can favor client-specific specialization but weakens the shared public expert that carries broadly trained visual-language knowledge. 
In contrast, \method assigns substantially more mass to the public expert, especially in later layers, while still activating private experts when useful. 
This behavior is consistent with the residual view of \method: private experts contribute calibrated shifts around a shared public anchor, rather than replacing the public expert.
As a result, \method better preserves out-of-domain and general instruction-following capability while retaining client-specific expert reuse. \\
Model merging baselines such as Model Soup, BTM, and BTX remain competitive on some QA tasks, but their static parameter-level combination lacks token-wise expert selection and modular control over client-specific knowledge. 
Offline reference models provide upperbound comparisons, but they assume centralized training or joint access to data and therefore do not address the distributed, cross-client rehearsal-free composition setting targeted by \method.

Appendices~\ref{sec:additional_results} and~\ref{sec:routing_analysis} report additional benchmark results and routing visualizations.

\subsection{Ablation Studies}
\label{sec:ablations}

\begin{table}[!t]
\centering
\begin{minipage}[t]{0.48\textwidth}
  \centering
  \footnotesize
  \caption{Effect of Stage-II components in \method. Stage-I is included as a reference.}
  \label{tab:components}
  \setlength{\tabcolsep}{3pt}       
  \renewcommand{\arraystretch}{1}
  \begin{tabular}{c ccc cccc}
    \toprule
    Stage & $T(\cdot)$ & $\mathcal{L}_{\mathrm{iso}}$ & $b_{m}$
    & GQA & SQA & VQA$^{\rm T}$ & POPE \\
    \midrule
    I & -- & -- & -- & 54.26 & 60.85 & 41.70 & 85.98 \\
    \midrule
    II & \xmark & \xmark & \xmark & 54.21 & 60.42 & 41.71 & 85.54 \\
    II & \cmark & \xmark & \xmark & 54.39 & 60.55 & 41.88 & 85.67 \\
    II & \cmark & \cmark & \xmark & 54.62 & 60.71 & 42.08 & 85.82 \\
    II & \xmark & \xmark & \cmark & 54.93 & 60.96 & 42.52 & 86.08 \\
    II & \cmark & \xmark & \cmark & 55.16 & 60.88 & 42.35 & 86.24 \\
    \rowcolor{yesgreen!8}
    II & \cmark & \cmark & \cmark & \textbf{55.35} & \textbf{61.03} & \textbf{42.66} & \textbf{86.38} \\
    \bottomrule
  \end{tabular}
\end{minipage}
\hfill
\begin{minipage}[t]{0.48\textwidth}
  \centering
  \footnotesize
  \caption{Effect of removing individual experts from \method.}
  \label{tab:experts}
  \setlength{\tabcolsep}{5pt}       
  \renewcommand{\arraystretch}{1}
  \begin{tabular}{l cccc}
    \toprule
    Expert & GQA & SQA & VQA$^{\rm T}$ & POPE \\
    \midrule
    Public model & 13.75 & 58.50 & 17.72 & 71.17 \\
    \midrule
    w/o COCO    & 54.51	&60.29	&41.14	&\textbf{86.93} \\
    w/o GQA     & 43.37	&61.53	&34.82	&85.95 \\
    w/o OCRVQA  & \textbf{56.28}	&58.55	&40.66	&86.80 \\
    w/o VG      & 56.23	&59.94	&41.81	&86.14 \\
    w/o TextVQA & 55.69	&\textbf{62.42}	&\textbf{45.19}	&84.37 \\
    \rowcolor{yesgreen!8}
    \method        & 55.35 & 61.03 & 42.66 & 86.38 \\
    \bottomrule
  \end{tabular}
\end{minipage}
\vspace{-0.5cm}
\end{table}
\begin{table*}[t]
\centering
\vspace{0.2cm}
\begin{minipage}[t]{0.56\textwidth}
  \centering
  \caption{Residual diagnostics on Phi-2.7B computed on the 19k public inputs. Metrics are defined in Appendix~\ref{sec:cr_additional}. Last row: mean relative change over three seeds.}
  \label{tab:residual}
  \footnotesize
  \setlength{\tabcolsep}{3pt}
  \begin{tabular}{l cccc}
    \toprule
    Condition & RMS-CV$\downarrow$ & Max/min$\downarrow$ & Eff.\ rank$\uparrow$ & SWD$\downarrow$ \\
    \midrule
    Stage-I                             & 0.1812 & 1.577 & 0.234 & 0.1012 \\
    Stage-II w/o $\mathcal{L}_{\mathrm{iso}}$ & 0.0094 & 1.022 & 0.650 & 0.1587 \\
    Stage-II w/ $\mathcal{L}_{\mathrm{iso}}$  & 0.0092 & 1.023 & 0.652 & 0.1550 \\
    \midrule
    Rel.\ change (3 seeds)              & $-$5.30\% & $-$0.10\% & +0.25\% & $-$1.76\% \\
    \bottomrule
  \end{tabular}
\end{minipage}
\hfill
\begin{minipage}[t]{0.40\textwidth}
  \centering
  \caption{Router composition on Phi-2.7B: averaged vs.\ gauge-fixed (Eq.~\eqref{eq:gauge_fixed}), without retraining.}
  \label{tab:gauge}
  \footnotesize
  \setlength{\tabcolsep}{4pt}
  \begin{tabular}{l cc}
    \toprule
    Method (composition) & VQA$^{\mathrm{T}}$ & POPE \\
    \midrule
    FlexOlmo              & 41.63 & 74.98 \\
    \method (averaged)    & 41.87 & 80.50 \\
    \rowcolor{yesgreen!8}
    \method (gauge-fixed) & 41.34 & 82.19 \\
    \bottomrule
  \end{tabular}
\end{minipage}

\vspace{0.3cm}
\caption{Inference profiling on POPE with top-$k$ routing on the same hardware: latency (s/sample), compute (TFLOPs/sample), peak memory (GB), and trainable parameters per client (B).}
\vspace{0.3cm}
\label{tab:efficiency}
\footnotesize
\setlength{\tabcolsep}{3.5pt}
\resizebox{\textwidth}{!}{%
\begin{tabular}{ll cccc>{\hspace{8pt}}cccc}
  \toprule
  & & \multicolumn{4}{c}{Full FFN experts} & \multicolumn{4}{c}{LoRA experts} \\
  \cmidrule(lr){3-6} \cmidrule(lr){7-10}
  LLM & Method & Latency & TFLOPs & Memory & Params & Latency & TFLOPs & Memory & Params \\
  \midrule
  \multirow{2}{*}{Qwen-1.8B}
    & FlexOlmo & 0.120 & 15.02 & 13.21 & 0.81 & 0.893 & 177.65 & 37.80 & 0.14 \\
    & \method  & 0.126 & 16.29 & 14.21 & 0.91 & 0.945 & 188.94 & 38.68 & 0.24 \\
  \midrule
  \multirow{2}{*}{Phi-2.7B}
    & FlexOlmo & 0.205 & 28.99 & 23.53 & 1.68 & 1.792 & 304.53 & 59.27 & 0.21 \\
    & \method  & 0.220 & 31.65 & 25.63 & 1.89 & 1.864 & 314.83 & 62.32 & 0.42 \\
  \midrule
  \multirow{2}{*}{StableLM-1.6B}
    & FlexOlmo & 0.119 & 15.04 & 12.87 & 0.83 & 1.541 & 268.39 & 35.16 & 0.14 \\
    & \method  & 0.253 & 57.37 & 14.89 & 0.93 & 1.625 & 274.03 & 36.78 & 0.24 \\
  \bottomrule
\end{tabular}%
}
\vspace{-0.2cm}
\end{table*}



\noindent\textbf{Model components.}
Tab.~\ref{tab:components} ablates the Stage-II components of \method (Qwen-1.8B): the client-private adapter $T(\cdot)$, isotropic residual loss $\mathcal{L}_{\mathrm{iso}}$, and client-private bias $b_m$. 
The Stage-I row is the reference before Stage-II calibration; Stage-II without these components does not improve over it, showing that mixed-data router calibration alone is insufficient. 
The full configuration performs best across benchmarks, showing that residual alignment and router calibration are complementary.





\begin{wrapfigure}[8]{r}{0.4\textwidth}
\vspace{-0.6cm}
  \centering
    \includegraphics[width=0.78\linewidth]{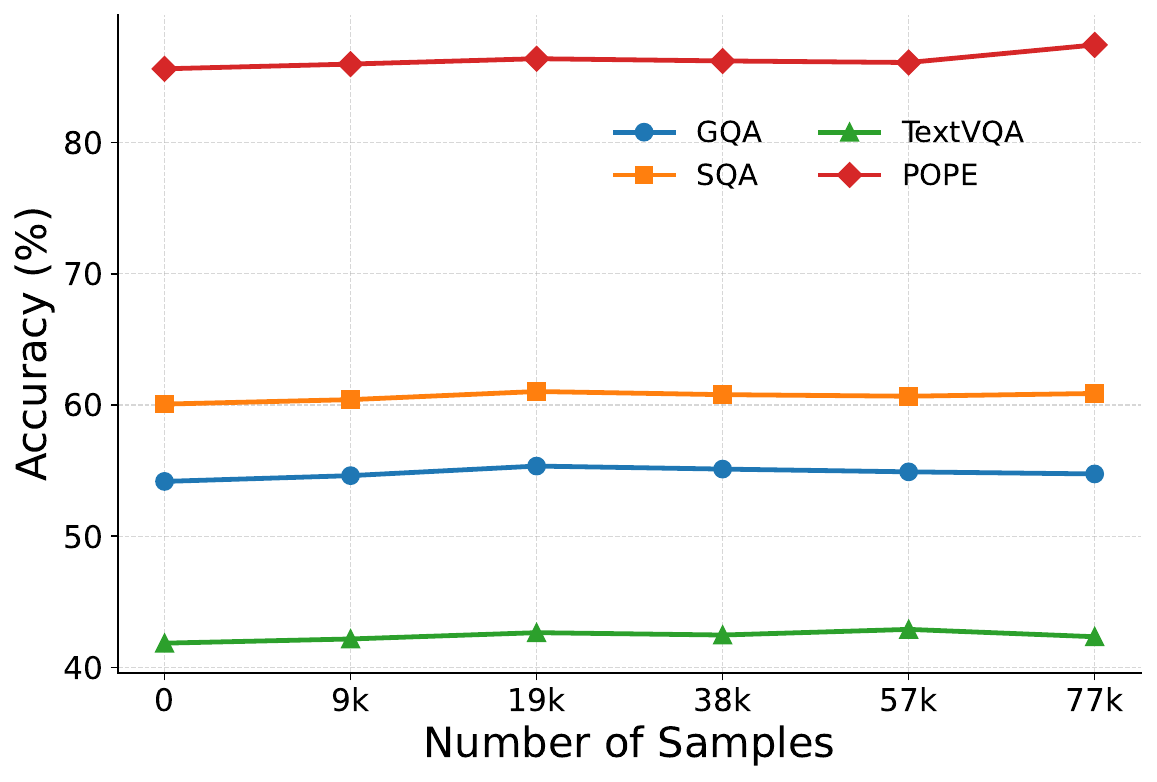}
    \vspace{-0.2cm}
    \caption{Public-data size in Stage-II.}
    \label{fig:stage2_public_data}
\end{wrapfigure}
\noindent\textbf{Effect of public-data anchoring.}
Fig.~\ref{fig:stage2_public_data} varies the number of public samples used in Stage-II (Qwen-1.8B). 
Without public data, routers and adapters are calibrated only on private data, which reduces compatibility with the shared MLLM. 
Adding public samples improves or stabilizes performance, and the gains quickly saturate, showing that a small public subset is sufficient for public-private calibration.
In \method Stage-II, we used 19k public samples.

\begin{wrapfigure}[10]{r}{0.4\textwidth}
\vspace{-0.3cm}
  \centering
    \includegraphics[width=0.78\linewidth]{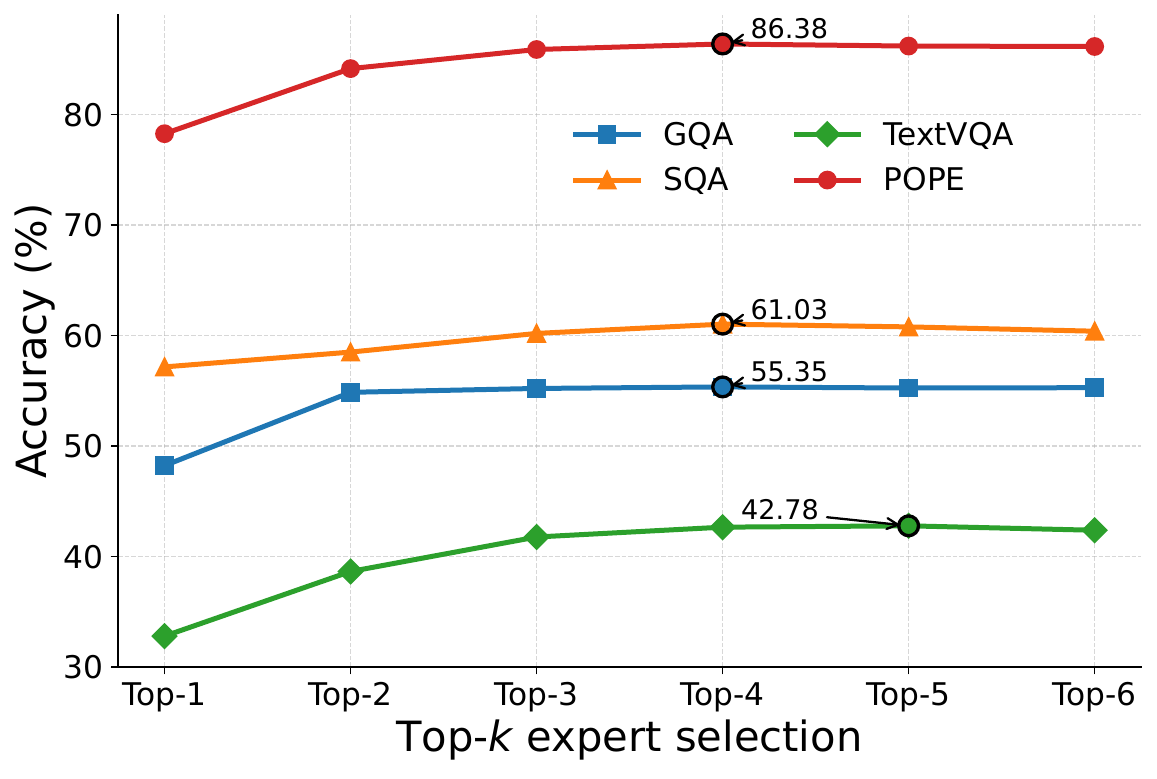}
    \vspace{-0.2cm}
    \caption{Top-$k$ routing budget.}
    \label{fig:topm_selection}
\end{wrapfigure}
\noindent\textbf{Top-$k$ expert selection.}
We study the effect of the sparse routing budget by varying the number of selected experts in  \method (Qwen-1.8B) in Fig.~\ref{fig:topm_selection}. Smaller values of $k$ encourage stronger sparsity and lower inference cost, but may underuse useful domain experts. Larger values allow more expert composition, but can introduce unnecessary computation and dilute specialization. The results show that a moderate top-$k$ routing budget provides the best trade-off, enabling \method to combine relevant private experts while maintaining efficient inference.

\noindent\textbf{Expert flexibility.}
Tab.~\ref{tab:experts} removes one client expert at a time during inference. 
All settings remain stronger than the public model, showing that \method does not rely on any single private expert. 
The impact is domain-dependent: removing GQA hurts GQA most, while removing some experts can improve certain metrics, suggesting reduced negative transfer.
Overall, \method offers a robust performance trade-off, maintaining balanced results even when individual client experts are removed.

\noindent\textbf{Residual geometry.}
Tab.~\ref{tab:residual} confirms the mismatch predicted by Proposition~\ref{prop:residual_bound}: Stage-I experts drift along their own low-dimensional subspaces with mismatched scales. Anchoring all adapters to the shared public expert removes most of it, while $\mathcal{L}_{\mathrm{iso}}$ refines the \emph{shape} rather than the scale of the residuals, hence its modest downstream effect (Appendix~\ref{sec:cr_additional}).

\noindent\textbf{Gauge-fixed router composition.}
Averaging public routes recalibrated by different clients mixes logits in different gauges and distorts each client's public--private balance; composing per-client log-odds (Eq.~\eqref{eq:gauge_fixed}) restores it and improves POPE without retraining (Tab.~\ref{tab:gauge}).

\noindent\textbf{Computational cost.}
Extra parameters over FlexOlmo come from one $d{\times}d$ adapter per layer (Tab.~\ref{tab:efficiency}). Adapters are applied only to the top-$k$ selected experts, so activated compute is set by $k$, while storage and router scoring grow with $M$. \method does not reduce latency or memory over FlexOlmo.

\vspace{-0.1cm}
\section{Conclusion}
\label{sec:conclusion}
\vspace{-0.1cm}


We presented \method, a cross-client rehearsal-free MoE framework for distributed visual instruction tuning. 
\method keeps a shared public MLLM as an anchor, learns client-specific private experts locally, and makes them directly comparable through a Stage-II calibration without using any data across clients.
At inference, it composes public and private experts via sparse routing. 
Across diverse vision-language benchmarks, \method improves over distributed merging and routing baselines, while \method-LoRA provides a competitive parameter-efficient alternative. 
These results show that public-anchored calibration can make independently trained multimodal experts reusable under distributed constraints, offering a practical path toward scalable MLLMs.

\noindent\textbf{Limitations and future work.}
Centralized or rehearsal-based MoE training can still perform better on some benchmarks, reflecting the trade-off between unrestricted joint optimization and rehearsal-free distributed adaptation. 
\method calibrate Stage-II with a small public subset, which may not be available for open-weight models \cite{casper2026open}, motivating training-free expert composition.

\begin{ack}
This project has received funding from the European Union's EU Framework Programme for Research and Innovation under the HORIZON-MSCA-DN-2023 Grant Agreement N. 101169439. The work is also supported by the EU Horizon Europe projects ELIAS (No. 101120237) and ELLIOT (No. 101214398), and the FIS project GUIDANCE (No. FIS2023-03251) and Ministero delle Imprese e del Made in Italy (IPCEI Cloud DM 27 giugno 2022 - IPCEI-CL-0000007). We also acknowledge CINECA for providing access to HPC resources on the Leonardo supercomputer through the ISCRA initiative.
\end{ack}

\bibliographystyle{unsrt}
\bibliography{main}

\clearpage
\newpage
\appendix

\section{Pseudocode Algorithm for \method}
\label{sec:algorithm}


\begin{algorithm}[!h]
\caption{\method: Distributed Visual Instruction Tuning with Cross-client Rehearsal-free Composition}
\label{alg:distmoe}
\begin{algorithmic}[1]
\Require Public MLLM backbone with public FFN layers $\{\mathcal{E}_{\mathrm{pub}}^\ell\}_{\ell=1}^{L}$ in the LLM; private datasets $\{\mathcal{D}_m\}_{m=1}^{M}$; public data $\mathcal{D}_{\mathrm{pub}}$.
\Ensure Client experts $\{\mathcal{E}_m^\ell\}$, routers $\{\mathbf{r}_{\mathrm{pub},m}^\ell,\mathbf{r}_{m}^\ell,b_{m}^\ell\}$, and adapters $\{T_m^\ell\}$.

\State Freeze the vision encoder, multimodal projector, and public FFNs $\{\mathcal{E}_{\mathrm{pub}}^\ell\}_{\ell=1}^{L}$.

\For{each client $m=1,\ldots,M$}
    \For{each layer $\ell=1,\ldots,L$}
        \State Initialize $\mathcal{E}_m^\ell \leftarrow \mathcal{E}_{\mathrm{pub}}^\ell$.
        \State Initialize $\mathbf{r}_{\mathrm{pub},m}^\ell,\mathbf{r}_{m}^\ell \leftarrow \mathbf{0}$ and $b_{m}^\ell \leftarrow 0$.
    \EndFor

    \Statex \textbf{Stage-I: Expert knowledge acquisition}
    \While{not converged}
        \State Sample a mini-batch $(I,q,y)\sim\mathcal{D}_m$.
        \For{each FFN layer $\ell$ and token hidden state $\mathbf{x}^{\ell}$}
            \State $z_{\mathrm{pub}}=(\mathbf{x}^{\ell})^\top\mathbf{r}_{\mathrm{pub},m}^{\ell}$,\quad
            $z_{m}=(\mathbf{x}^{\ell})^\top\mathbf{r}_{m}^{\ell}+b_{m}^{\ell}$.
            \State $[\alpha_{\mathrm{pub},m}^{\ell},\alpha_{m}^{\ell}]
            =\mathrm{softmax}([z_{\mathrm{pub}},z_{m}])$.
            \State $\mathbf{o}^{\ell}
            =
            \alpha_{\mathrm{pub},m}^{\ell}\mathcal{E}_{\mathrm{pub}}^{\ell}(\mathbf{x}^{\ell})
            +
            \alpha_{m}^{\ell}\mathcal{E}_{m}^{\ell}(\mathbf{x}^{\ell})$.
        \EndFor
        \State Update only $\{\mathcal{E}_{m}^{\ell}\}_{\ell=1}^{L}$ and the client routers by minimizing $\mathcal{L}_{\mathrm{NTP}}$ on $\mathcal{D}_m$.
    \EndWhile

    \Statex \textbf{Stage-II: Public-anchored composition calibration}
    \State Freeze $\{\mathcal{E}_{\mathrm{pub}}^{\ell}\}_{\ell=1}^{L}$ and $\{\mathcal{E}_{m}^{\ell}\}_{\ell=1}^{L}$.
    \State Construct $\mathcal{D}_{m}^{\mathrm{mix}}=\mathcal{D}_{m}\cup\mathcal{D}_{\mathrm{pub}}$.
    \While{not converged}
        \State Sample a mini-batch $(I,q,y)\sim\mathcal{D}_{m}^{\mathrm{mix}}$.
        \For{each FFN layer $\ell$ and token hidden state $\mathbf{x}^{\ell}$}
            \State $\mathbf{h}_{\mathrm{pub}}^{\ell}=\mathcal{E}_{\mathrm{pub}}^{\ell}(\mathbf{x}^{\ell})$,\quad
            $\widetilde{\mathbf{h}}_{m}^{\ell}
            =T_{m}^{\ell}(\mathcal{E}_{m}^{\ell}(\mathbf{x}^{\ell}))$.
            \State $z_{\mathrm{pub}}=(\mathbf{x}^{\ell})^\top\mathbf{r}_{\mathrm{pub},m}^{\ell}$,\quad
            $z_{m}=(\mathbf{x}^{\ell})^\top\mathbf{r}_{m}^{\ell}+b_{m}^{\ell}$.
            \State $[\alpha_{\mathrm{pub},m}^{\ell},\alpha_{m}^{\ell}]
            =\mathrm{softmax}([z_{\mathrm{pub}},z_{m}])$.
            \State $\mathbf{o}^{\ell}
            =
            \alpha_{\mathrm{pub},m}^{\ell}\mathbf{h}_{\mathrm{pub}}^{\ell}
            +
            \alpha_{m}^{\ell}\widetilde{\mathbf{h}}_{m}^{\ell}$.
            \State $\boldsymbol{\Delta}_{m}^{\ell}
            =
            \widetilde{\mathbf{h}}_{m}^{\ell}
            -
            \mathbf{h}_{\mathrm{pub}}^{\ell}$.
        \EndFor
        \State Compute $\mathcal{L}_{\mathrm{iso}}$ from the residuals $\{\boldsymbol{\Delta}_{m}^{\ell}\}_{\ell=1}^{L}$.
        \State Update $\{\mathbf{r}_{\mathrm{pub},m}^{\ell},\mathbf{r}_{m}^{\ell},b_{m}^{\ell},T_m^\ell\}_{\ell=1}^{L}$ by minimizing
        $\mathcal{L}_{\mathrm{Stage\text{-}II}}
        =
        \mathcal{L}_{\mathrm{NTP}}
        +
        \lambda\mathcal{L}_{\mathrm{iso}}$.
    \EndWhile
\EndFor

\Statex \textbf{Inference: Rehearsal-free modular routing}
\For{each layer $\ell$ and token hidden state $\mathbf{x}^{\ell}$}
    \State $\bar{\mathbf{r}}_{\mathrm{pub}}^{\ell}
    =
    \frac{1}{M}
    \sum_{m=1}^{M}\mathbf{r}_{\mathrm{pub},m}^{\ell}$.
    \State $z_{\mathrm{pub}}^{\ell}
    =
    (\mathbf{x}^{\ell})^\top\bar{\mathbf{r}}_{\mathrm{pub}}^{\ell}$.
    \For{each client expert $m=1,\ldots,M$}
        \State $z_{m}^{\ell}
        =
        (\mathbf{x}^{\ell})^\top\mathbf{r}_{m}^{\ell}
        +
        b_{m}^{\ell}$.
    \EndFor
    \State $\mathbf{z}^{\ell}
    =
    [z_{\mathrm{pub}}^{\ell},z_{1}^{\ell},\ldots,z_{M}^{\ell}]$.
    \State Select active routes $\mathcal{S}^{\ell}$ from $\mathbf{z}^{\ell}$ using the inference routing policy (e.g. top-$k$ routing).
    \State Normalize routing weights over $\mathcal{S}^{\ell}$ to obtain $\alpha_{\mathrm{pub}}^{\ell}$ and $\{\alpha_{m}^{\ell}\}_{m=1}^{M}$.
    \State Assign zero weight to non-selected routes.
    \State $\mathbf{o}^{\ell}
    =
    \alpha_{\mathrm{pub}}^{\ell}\mathcal{E}_{\mathrm{pub}}^{\ell}(\mathbf{x}^{\ell})
    +
    \sum_{m=1}^{M}
    \alpha_{m}^{\ell}
    T_{m}^{\ell}
    \left(\mathcal{E}_{m}^{\ell}(\mathbf{x}^{\ell})\right)$.
\EndFor
\end{algorithmic}
\end{algorithm}

\newpage
\section{Derivation of the Residual Influence Bound}
\label{app:residual_bound}

For clarity, we omit the layer index and consider one token representation. 
The same derivation applies independently at every layer. 
At inference, routing is performed over $M+1$ routes: one shared public route and $M$ private routes. 
Top-$K$ selection assigns zero weight to non-selected routes, so the normalized weights satisfy
$\alpha_{\mathrm{pub}}+\sum_{m=1}^{M}\alpha_m=1$.

\subsection{Mathematical derivation of \texorpdfstring{Eq.~\eqref{eq:residual_decomposition}}{Eq. (10)}}
\label{sec:math_derivation}

The inference-time routed output is
\begin{equation}
    \mathbf{o}
    =
    \alpha_{\mathrm{pub}}
    \mathbf{h}_{\mathrm{pub}}
    +
    \sum_{m=1}^{M}
    \alpha_m
    \widetilde{\mathbf{h}}_{m}.
\end{equation}
The routing weights are normalized over the selected routes, with non-selected routes assigned zero weight:
\begin{equation}
    \alpha_{\mathrm{pub}}
    +
    \sum_{m=1}^{M}
    \alpha_m
    =
    1.
\end{equation}
For each private expert $m$, the client residual is defined as
\begin{equation}
    \boldsymbol{\Delta}_{m}
    =
    \widetilde{\mathbf{h}}_{m}
    -
    \mathbf{h}_{\mathrm{pub}}.
\end{equation}
Equivalently,
\begin{equation}
    \widetilde{\mathbf{h}}_{m}
    =
    \mathbf{h}_{\mathrm{pub}}
    +
    \boldsymbol{\Delta}_{m}.
\end{equation}
Substituting this expression into the routed output gives
\begin{align}
    \mathbf{o}
    &=
    \alpha_{\mathrm{pub}}
    \mathbf{h}_{\mathrm{pub}}
    +
    \sum_{m=1}^{M}
    \alpha_m
    \left(
        \mathbf{h}_{\mathrm{pub}}
        +
        \boldsymbol{\Delta}_{m}
    \right) \\
    &=
    \left(
        \alpha_{\mathrm{pub}}
        +
        \sum_{m=1}^{M}
        \alpha_m
    \right)
    \mathbf{h}_{\mathrm{pub}}
    +
    \sum_{m=1}^{M}
    \alpha_m
    \boldsymbol{\Delta}_{m}.
\end{align}
Using the normalization of the routing weights, we obtain
\begin{equation}
    \mathbf{o}
    =
    \mathbf{h}_{\mathrm{pub}}
    +
    \sum_{m=1}^{M}
    \alpha_m
    \boldsymbol{\Delta}_{m},
\end{equation}
which is Eq.~\eqref{eq:residual_decomposition}.

\subsection{Proof of Proposition~\ref{prop:residual_bound}}
\label{sec:proof_prop}

\begin{proposition*}[\ref{prop:residual_bound}]
For the routed output in Eq.~\eqref{eq:residual_decomposition}, the deviation from the shared public expert is bounded by the routing-weighted residual magnitudes:
\begin{equation}
    \left\|
        \mathbf{o}
        -
        \mathbf{h}_{\mathrm{pub}}
    \right\|_{2}
    \leq
    \sum_{m=1}^{M}
    \alpha_m
    \left\|
        \boldsymbol{\Delta}_{m}
    \right\|_{2}.
\end{equation}
\end{proposition*}

\begin{proof}
From Eq.~\eqref{eq:residual_decomposition}, we have
\begin{equation}
    \mathbf{o}
    -
    \mathbf{h}_{\mathrm{pub}}
    =
    \sum_{m=1}^{M}
    \alpha_m
    \boldsymbol{\Delta}_{m}.
\end{equation}
Taking the Euclidean norm on both sides gives
\begin{equation}
    \left\|
        \mathbf{o}
        -
        \mathbf{h}_{\mathrm{pub}}
    \right\|_{2}
    =
    \left\|
        \sum_{m=1}^{M}
        \alpha_m
        \boldsymbol{\Delta}_{m}
    \right\|_{2}.
\end{equation}
Applying the triangle inequality,
\begin{equation}
    \left\|
        \sum_{m=1}^{M}
        \alpha_m
        \boldsymbol{\Delta}_{m}
    \right\|_{2}
    \leq
    \sum_{m=1}^{M}
    \left\|
        \alpha_m
        \boldsymbol{\Delta}_{m}
    \right\|_{2}.
\end{equation}
Since routing weights are non-negative, $\alpha_m\geq 0$, we have
\begin{equation}
    \left\|
        \alpha_m
        \boldsymbol{\Delta}_{m}
    \right\|_{2}
    =
    \alpha_m
    \left\|
        \boldsymbol{\Delta}_{m}
    \right\|_{2}.
\end{equation}
Therefore,
\begin{equation}
    \left\|
        \mathbf{o}
        -
        \mathbf{h}_{\mathrm{pub}}
    \right\|_{2}
    \leq
    \sum_{m=1}^{M}
    \alpha_m
    \left\|
        \boldsymbol{\Delta}_{m}
    \right\|_{2},
\end{equation}
which proves the proposition.
\end{proof}

\section{Additional Experimental Details}
\label{sec:additional_experimental_details}

\subsection{Dataset Details}
\label{sec:dataset_details}
We construct a distributed visual-instruction tuning setup with five private client domains: COCO~\cite{lin2014microsoft}, GQA~\cite{hudson2019gqa}, OCRVQA~\cite{8978122}, VQA$^{\mathrm{T}}$~\cite{singh2019towards}, and Visual Genome (VG)~\cite{krishna2017visual}. In Stage-I, each client is trained only on its own private dataset, without accessing data from other clients. In Stage-II, we keep the client-specific dataset for each client and add a small public subset for router and adapter calibration. The public subset contains $19$k samples drawn from the Hybrid-FT data group used in MoE-LLaVA, including SViT-157k~\cite{zhao2023svit}, LVIS-220k~\cite{wang2023see}, LRV-331k~\cite{liu2024mitigating} and MIMIC-IT-256k~\cite{li2023mimic}. In addition, the public MLLM backbone i.e. Stage-2 MoE-LLaVA checkpoint is already pretrained on approximately $1.5$M multimodal instruction data including LLaVA 1.5-558k~\cite{liu2024improved} and the Hybrid-FT data group, while \method only performs distributed client-specific adaptation on top of this backbone. Tab.~\ref{tab:training_data} summarizes the training data used for public model initialization, private and public data distribution. Below, we also discuss the test dataset along with the evaluation metrics used in our experiments.

\subsection{Evaluation Metrics}
\label{sec:evaluation_metrics}

\noindent\textbf{Academic-task benchmark.} In Tab.~\ref{tab:qwen_img}, ~\ref{tab:phi_img}, ~\ref{tab:stablelm_img}, we use the following multiple evaluation metrics of each dataset. 
(i) \textbf{VQA$^{\mathrm{v2}}$}~\cite{goyal2017making} -
Accuracy for \textit{Yes/No}, \textit{Number}, and \textit{Other} questions, together with \textit{Overall} accuracy, which evaluate binary recognition, numerical reasoning, and open-ended visual question answering respectively are reported.
(ii) \textbf{GQA}~\cite{hudson2019gqa} - We report \textit{Binary}, \textit{Open}, and overall accuracy (\textit{Acc.}), along with the distribution score (\textit{Dist.}). Binary and open accuracy measure closed-form and open-ended compositional reasoning, while lower \textit{Dist.} indicates less answer-distribution bias.
(iii) \textbf{SQA}: ScienceQA~\cite{lu2022learn} -
We use standard answer accuracy (\textit{Acc.}) and image-grounded accuracy (\textit{IMG-Acc.}). The latter focuses on questions requiring visual evidence rather than language priors. 
(iv) \textbf{VQA$^{\mathrm{T}}$}: TextVQA~\cite{singh2019towards} - We report answer accuracy, which evaluates the model's ability to read and reason over text appearing in images.
However, in main Tab.~\ref{tab:main_all_backbones}, we showcase only overall accuracy for VQA$^{\mathrm{v2}}$ GQA, and VQA$^{\mathrm{T}}$, whereas image-grounded accuracy for SQA dataset.

\begin{table}[!t]
\centering
\caption{\textbf{Training data used in \method.} Stage-I trains each client expert using only its own private dataset. Stage-II uses the same client data together with a small public subset sampled from the Hybrid-FT data group.}
\label{tab:training_data}
\renewcommand{\arraystretch}{1.15}
\setlength{\tabcolsep}{5pt}
\resizebox{\linewidth}{!}{
\begin{tabular}{l|c|c|c}
\toprule
\textbf{Data group} & \textbf{Usage} & \textbf{Source} & \textbf{\# Samples} \\
\midrule
Public backbone pretraining 
& Initialization 
& \begin{tabular}[c]{@{}c@{}}
LLaVA-PT group: LLaVA 1.5-558k \\
Hybrid-FT group: SViT-157k, LVIS-220k, \\
LRV-331k, MIMIC-IT-256k
\end{tabular}
& $\sim$1.5M \\
\midrule
\multirow{5}{*}{Private client data}
& Stage-I / II for $\mathcal{E}_{1}$
& COCO 
& 364k \\
& Stage-I / II for $\mathcal{E}_{2}$
& GQA 
& 72k \\
& Stage-I / II for $\mathcal{E}_{3}$
& OCRVQA 
& 80k \\
& Stage-I / II for $\mathcal{E}_{4}$
& VQA$^{\rm T}$ 
& 22k \\
& Stage-I / II for $\mathcal{E}_{5}$
& Visual Genome (VG) 
& 86k \\
\midrule
Public data
& Stage-II only
& \begin{tabular}[c]{@{}l@{}}
Subset of Hybrid-FT group: SViT-157k,\\
LVIS-220k, LRV-331k, MIMIC-IT-256k
\end{tabular}
& 19k \\
\bottomrule
\end{tabular}
}
\end{table}

\noindent\textbf{Instruction-following benchmark.} In Tab.~\ref{tab:qwen_bench}, ~\ref{tab:phi_bench}, ~\ref{tab:stablelm_bench}, we use the following multiple evaluation metrics of each dataset.
(i) \textbf{POPE}~\cite{li2023evaluating} - We report accuracy under \textit{Random}, \textit{Popular}, and \textit{Adversarial} settings, together with their average, that evaluate object hallucination robustness under increasingly challenging negative samples.
(ii) \textbf{MME}~\cite{fu2026mme} - We report \textit{Perception} and \textit{Cognition} scores. Perception measures visual understanding such as recognition, counting, position, and OCR, while cognition measures higher-level multimodal reasoning.
(iii) \textbf{SEED-Bench}~\cite{li2024seed} -
We report \textit{All}, \textit{Image}, and \textit{Video} scores, measuring overall performance, static-image understanding, and video-related reasoning.
(iv) \textbf{LLaVA$^{\mathrm{W}}$}: LLaVA-Bench (in-the-Wild)~\cite{llava} - We report \textit{All}, \textit{Complex Reasoning} (\textit{Comp.}), \textit{Conversation} (\textit{Conv.}), and \textit{Detailed Description} (\textit{Det.}) scores for open-ended multimodal instruction following.
(v) \textbf{MM-Vet}~\cite{pmlr-v235-yu24o} - We report overall accuracy, which evaluates integrated multimodal abilities such as recognition, OCR, spatial reasoning, knowledge reasoning, and language generation.
However, in main Tab.~\ref{tab:main_all_backbones}, we report only average accuracy for POPE, perception score for MME, \textit{All} for SEED-Bench and LLaVA$^{\mathrm{W}}$, and overall accuracy for MM-Vet dataset.

\begin{table*}[t]
\centering
\caption{\textbf{Zero-shot object hallucination evaluation in POPE} dataset for \method with Qwen-1.8B, Phi-2.7B, and StableLM-1.6B. We report accuracy, F1-score, and the yes-answer ratio under Random, Popular, and Adversarial settings. ``Yes'' indicates the proportion of positive responses to the given question.}
\label{tab:pope_all_backbones}
\setlength{\tabcolsep}{4pt}
\footnotesize
\resizebox{\textwidth}{!}{%
\begin{tabular}{ll ccc>{\hspace{8pt}} ccc>{\hspace{8pt}} ccc}
\toprule
\textbf{LLM} & \textbf{Method}
& \multicolumn{3}{c}{\textbf{Random}}
& \multicolumn{3}{c}{\textbf{Popular}}
& \multicolumn{3}{c}{\textbf{Adversarial}} \\
\cmidrule(lr){3-5} \cmidrule(lr){6-8} \cmidrule(lr){9-11}
& & Acc. $\uparrow$ & F1 $\uparrow$ & Yes $\downarrow$
& Acc. $\uparrow$ & F1 $\uparrow$ & Yes $\downarrow$
& Acc. $\uparrow$ & F1 $\uparrow$ & Yes $\downarrow$ \\
\midrule

\cellcolor{white}\multirow{5}{*}{Qwen-1.8B}
& Public model
& 68.38 & 75.57 & 77.87
& 75.90 & 79.74 & 68.97
& 69.23 & 75.51 & 75.63 \\

& FlexOlmo-LoRA
& 86.80 & 86.48 & 46.11
& 86.63 & 85.54 & 45.30
& 82.63 & 82.51 & 49.30 \\

& FlexOlmo
& 86.84 & 85.75 & \textbf{40.79}
& 86.26 & 84.83 & \textbf{40.53}
& \textbf{84.83} & 82.99 & \textbf{42.53} \\

\rowcolor{yesgreen!8}
\cellcolor{white}
& \method-LoRA
& 87.08 & 86.43 & 43.64
& 86.67 & 85.68 & 43.13
& 83.30 & 82.69 & 46.50 \\

\rowcolor{yesgreen!8}
\cellcolor{white}
& \method
& \textbf{88.18} & \textbf{87.51} & 43.09
& \textbf{86.73} & \textbf{85.83} & 43.60
& 84.23 & \textbf{83.59} & 46.10 \\

\midrule

\cellcolor{white}\multirow{5}{*}{Phi-2.7B}
& Public model
& 83.54	&84.89	&57.42	&79.37	&81.30	&60.37	&78.53	&80.69	&61.20 \\

& FlexOlmo-LoRA
& 83.54 & 84.85 & 57.42
& 79.43 & 81.35 & 60.30
& 78.63 & 80.77 & 61.10 \\

& FlexOlmo
& 74.91 & 67.98 & \textbf{26.80}
& 75.30 & 67.66 & \textbf{26.37}
& 74.73 & 67.15 & \textbf{26.93} \\

\rowcolor{yesgreen!8}
\cellcolor{white}
& \method-LoRA
& \textbf{83.95} & \textbf{84.89} & 57.43
& 79.82 & \textbf{81.42} & 60.37
& 78.76 & \textbf{80.92} & 61.34 \\

\rowcolor{yesgreen!8}
\cellcolor{white}
& \method
& 80.86 & 77.39 & 33.09
& \textbf{80.80} & 76.79 & 32.73
& \textbf{79.83} & 75.91 & 33.70 \\

\midrule

\cellcolor{white}\multirow{5}{*}{StableLM-1.6B}
& Public model
& 58.76 & 71.32 & 92.23
& 53.80 & 68.28 & 95.67
& 53.37 & 68.08 & 96.10 \\

& FlexOlmo-LoRA
& 75.62 & 70.37 & \textbf{28.56}
& 77.45 & 70.15 & \textbf{28.78}
& 76.04 & 70.83 & \textbf{30.05} \\

& FlexOlmo
& 77.84 & 72.82 & 30.00
& 78.07 & 72.42 & 29.53
& 76.87 & 71.34 & 30.73 \\

\rowcolor{yesgreen!8}
\cellcolor{white}
& \method-LoRA
& 84.06 & 82.45 & 37.96
& 83.92 & 81.23 & 37.65
& 81.45 & 78.19 & 38.20 \\

\rowcolor{yesgreen!8}
\cellcolor{white}
& \method
& \textbf{84.98} & \textbf{83.29} & 38.32
& \textbf{84.27} & \textbf{82.19} & 38.33
& \textbf{82.27} & \textbf{80.37} & 40.33 \\

\bottomrule
\end{tabular}%
}
\vspace{-0.3cm}
\end{table*}

\subsection{Additional Implementation Details}
\label{sec:implementation_details}
We implement \method on three MoE-LLaVA-Stage2 backbones, using Qwen-1.8B, StableLM-1.6B, and Phi-2.7B as the LLMs, with CLIP ViT-L/14 as the shared vision encoder, in which the input image resolution is set to $336 \times 336$. For each backbone, we replace every LLM FFN with a routed public-private FFN, while keeping the vision encoder and multimodal projector frozen. In the full fine-tuning setting, each private expert is initialized as a full copy of the corresponding public FFN. In the LoRA setting, each client uses $R=16$ LoRA experts on the linear projection matrices, with rank $r=16$ and scaling factor $\alpha=32$. Stage-II freezes all private experts and updates only the routers, router biases, and lightweight adapter $T_m(\cdot)$, which is taken as a linear layer of same input-output dimension of the FFN blocks of the corresponding LLMs. We use top-$k$ routing with $k=4$ for full experts and $k=4R=64$ for LoRA experts. For expert composition, we use $|\mathcal{D}_{\mathrm{pub}}|=19\mathrm{k}$ public samples, and $\lambda=10^{-4}$ in Eq.~\eqref{eq:stage2_objective}. All experiments are conducted on $4$ NVIDIA A100 GPUs with $64$GB memory each. We train for $1$ epoch using AdamW with learning rate $2\times10^{-5}$, cosine learning-rate scheduling, warmup ratio $0.03$, and per-device batch size $4$ in both of the training stages.

\begin{table*}[!t]
  \centering
\caption{\textbf{Detailed academic-task benchmark results with Qwen-1.8B.} 
We report fine-grained results on VQA$^{\mathrm{v2}}$, GQA, SQA, and VQA$^{\mathrm{T}}$. 
The Public model denotes the shared backbone before client-specific expert composition. 
MoE-LLaVA$^\dagger$ denotes the original results reported in~\cite{moellava}, while MoE-LLaVA denotes the offline MoE model trained with joint access to the same client datasets used in our distributed setting. }
  \renewcommand{\arraystretch}{1.1}
  \setlength{\tabcolsep}{3.6pt}
  \scriptsize
  \resizebox{\linewidth}{!}{%
  \begin{tabular}{l cccc>{\hspace{8pt}} cccc>{\hspace{8pt}} cc>{\hspace{8pt}} c}
    \toprule
    \multirow{2}{*}{\textbf{Method}}
      & \multicolumn{4}{c}{\textbf{VQA$^{\mathrm{v2}}$}}
      & \multicolumn{4}{c}{\textbf{GQA}}
      & \multicolumn{2}{c}{\textbf{SQA}}
      & \multicolumn{1}{c}{\textbf{VQA$^{\mathrm{T}}$}} \\
    \cmidrule(lr){2-5}
    \cmidrule(lr){6-9}
    \cmidrule(lr){10-11}
    \cmidrule(lr){12-12}
      & Yes/No $\uparrow$
      & Num. $\uparrow$
      & Other $\uparrow$
      & Overall $\uparrow$
      & Binary $\uparrow$
      & Open $\uparrow$
      & Acc. $\uparrow$
      & Dist. $\downarrow$
      & Acc. $\uparrow$
      & IMG-Acc. $\uparrow$
      & Acc. $\uparrow$ \\
    \midrule

    \multicolumn{12}{l}{\textbf{\textit{Offline reference models (upperbound)}}} \\
    Qwen-VL
      & \ub{--} & \ub{--} & \ub{--} & \ub{78.80}
      & \ub{--} & \ub{--} & \ub{59.30} & \ub{--}
      & \ub{--} & \ub{67.10}
      & \ub{63.80} \\

    MoE-LLaVA$^\dagger$
      & \ub{--} & \ub{--} & \ub{--} & \ub{76.20}
      & \ub{--} & \ub{--} & \ub{61.50} & \ub{--}
      & \ub{--} & \ub{63.10}
      & \ub{48.00} \\
      
    MoE-LLaVA
      & \softmask{88.26} & \softmask{53.46} & \softmask{63.10} & \softmask{72.81}
      & \softmask{73.56} & \softmask{44.41} & \softmask{57.29} & \softmask{2.07}
      & \softmask{64.37} & \softmask{63.22}
      & \softmask{44.15} \\

    \cmidrule(lr){1-12}
    \multicolumn{12}{l}{\textbf{\textit{Individual experts}}} \\
    COCO
      & 84.41 & 51.84 & 58.28 & 69.94
      & 62.01 & 34.03 & 46.88 & 2.34
      & 60.56 & 59.24
      & 35.51 \\

    GQA
      & 82.65 & 46.47 & 56.98 & 66.36
      & 68.15 & 37.38 & 51.62 & 2.62
      & 38.91 & 29.75
      & 32.27 \\

    VQA$^{\text{OCR}}$
      & 70.10 & 29.08 & 45.21 & 53.63
      & 39.29 & 17.13 & 27.30 & 2.41
      & 60.25 & 59.11
      & 36.16 \\

    VG
      & 0.01 & 0.16 & 6.00 & 2.87
      & 0.12 & 0.84 & 0.51 & 4.23
      & 41.22 & 37.83
      & 16.62 \\

    VQA$^{\mathrm{T}}$
      & 4.27 & 1.88 & 39.09 & 20.52
      & 9.28 & 16.31 & 13.09 & 2.54
      & 56.87 & 58.75
      & 0.88 \\

    \cmidrule(lr){1-12}

    Public model
      & 5.65 & 1.76 & 38.85 & 20.96
      & 10.03 & 16.91 & 13.75 & 2.45
      & 55.34 & 58.50
      & 17.72 \\
      
    Model Soup (avg)
      & 84.92 & 50.61 & 59.18 & 68.92
      & 69.54 & 39.52 & 53.18 & 2.35
      & 60.52 & 59.98
      & 40.83 \\

    Model Soup (wtd)
      & 85.21 & 50.94 & 59.43 & 69.18
      & 69.81 & 39.86 & 53.46 & 2.33
      & 60.83 & 60.21
      & 41.18 \\

    BTM
      & 85.47 & 51.22 & 59.72 & 69.73
      & 70.08 & 40.18 & 53.81 & 2.28
      & 61.36 & 60.66
      & 41.87 \\

    BTX
      & 85.74 & 51.48 & 60.03 & 69.46
      & 70.41 & 40.54 & 54.16 & 2.28
      & 61.12 & 60.43
      & 41.52 \\

    FlexOlmo-LoRA
      & 74.39 & 45.25 & 52.94 & 63.45
      & 61.65 & 40.63 & 50.28 & 1.85
      & 59.33 & 59.15
      & 34.46 \\

    FlexOlmo
      & 85.15 & 51.67 & 60.04 & 70.01
      & 70.32 & 40.50 & 54.26 & 2.26
      & 61.68 & 60.85
      & 41.70 \\

    \rowcolor{yesgreen!8}
    \method-LoRA
      & 77.28 & 47.16 & 55.47 & 65.80
      & 62.03 & 40.73 & 50.51 & \textbf{1.79}
      & 59.21 & 59.30
      & 33.69 \\

    \rowcolor{yesgreen!8}
    \method
      & \textbf{86.38} & \textbf{52.04} & \textbf{60.64} & \textbf{70.27}
      & \textbf{71.25} & \textbf{41.87} & \textbf{55.35} & 2.11
      & \textbf{61.94} & \textbf{61.03}
      & \textbf{42.66} \\

    \bottomrule
  \end{tabular}%
  }
  \label{tab:qwen_img}
\end{table*}
\begin{table*}[!th]
  \centering
  \caption{\textbf{Detailed instruction-following benchmark results with Qwen-1.8B.} 
We report fine-grained results on POPE, MME, SEED-Bench, LLaVA$^{\mathrm{W}}$, and MM-Vet. 
The Public model denotes the shared backbone before client-specific expert composition. 
MoE-LLaVA$^\dagger$ denotes the original results reported in~\cite{moellava}, while MoE-LLaVA denotes the offline MoE model trained with joint access to the same client datasets used in our distributed setting. }
  \renewcommand{\arraystretch}{1.1}
  \setlength{\tabcolsep}{3.6pt}
  \scriptsize
  \resizebox{\linewidth}{!}{%
  \begin{tabular}{l cccc>{\hspace{8pt}} cc>{\hspace{8pt}} ccc>{\hspace{8pt}} cccc>{\hspace{8pt}} c}
    \toprule
    \multirow{2}{*}{\textbf{Method}} 
    & \multicolumn{4}{c}{\textbf{POPE}}
      & \multicolumn{2}{c}{\textbf{MME}}
      & \multicolumn{3}{c}{\textbf{SEED-Bench}}
      & \multicolumn{4}{c}{\textbf{LLaVA$^{\mathrm{W}}$}}
      & \multicolumn{1}{c}{\textbf{MM-Vet}} \\
    \cmidrule(lr){2-5}
    \cmidrule(lr){6-7}
    \cmidrule(lr){8-10}
    \cmidrule(lr){11-14}
    \cmidrule(lr){15-15}
      & Rand. $\uparrow$
      & Pop. $\uparrow$
      & Adv. $\uparrow$
      & Avg. $\uparrow$
      & Perc. $\uparrow$
      & Cog. $\uparrow$
      & All. $\uparrow$
      & Img. $\uparrow$
      & Vid. $\uparrow$
      & All $\uparrow$
      & Comp. $\uparrow$
      & Conv. $\uparrow$
      & Det. $\uparrow$ 
      & Acc. $\uparrow$ \\
    \midrule

    \multicolumn{12}{l}{\textbf{\textit{Offline reference models (upperbound)}}} \\
    MoE-LLaVA
      & \ub{90.48} & \ub{88.82} & \ub{88.67} & \ub{89.32}
      & \ub{1230.16} & \ub{272.47}
      & \ub{57.08} & \ub{61.47} & \ub{40.52}
      & \ub{62.20} & \ub{57.20} & \ub{74.60} & \ub{52.50}
      & \ub{15.60} \\
      
    \midrule
    
    \multicolumn{15}{l}{\textbf{\textit{Individual experts}}} \\
    COCO
      & 86.49 & 85.93 & 82.97 & 85.13
      & 1057.79 & 241.52
      & 53.40 & 58.35 & 36.81
      & 71.70 & 73.10 & 72.50 & 68.40
      & 20.10 \\

    GQA
      & 82.92 & 83.13 & 82.07 & 82.71
      & 1109.63 & 252.50
      & 14.76 & 11.76 & 26.14
      & 58.60 & 51.90 & 64.80 & 63.30
      & 18.10 \\

    VQA$^{\text{OCR}}$
      & 84.47 & 79.07 & 77.67 & 80.40
      & 1122.81 & 212.14
      & 51.83 & 55.13 & 35.42
      & 69.10 & 68.50 & 71.50 & 67.00
      & 17.10 \\

    VG
      & 51.58 & 50.00 & 50.03 & 50.54
      & 1.34 & 2.86
      & 36.09 & 36.61 & 34.12
      & 44.70 & 44.80 & 50.80 & 35.60
      & 18.30 \\

    VQA$^{\text{T}}$
      & 69.28 & 73.83 & 67.80 & 70.30
      & 455.94 & 190.36
      & 52.23 & 56.25 & 37.02
      & 64.30 & 65.90 & 67.20 & 57.70
      & 18.20 \\

    \cmidrule(lr){1-15}
    Public model
      & 68.38 & 75.90 & 69.23 & 71.17
      & 402.54 & 176.79
      & 51.17 & 55.14 & 36.15
      & 71.00 & 70.30 & 79.30 & 62.50
      & 18.60 \\

    Model Soup (avg)
      & 85.31 & 85.14 & 83.21 & 84.55
      & 1154.80 & 235.42
      & 53.82 & 57.91 & 37.84
      & 50.40 & 45.80 & 62.40 & 41.90
      & 14.60 \\

    Model Soup (wtd)
      & 85.67 & 85.39 & 83.46 & 84.84
      & 1160.35 & 239.18
      & 54.06 & 58.24 & 38.02
      & 51.10 & 46.30 & 63.70 & 42.40
      & 14.70 \\

    BTM
      & 86.39 & 86.03 & 83.96 & 85.46
      & 1167.92 & 243.61
      & 54.31 & 58.55 & 38.21
      & 52.00 & 46.70 & 64.90 & 42.90
      & 14.80 \\

    BTX
      & 86.02 & 85.72 & 83.71 & 85.15
      & 1174.63 & 247.86
      & 54.63 & 58.91 & 38.43
      & 52.90 & 47.10 & 66.10 & 43.40
      & 14.90 \\

    FlexOlmo-LoRA
      & 86.80 & 86.63 & 82.63 & 85.35
      & 938.50 & 251.07
      & 53.27 & 57.35 & 37.82
      & 72.00 & 73.00 & 75.20 & 66.30
      & 21.40 \\

    FlexOlmo
      & 86.84 & 86.26 & \textbf{84.83} & 85.98
      & 1176.35 & 249.29
      & 54.91 & 59.19 & 38.67
      & 52.20 & 47.20 & 65.60 & 43.80
      & 15.00 \\

    \rowcolor{yesgreen!8}
    \method-LoRA
      & 87.08 & 86.67 & 83.30 & 85.68
      & 995.61 & 231.43
      & 53.34 & 57.30 & 38.33
      & \textbf{73.50} & \textbf{73.80} & \textbf{79.10} & \textbf{65.30}
      & \textbf{24.40} \\

    \rowcolor{yesgreen!8}
    \method
      & \textbf{88.18} & \textbf{86.73} & 84.23 & \textbf{86.38}
      & \textbf{1188.56} & \textbf{259.29}
      & \textbf{55.11} & \textbf{59.38} & \textbf{38.94}
      & 59.90 & 53.30 & 77.50 & 51.00
      & 15.10 \\

    \bottomrule
  \end{tabular}%
  }
  \label{tab:qwen_bench}
\end{table*}
\section{Additional Experiments}
\label{sec:additional_results}

This section provides the object hallucination evaluation results in Tab.~\ref{tab:pope_all_backbones} and the full per-benchmark breakdowns underlying the aggregate results in Tab.~\ref{tab:main_all_backbones}. 
Tabs.~\ref{tab:qwen_img}--\ref{tab:phi_bench} report detailed results for Qwen-1.8B and Phi-2.7B, while Tabs.~\ref{tab:stablelm_img}--\ref{tab:stablelm_bench} provide the complete StableLM-1.6B results.

\begin{table*}[t]
  \centering
  \caption{\textbf{Detailed academic-task benchmark results with Phi-2.7B.} 
We report fine-grained results on VQA$^{\mathrm{v2}}$, GQA, SQA, and VQA$^{\mathrm{T}}$. 
The Public model denotes the shared backbone before client-specific expert composition. 
MoE-LLaVA$^\dagger$ denotes the original results reported in~\cite{moellava}, while MoE-LLaVA denotes the offline MoE model trained with joint access to the same client datasets used in our distributed setting. }
  \renewcommand{\arraystretch}{1.1}
  \setlength{\tabcolsep}{3.6pt}
  \scriptsize
  \resizebox{\linewidth}{!}{%
  \begin{tabular}{l cccc>{\hspace{8pt}} cccc>{\hspace{8pt}} cc>{\hspace{8pt}} c}
    \toprule
    \multirow{2}{*}{\textbf{Method}}
      & \multicolumn{4}{c}{\textbf{VQA$^{\mathrm{v2}}$}}
      & \multicolumn{4}{c}{\textbf{GQA}}
      & \multicolumn{2}{c}{\textbf{SQA}}
      & \multicolumn{1}{c}{\textbf{VQA$^{\mathrm {T}}$}} \\
    \cmidrule(lr){2-5}
    \cmidrule(lr){6-9}
    \cmidrule(lr){10-11}
    \cmidrule(lr){12-12}
      & Yes/No $\uparrow$
      & Num. $\uparrow$
      & Other $\uparrow$
      & Overall $\uparrow$
      & Binary $\uparrow$
      & Open $\uparrow$
      & Acc. $\uparrow$
      & Dist. $\downarrow$
      & Acc. $\uparrow$
      & IMG-Acc. $\uparrow$
      & Acc. $\uparrow$ \\
    \midrule

    \multicolumn{12}{l}{\textbf{\textit{Offline reference models (upper bound)}}} \\
    LLaVA-Phi
      & \ub{--} & \ub{--} & \ub{--} & \ub{71.40}
      & \ub{--} & \ub{--} & \ub{--} & \ub{--}
      & \ub{--} & \ub{68.40}
      & \ub{48.60} \\

    MoE-LLaVA$^\dagger$
      & \ub{--} & \ub{--} & \ub{--} & \ub{77.60}
      & \ub{--} & \ub{--} & \ub{61.40} & \ub{--}
      & \ub{--} & \ub{68.50}
      & \ub{51.40} \\

    MoE-LLaVA 
      & \ub{90.30} & \ub{56.96} & \ub{66.00} & \ub{74.99}
      & \ub{71.21} & \ub{44.38} & \ub{56.69} & \ub{1.60}
      & \ub{73.03} & \ub{68.37}
      & \ub{46.52} \\

    \cmidrule(lr){1-12}
    \multicolumn{12}{l}{\textbf{\textit{Individual experts}}} \\
    COCO
      & 84.73 & 47.33 & 58.96 & 65.33
      & 62.83 & 36.86 & 48.78 & 2.33
      & 63.43 & 59.26
      & 36.49 \\

    GQA
      & 85.02 & 48.50 & 59.68 & 68.26
      & 67.81 & 40.99 & 51.68 & 2.28
      & 39.99 & 27.66
      & 32.30 \\

    VQA$^{\text{OCR}}$
      & 71.46 & 30.57 & 46.14 & 51.46
      & 52.40 & 27.70 & 39.04 & 2.15
      & 63.17 & 59.21
      & 30.67 \\

    VG
      & 0.05 & 0.20 & 6.25 & 2.19
      & 31.61 & 4.00 & 16.67 & 3.95
      & 48.15 & 43.48
      & 1.22 \\

    VQA$^{\mathrm {T}}$
      & 5.26 & 1.67 & 38.42 & 20.73
      & 0.12 & 0.37 & 0.25 & 4.27
      & 68.22 & 65.69
      & 0.12 \\

    \cmidrule(lr){1-12}

    Public model
      & 53.36 & 7.90 & 42.13 & 42.87
      & 51.65 & 16.81 & 32.80 & 2.54
      & 70.69 & 69.26
      & 31.82 \\
    
    Model Soup (avg)
      & 82.90 & 45.79 & 57.18 & 64.89
      & 66.11 & 39.11 & 51.67 & 2.31
      & 64.68 & 59.28
      & 39.57 \\

    Model Soup (wtd)
      & 83.04 & 46.28 & 58.67 & 65.92
      & 67.26 & 39.95 & 52.34 & 2.29
      & 65.83 & 59.89
      & 39.76 \\

    BTM
      & 83.98 & 47.07 & 59.28 & 66.53
      & 67.54 & 40.07 & 52.56 & 2.27
      & 66.21 & 60.56
      & 40.76 \\

    BTX
      & 83.56 & 47.29 & 59.34 & 67.80
      & 67.82 & 40.16 & 52.67 & 2.26
      & 65.98 & 60.24
      & 40.53 \\

    FlexOlmo-LoRA
      & 53.37 & 7.89 & 42.13 & 42.87
      & 51.60 & 16.81 & 32.78 & 2.55
      & 70.74 & 69.41
      & 31.74 \\

    FlexOlmo
      & 85.18 & 49.85 & 61.69 & 70.03
      & 68.73 & 43.06 & 54.84 & 2.29
      & 67.67 & 62.77
      & 41.63 \\

    \rowcolor{yesgreen!8}
    \method-LoRA
      & 53.39 & 7.91 & 42.14 & 42.89
      & 51.67 & 16.86 & 32.84 & 2.55
      & \textbf{70.78} & \textbf{69.52}
      & 31.77 \\

    \rowcolor{yesgreen!8}
    \method
      & \textbf{87.93} & \textbf{50.39} & \textbf{61.89} & \textbf{71.31}
      & \textbf{70.95} & \textbf{44.29} & \textbf{56.53} & \textbf{2.13}
      & 69.32 & 66.34
      & \textbf{41.87} \\

    \bottomrule
  \end{tabular}%
  }
  \label{tab:phi_img}
\end{table*}
\begin{table*}[t]
  \centering
  \caption{\textbf{Detailed instruction-following benchmark results with Phi-2.7B.} 
We report fine-grained results on POPE, MME, SEED-Bench, LLaVA$^{\mathrm{W}}$, and MM-Vet. 
The Public model denotes the shared backbone before client-specific expert composition. 
MoE-LLaVA$^\dagger$ denotes the original results reported in~\cite{moellava}, while MoE-LLaVA denotes the offline MoE model trained with joint access to the same client datasets used in our distributed setting. }
  \renewcommand{\arraystretch}{1.1}
  \setlength{\tabcolsep}{3.6pt}
  \scriptsize
  \resizebox{\linewidth}{!}{%
  \begin{tabular}{l cccc>{\hspace{8pt}} cc>{\hspace{8pt}} ccc>{\hspace{8pt}} cccc>{\hspace{8pt}} c}
    \toprule
    \multirow{2}{*}{\textbf{Method}} 
    & \multicolumn{4}{c}{\textbf{POPE}}
      & \multicolumn{2}{c}{\textbf{MME}}
      & \multicolumn{3}{c}{\textbf{SEED-Bench}}
      & \multicolumn{4}{c}{\textbf{LLaVA$^{\mathrm{W}}$}}
      & \multicolumn{1}{c}{\textbf{MM-Vet}} \\
    \cmidrule(lr){2-5}
    \cmidrule(lr){6-7}
    \cmidrule(lr){8-10}
    \cmidrule(lr){11-14}
    \cmidrule(lr){15-15}
      & Rand. $\uparrow$
      & Pop. $\uparrow$
      & Adv. $\uparrow$
      & Avg. $\uparrow$
      & Perc. $\uparrow$
      & Cog. $\uparrow$
      & All. $\uparrow$
      & Img. $\uparrow$
      & Vid. $\uparrow$
      & All $\uparrow$
      & Comp. $\uparrow$
      & Conv. $\uparrow$
      & Det. $\uparrow$ 
      & Acc. $\uparrow$ \\
    \midrule

    \multicolumn{15}{l}{\textbf{\textit{Offline reference models (upper bound)}}} \\
    LLaVA-Phi
      & \ub{--} & \ub{--} & \ub{--} & \ub{85.00}
      & \ub{1335.10} & \ub{--}
      & \ub{--} & \ub{--} & \ub{--}
      & \ub{--} & \ub{--} & \ub{--} & \ub{--}
      & \ub{28.90} \\

    MoE-LLaVA$^\dagger$
      & \ub{--} & \ub{--} & \ub{--} & \ub{85.70}
      & \ub{1423.00} & \ub{--}
      & \ub{--} & \ub{--} & \ub{--}
      & \ub{94.10} & \ub{--} & \ub{--} & \ub{--}
      & \ub{34.30} \\

      MoE-LLaVA
      & \ub{87.39} & \ub{85.83} & \ub{83.73} & \ub{85.65}
      & \ub{1331.75} & \ub{274.64}
      & \ub{60.84} & \ub{66.40} & \ub{39.79}
      & \ub{76.20} & \ub{78.70} & \ub{74.40} & \ub{73.70}
      & \ub{20.20} \\

    \cmidrule(lr){1-15}

    \multicolumn{15}{l}{\textbf{\textit{Individual experts}}} \\
    COCO
      & 80.07 & 76.13 & 73.29 & 76.50
      & 1096.75 & 250.00
      & 56.01 & 61.24 & 36.19
      & 65.70 & 69.50 & 63.20 & 60.70
      & 24.40 \\

    GQA
      & 78.47 & 74.47 & 73.37 & 75.44
      & 1074.07 & 251.79
      & 29.61 & 27.46 & 37.74
      & 52.50 & 47.90 & 55.10 & 57.60
      & 16.70 \\

    VQA$^{\text{OCR}}$
      & 63.95 & 64.93 & 64.80 & 64.56
      & 994.06 & 262.86
      & 58.34 & 63.44 & 39.02
      & 61.30 & 58.80 & 67.50 & 57.80
      & 17.60 \\

    VG
      & 66.05 & 73.07 & 68.27 & 69.13
      & 598.64 & 236.07
      & 49.57 & 53.72 & 33.83
      & 49.70 & 48.50 & 63.40 & 32.90
      & 16.80 \\

    VQA$^{\text{T}}$
      & 51.75 & 50.13 & 50.13 & 50.67
      & 71.55 & 31.07
      & 58.37 & 63.38 & 39.39
      & 61.00 & 59.30 & 71.30 & 51.90
      & 23.00 \\

    \cmidrule(lr){1-15}

    Public model
      & 83.54 & 79.37 & 78.53 & 80.48
      & 862.53 & 281.07
      & 59.36 & 64.42 & 40.17
      & 73.30 & 77.90 & 72.90 & 65.30
      & 23.20 \\

    Model Soup (avg)
      & 76.26 & 73.36 & 72.89 & 74.17
      & 1054.37 & 242.38
      & 53.38 & 58.17 & 33.50
      & 42.40 & 43.70 & 45.20 & 44.90
      & 15.70 \\

    Model Soup (wtd)
      & 76.72 & 73.54 & 73.04 & 74.43
      & 1067.26 & 250.40
      & 54.05 & 59.92 & 34.78
      & 43.20 & 44.70 & 46.20 & 45.10
      & 14.40 \\

    BTM
      & 74.81 & 72.57 & 71.03 & 72.80
      & 1106.64 & 256.64
      & 52.82 & 60.16 & 35.81
      & 41.70 & 44.30 & 45.40 & 44.80
      & 16.30 \\

    BTX
      & 75.37 & 73.14 & 72.35 & 73.62
      & 1096.48 & 252.10
      & 51.43 & 58.36 & 33.58
      & 44.40 & 45.10 & 46.80 & 46.30
      & 17.00 \\

    FlexOlmo-LoRA
      & 83.54 & 79.43 & 78.63 & 80.53
      & 862.53 & 281.06
      & 59.39 & 64.44 & 40.24
      & 72.10 & 75.40 & 72.90 & 65.00
      & 22.30 \\

    FlexOlmo
      & 74.91 & 75.30 & 74.73 & 74.98
      & 1193.71 & 250.00
      & 58.79 & 63.80 & 39.79
      & 52.00 & 52.60 & 56.00 & 45.70
      & 18.30 \\

    \rowcolor{yesgreen!8}
    \method-LoRA
      & \textbf{83.95} & 79.82 & 78.76 & \textbf{80.84}
      & 862.60 & 281.07
      & 59.38 & 64.46 & 40.14
      & \textbf{75.30} & \textbf{77.20} & \textbf{78.60} & \textbf{66.20}
      & \textbf{23.80} \\

    \rowcolor{yesgreen!8}
    \method
      & 80.86 & \textbf{80.80} & \textbf{79.83} & 80.50
      & \textbf{1291.41} & \textbf{287.50}
      & \textbf{59.96} & \textbf{65.00} & \textbf{40.88}
      & 55.90 & 54.10 & 69.10 & 42.10
      & 19.80 \\

    \bottomrule
  \end{tabular}%
  }
  \label{tab:phi_bench}
\end{table*}

\paragraph{Object Hallucination Evaluation.} We evaluate object hallucination using POPE, which formulates hallucination detection as a polling-based object-existence question answering task. Tab.~\ref{tab:pope_all_backbones} reports results under the \textit{Random}, \textit{Popular}, and \textit{Adversarial} settings, together with F1-score and yes-answer ratio. Across the three backbones, \method generally improves hallucination robustness over the corresponding public model and distributed baselines. In particular, the full \method variant achieves strong accuracy and F1-score on Qwen-1.8B and StableLM-1.6B, showing that public-anchored routing helps preserve reliable object grounding while reusing domain-specific experts. However, \method-LoRA is also competitive, especially for Phi-2.7B, indicating that lightweight expert adaptation can still provide effective hallucination mitigation. We also observe that the yes-answer ratio becomes more balanced compared with the public model in several settings, suggesting that \method reduces the tendency to over-answer ``yes'' and produces responses more consistent with the visual evidence.

\paragraph{Detailed Results.}
Tabs.~\ref{tab:qwen_img} and~\ref{tab:qwen_bench} show the detailed Qwen-1.8B results. 
Across the fine-grained academic-task metrics, \method generally improves over the distributed baselines. 
On instruction-following benchmarks, \method-LoRA is particularly strong on LLaVA$^{\mathrm{W}}$ and MM-Vet, showing that the parameter-efficient variant can provide competitive expert composition with lower adaptation cost. 
The individual experts are often strong on specific domains but vary substantially across benchmarks, motivating routing-based composition rather than relying on a single client expert.

Tabs.~\ref{tab:phi_img} and~\ref{tab:phi_bench} show the same evaluation for Phi-2.7B. 
The trends are consistent with Qwen-1.8B: \method improves the strongest distributed performance on most academic-task metrics, while \method-LoRA remains competitive on instruction-following benchmarks. 
These results indicate that the gains are not tied to a single language backbone and that public-anchored calibration transfers across MLLM configurations.

\paragraph{StableLM-1.6B Results.} Tabs.~\ref{tab:stablelm_img} and~\ref{tab:stablelm_bench} report the complete StableLM-1.6B results. 
The results further support the main findings: \method improves over model-merging and FlexOlmo-style baselines on most academic-task metrics, while \method-LoRA provides strong instruction-following performance, especially on LLaVA$^{\mathrm{W}}$ and MM-Vet. 
This confirms that the proposed public-private calibration remains effective on a third MLLM backbone.

\begin{table*}[t]
  \centering
\caption{\textbf{Detailed academic-task benchmark results with StableLM-1.6B.} 
We report fine-grained results on VQA$^{\mathrm{v2}}$, GQA, SQA, and VQA$^{\mathrm{T}}$. 
The Public model denotes the shared backbone before client-specific expert composition. 
MoE-LLaVA$^\dagger$ denotes the original results reported in~\cite{moellava}, while MoE-LLaVA denotes the offline MoE model trained with joint access to the same client datasets used in our distributed setting. }
  \renewcommand{\arraystretch}{1.1}
  \setlength{\tabcolsep}{3.6pt}
  \scriptsize
  \resizebox{\linewidth}{!}{%
  \begin{tabular}{l cccc>{\hspace{8pt}} cccc>{\hspace{8pt}} cc>{\hspace{8pt}} c}
    \toprule
    \multirow{2}{*}{\textbf{Method}}
      & \multicolumn{4}{c}{\textbf{VQA$^{\mathrm{v2}}$}}
      & \multicolumn{4}{c}{\textbf{GQA}}
      & \multicolumn{2}{c}{\textbf{SQA}}
      & \multicolumn{1}{c}{\textbf{VQA$^{\mathrm{T}}$}} \\
    \cmidrule(lr){2-5}
    \cmidrule(lr){6-9}
    \cmidrule(lr){10-11}
    \cmidrule(lr){12-12}
      & Yes/No $\uparrow$
      & Num. $\uparrow$
      & Other $\uparrow$
      & Overall $\uparrow$
      & Binary $\uparrow$
      & Open $\uparrow$
      & Acc. $\uparrow$
      & Dist. $\downarrow$
      & Acc. $\uparrow$
      & IMG-Acc. $\uparrow$
      & Acc. $\uparrow$ \\
    \midrule

    \multicolumn{12}{l}{\textbf{\textit{Offline reference models (upper bound)}}} \\
    MoE-LLaVA$^\dagger$
      & \ub{--} & \ub{--} & \ub{--} & \ub{76.70}
      & \ub{--} & \ub{--} & \ub{60.30} & \ub{--}
      & \ub{--} & \ub{62.60}
      & \ub{50.10} \\

    MoE-LLaVA
      & \ub{89.37} & \ub{55.63} & \ub{67.28} & \ub{76.65}
      & \ub{72.30} & \ub{44.28} & \ub{57.90} & \ub{2.07}
      & \ub{58.23} & \ub{57.96}
      & \ub{59.15} \\

    \cmidrule(lr){1-12}
    \multicolumn{12}{l}{\textbf{\textit{Individual experts}}} \\
    COCO
      & 85.06 & 50.21 & 60.35 & 69.26
      & 62.01 & 36.25 & 48.08 & 2.26
      & 50.04 & 38.28
      & 44.35 \\

    GQA
      & 84.94 & 49.52 & 58.93 & 68.57
      & 66.38 & 41.39 & 52.28 & 2.21
      & 48.01 & 40.31
      & 37.84 \\

    VQA$^{\text{OCR}}$
      & 84.04 & 46.41 & 53.26 & 65.16
      & 55.38 & 25.63 & 39.28 & 2.25
      & 50.08 & 47.94
      & 39.35 \\

    VG
      & 60.17 & 16.60 & 43.19 & 47.17
      & 30.95 & 12.51 & 20.97 & 2.91
      & 41.22 & 37.83
      & 11.47 \\

    VQA$^{\mathrm{T}}$
      & 5.14 & 2.05 & 40.51 & 22.48
      & 0.33 & 0.60 & 0.48 & 4.24
      & 49.75 & 50.28
      & 2.54 \\

    \cmidrule(lr){1-12}

    Public model
      & 2.38 & 6.39 & 47.95 & 24.46
      & 12.02 & 20.69 & 16.71 & 2.40
      & 50.70 & 54.73
      & 27.24 \\
      
    Model Soup (avg)
      & 82.58 & 49.50 & 58.99 & 67.63
      & 64.93 & 38.78 & 50.79 & 2.28
      & 51.93 & 50.96
      & 42.81 \\

    Model Soup (wtd)
      & 83.45 & 50.02 & 59.61 & 68.34
      & 65.62 & 39.19 & 51.32 & 2.26
      & 52.48 & 51.50
      & 43.26 \\

    BTM
      & 84.32 & 50.57 & 60.23 & 69.05
      & 66.98 & 40.01 & 52.39 & 2.21
      & 53.14 & 52.03
      & 43.71 \\

    BTX
      & 85.19 & 51.07 & 60.86 & 69.76
      & 66.30 & 39.60 & 51.86 & 2.24
      & 53.57 & 52.48
      & 44.16 \\

    FlexOlmo-LoRA
      & 72.48 & 43.68 & 53.40 & 62.61
      & 57.38 & 39.58 & 47.92 & 1.96
      & 49.06 & 42.05
      & 37.42 \\

    FlexOlmo
      & 84.05 & 53.35 & 63.21 & 70.67
      & 62.45 & 42.01 & 51.39 & 2.13
      & 50.29 & 42.64
      & 44.50 \\

    \rowcolor{yesgreen!8}
    \method-LoRA
      & 76.21 & 46.82 & 56.27 & 64.93
      & 60.25 & 39.15 & 49.17 & \textbf{1.85}
      & 51.48 & 50.25
      & 39.40 \\

    \rowcolor{yesgreen!8}
    \method
      & \textbf{86.93} & \textbf{52.11} & \textbf{62.10} & \textbf{71.19}
      & \textbf{68.35} & \textbf{40.82} & \textbf{53.46} & 2.19
      & \textbf{54.66} & \textbf{53.64}
      & \textbf{45.06} \\

    \bottomrule
  \end{tabular}%
  }
  \label{tab:stablelm_img}
\end{table*}
\begin{table*}[!th]
  \centering
  \caption{\textbf{Detailed instruction-following benchmark results with StableLM-1.6B.} 
We report fine-grained results on POPE, MME, SEED-Bench, LLaVA$^{\mathrm{W}}$, and MM-Vet. 
The Public model denotes the shared backbone before client-specific expert composition. 
MoE-LLaVA$^\dagger$ denotes the original results reported in~\cite{moellava}, while MoE-LLaVA denotes the offline MoE model trained with joint access to the same client datasets used in our distributed setting. }
  \renewcommand{\arraystretch}{1.1}
  \setlength{\tabcolsep}{3.6pt}
  \scriptsize
  \resizebox{\linewidth}{!}{%
  \begin{tabular}{l cccc>{\hspace{8pt}} cc>{\hspace{8pt}} ccc>{\hspace{8pt}} cccc>{\hspace{8pt}} c}
    \toprule
    \multirow{2}{*}{\textbf{Method}} 
    & \multicolumn{4}{c}{\textbf{POPE}}
      & \multicolumn{2}{c}{\textbf{MME}}
      & \multicolumn{3}{c}{\textbf{SEED-Bench}}
      & \multicolumn{4}{c}{\textbf{LLaVA$^{\mathrm{W}}$}}
      & \multicolumn{1}{c}{\textbf{MM-Vet}} \\
    \cmidrule(lr){2-5}
    \cmidrule(lr){6-7}
    \cmidrule(lr){8-10}
    \cmidrule(lr){11-14}
    \cmidrule(lr){15-15}
      & Rand. $\uparrow$
      & Pop. $\uparrow$
      & Adv. $\uparrow$
      & Avg. $\uparrow$
      & Perc. $\uparrow$
      & Cog. $\uparrow$
      & All. $\uparrow$
      & Img. $\uparrow$
      & Vid. $\uparrow$
      & All $\uparrow$
      & Comp. $\uparrow$
      & Conv. $\uparrow$
      & Det. $\uparrow$ 
      & Acc. $\uparrow$ \\
    \midrule

    \multicolumn{15}{l}{\textbf{\textit{Offline reference models (upper bound)}}} \\
    MoE-LLaVA$^\dagger$
      & \ub{--} & \ub{--} & \ub{--} & \ub{85.70}
      & \ub{1318.20} & \ub{--}
      & \ub{--} & \ub{--} & \ub{--}
      & \ub{86.80} & \ub{--} & \ub{--} & \ub{--}
      & \ub{26.90} \\

    MoE-LLaVA
      & \ub{85.36} & \ub{85.61} & \ub{83.59} & \ub{84.85}
      & \ub{1304.56} & \ub{233.04}
      & \ub{55.42} & \ub{61.56} & \ub{36.91}
      & \ub{54.20} & \ub{53.70} & \ub{68.60} & \ub{53.50}
      & \ub{18.40} \\

    \cmidrule(lr){1-15}
    \multicolumn{15}{l}{\textbf{\textit{Individual experts}}} \\
    COCO
      & 83.74 & 83.56 & 78.28 & 81.86
      & 1230.27 & 209.48
      & 48.47 & 52.31 & 31.08
      & 69.50 & 71.00 & 73.00 & 60.30
      & 20.60 \\

    GQA
      & 83.82 & 84.01 & 79.32 & 82.38
      & 1243.68 & 211.70
      & 36.69 & 36.82 & 29.24
      & 53.00 & 50.30 & 66.40 & 39.80
      & 17.70 \\

    VQA$^{\text{OCR}}$
      & 74.36 & 74.17 & 72.90 & 73.81
      & 1013.82 & 201.32
      & 47.92 & 51.51 & 30.28
      & 61.10 & 61.80 & 57.70 & 61.10
      & 15.50 \\

    VG
      & 84.78 & 77.50 & 74.20 & 78.83
      & 649.23 & 125.71
      & 4.56 & 5.77 & 0.05
      & 51.70 & 51.60 & 60.30 & 41.40
      & 20.60 \\

    VQA$^{\text{T}}$
      & 55.70 & 53.47 & 52.57 & 53.91
      & 462.42 & 131.43
      & 46.80 & 50.31 & 33.48
      & 61.40 & 58.60 & 69.30 & 56.40
      & 21.10 \\

    \cmidrule(lr){1-15}

    Public model
      & 58.76 & 53.80 & 53.37 & 55.31
      & 518.52 & 243.57
      & 48.10 & 52.02 & 33.27
      & 68.00 & 71.30 & 71.30 & 57.10
      & 20.00 \\
    
    Model Soup (avg)
      & 79.96 & 79.29 & 77.41 & 78.89
      & 1126.25 & 205.08
      & 50.14 & 54.96 & 31.98
      & 48.80 & 44.70 & 59.50 & 43.00
      & 12.70 \\

    Model Soup (wtd)
      & 80.73 & 80.05 & 78.15 & 79.64
      & 1134.57 & 203.11
      & 50.63 & 55.49 & 32.29
      & 49.30 & 45.10 & 60.10 & 43.40
      & 12.90 \\

    BTM
      & 81.48 & 80.80 & 79.61 & 80.63
      & 1179.35 & 208.97
      & 51.11 & 56.02 & 32.60
      & 49.80 & 45.60 & 60.70 & 43.80
      & 13.00 \\

    BTX
      & 82.24 & 81.55 & 78.88 & 80.89
      & 1190.47 & 210.92
      & 51.60 & 56.55 & 32.91
      & 50.20 & 46.00 & 61.20 & 44.20
      & 13.10 \\

    FlexOlmo-LoRA
      & 75.62 & 77.45 & 76.04 & 76.37
      & 956.41 & 200.57
      & 51.49 & 55.29 & 31.48
      & 70.40 & 72.60 & 73.50 & 60.20
      & 21.70 \\

    FlexOlmo
      & 77.84 & 78.07 & 76.87 & 77.59
      & 1171.74 & 221.79
      & 52.18 & 57.01 & 33.91
      & 50.40 & 44.50 & 60.00 & 49.50
      & 17.70 \\

    \rowcolor{yesgreen!8}
    \method-LoRA
      & 84.06 & 83.92 & 81.45 & 83.14
      & 1023.57 & 203.42
      & 52.41 & 57.05 & 33.78
      & \textbf{73.60} & \textbf{74.00} & \textbf{75.10} & \textbf{63.70}
      & \textbf{24.50} \\

    \rowcolor{yesgreen!8}
    \method
      & \textbf{84.98} & \textbf{84.27} & \textbf{82.27} & \textbf{83.84}
      & \textbf{1274.28} & \textbf{216.43}
      & \textbf{53.94} & \textbf{59.10} & \textbf{34.39}
      & 52.80 & 48.30 & 64.20 & 46.50
      & 14.80 \\

    \bottomrule
  \end{tabular}%
  }
  \label{tab:stablelm_bench}
\end{table*}




\section{\texorpdfstring{Additional Analyses of \method}{Additional Analyses of DistMoE}}
\label{sec:cr_additional}

\newcommand{\sd}[1]{\,{\scriptsize$\pm$#1}}
\begin{table*}[t]
\centering
\caption{Three-seed comparison (mean\,$\pm$\,std) of \method and FlexOlmo. $\Delta$: difference of means. Wins: matched seed-benchmark comparisons won by \method.}
\label{tab:seeds}
\vspace{-0.1cm}
\setlength{\tabcolsep}{3pt}
\footnotesize
\resizebox{\textwidth}{!}{%
\begin{tabular}{ll ccccccccc c}
\toprule
\textbf{LLM} & \textbf{Method}
& VQA$^{\mathrm{v2}}$ & GQA & SQA$^{\mathrm{I}}$ & VQA$^{\mathrm{T}}$
& POPE & MME & SEED & LLaVA$^{\mathrm{W}}$ & MM-Vet & Wins \\
\midrule
\multirow{3}{*}{Qwen-1.8B}
& FlexOlmo
& 69.95\sd{0.24} & 54.25\sd{0.07} & 60.92\sd{0.07} & 41.69\sd{0.07} & 85.72\sd{0.36} & 1173.80\sd{3.08} & 54.93\sd{0.30} & 52.23\sd{0.45} & 14.93\sd{0.12} & -- \\
& \method
& \textbf{70.25}\sd{0.16} & \textbf{55.38}\sd{0.11} & \textbf{61.28}\sd{0.22} & \textbf{42.70}\sd{0.05} & \textbf{86.41}\sd{0.13} & \textbf{1188.87}\sd{1.32} & \textbf{55.36}\sd{0.25} & \textbf{59.57}\sd{0.31} & \textbf{15.23}\sd{0.15} & 27/27 \\
& $\Delta$
& +0.30 & +1.13 & +0.36 & +1.01 & +0.70 & +15.06 & +0.43 & +7.33 & +0.30 & \\
\midrule
\multirow{3}{*}{Phi-2.7B}
& FlexOlmo
& 67.57\sd{2.13} & 49.41\sd{4.75} & 63.12\sd{1.64} & 33.63\sd{6.93} & 81.18\sd{5.41} & \textbf{1113.72}\sd{69.30} & 59.00\sd{0.18} & 56.57\sd{4.82} & 16.60\sd{1.75} & -- \\
& \method
& \textbf{69.95}\sd{1.22} & \textbf{50.18}\sd{5.74} & \textbf{66.34}\sd{0.25} & \textbf{35.42}\sd{5.68} & \textbf{82.21}\sd{2.43} & 1053.75\sd{210.19} & \textbf{59.06}\sd{0.79} & \textbf{59.30}\sd{3.14} & \textbf{19.40}\sd{0.53} & 21/27 \\
& $\Delta$
& +2.38 & +0.77 & +3.22 & +1.79 & +1.03 & $-$59.97 & +0.06 & +2.73 & +2.80 & \\
\bottomrule
\end{tabular}%
}
\vspace{-0.3cm}
\end{table*}

\begin{table*}[t]
\centering
\caption{Effect of Stage-II components of \method with Phi-2.7B (single run, averaged router composition). Stage-I corresponds to FlexOlmo.}
\label{tab:phi_components}
\vspace{-0.1cm}
\setlength{\tabcolsep}{4pt}
\footnotesize
\resizebox{\textwidth}{!}{%
\begin{tabular}{cccc ccccccccc}
\toprule
Stage & $T(\cdot)$ & $\mathcal{L}_{\mathrm{iso}}$ & $b_{m}$
& VQA$^{\mathrm{v2}}$ & GQA & SQA$^{\mathrm{I}}$ & VQA$^{\mathrm{T}}$
& POPE & MME & SEED & LLaVA$^{\mathrm{W}}$ & MM-Vet \\
\midrule
I & -- & -- & -- & 70.03 & 54.84 & 62.77 & 41.63 & 74.98 & 1193.71 & 58.79 & 52.00 & 18.30 \\
\midrule
II & \cmark & \xmark & \xmark & 68.72 & 54.04 & 63.05 & 41.04 & 74.83 & 1120.46 & 59.25 & 52.20 & 17.10 \\
II & \cmark & \cmark & \xmark & 69.96 & 55.75 & 65.28 & 41.70 & 78.46 & 1133.58 & 59.74 & 52.80 & 17.60 \\
II & \xmark & \xmark & \cmark & 69.89 & 53.24 & 62.92 & 41.53 & 75.68 & 1144.60 & 59.89 & 53.20 & 17.30 \\
II & \cmark & \xmark & \cmark & 71.27 & 54.92 & 64.75 & 41.49 & 78.65 & \textbf{1294.44} & 59.88 & 54.60 & 19.60 \\
\rowcolor{yesgreen!8}
II & \cmark & \cmark & \cmark & \textbf{71.31} & \textbf{56.53} & \textbf{66.34} & \textbf{41.87} & \textbf{80.50} & 1291.41 & \textbf{59.96} & \textbf{55.90} & \textbf{19.80} \\
\bottomrule
\end{tabular}%
}
\vspace{-0.3cm}
\end{table*}

This section examines \emph{why} \method works beyond single-run accuracy: whether its gains survive training randomness, how each Stage-II component contributes to composition, what role $\mathcal{L}_{\mathrm{iso}}$ actually plays, and whether the improvements can be attributed to additional capacity.

\textbf{Residual diagnostics.} All statistics are computed on $\Delta_m(x)$ from the same 19k public inputs for every client and layer, then averaged over layers. RMS-CV is the coefficient of variation (std/mean) of the residual RMS across the M clients, and Max/min is the ratio between the largest and smallest client RMS. Effective rank is the entropy-based rank \cite{roy2007effective} of the residual covariance, normalized by d. SWD is the mean pairwise sliced Wasserstein distance \cite{bonneel2015sliced} between the M clients’ residual distributions.

\subsection{\texorpdfstring{Robustness across random seeds}{Robustness across random seeds}}
\label{sec:cr_seeds}
Tab.~\ref{tab:seeds} repeats the comparison with FlexOlmo over three seeds and shows that the advantage of \method is systematic rather than the product of a favorable run. 
More revealing than the means is the variance. 
Composing independently trained experts is intrinsically fragile: on Phi-2.7B, both methods fluctuate considerably across seeds, indicating that small differences in how each expert drifts during Stage-I propagate into the composed model. 
Stage-II dampens this effect, and \method is generally more stable than FlexOlmo, most visibly on 12/18 benchmarks. 
This suggests that recalibrating every expert against the same public anchor does not only shift performance upward, but also makes the outcome of composition less dependent on the idiosyncrasies of each client's training run. 
Consistently, the advantage tends to be larger on the less stable backbone, i.e., calibration matters most precisely where uncalibrated composition is least reliable; we do not extrapolate this trend beyond the evaluated model range. 
MME on Phi-2.7B is the exception, where a single unstable seed dominates the mean, and we therefore do not claim a uniform improvement on this backbone.

\subsection{\texorpdfstring{Component attribution on Phi-2.7B}{Component attribution on Phi-2.7B}}
\label{sec:cr_components}
Tab.~\ref{tab:phi_components} extends Tab.~\ref{tab:components} to Phi-2.7B and is best read through the two factors of Proposition~\ref{prop:residual_bound}: the adapter $T(\cdot)$ and $\mathcal{L}_{\mathrm{iso}}$ shape the residual $\boldsymbol{\Delta}_m$, i.e., \emph{what} an expert contributes, whereas the bias $b_m$ sets the routing prior $\alpha_m$, i.e., \emph{how much} it contributes. 
Three lessons emerge. 
\textit{First, extra flexibility without a constraint is harmful:} an adapter trained only with next-token prediction tends to fall below Stage-I, as it is free to distort expert outputs in client-specific ways that make them even less comparable. 
\textit{Second, the constraint is what turns the adapter into a calibration tool:} pairing it with $\mathcal{L}_{\mathrm{iso}}$ ties every client to the same reference geometry, and the adapter starts to help consistently. 
\textit{Third, both factors are needed:} recalibrating the routing prior on top of the adapter yields the largest jump on perception-heavy benchmarks such as MME, yet the prior alone cannot compensate for residuals of mismatched geometry, and the full model is only reached when both what and how much are calibrated.
Once the routing prior is in place, the marginal gain of $\mathcal{L}_{\mathrm{iso}}$ becomes smaller, consistent with the two factors partially compensating each other.

\begin{table}[t]
\centering
\caption{Three-seed comparison of Stage-II with Phi-2.7B, with and without $\mathcal{L}_{\mathrm{iso}}$. Both variants use $T_m{+}b_m$ and the same seeds. Paired $\Delta$: with minus without, mean\,$\pm$\,std over seeds.}
\label{tab:liso_seeds}
\setlength{\tabcolsep}{5pt}
\footnotesize
\begin{tabular}{l ccc}
\toprule
Benchmark & w/o $\mathcal{L}_{\mathrm{iso}}$ & w/ $\mathcal{L}_{\mathrm{iso}}$ & Paired $\Delta$ \\
\midrule
VQA$^{\mathrm{v2}}$  & 69.94\,$\pm$\,1.47 & 69.95\,$\pm$\,1.22 & +0.01\,$\pm$\,0.59 \\
GQA                  & 50.62\,$\pm$\,5.26 & 50.18\,$\pm$\,5.74 & $-$0.45\,$\pm$\,0.80 \\
SQA$^{\mathrm{I}}$   & 66.56\,$\pm$\,0.92 & 66.34\,$\pm$\,0.25 & $-$0.23\,$\pm$\,0.68 \\
VQA$^{\mathrm{T}}$   & 36.21\,$\pm$\,4.80 & 35.42\,$\pm$\,5.68 & $-$0.80\,$\pm$\,1.26 \\
POPE                 & 82.68\,$\pm$\,2.47 & 82.21\,$\pm$\,2.43 & $-$0.47\,$\pm$\,1.18 \\
MME                  & 1084.13\,$\pm$\,177.99 & 1053.75\,$\pm$\,210.19 & $-$30.38\,$\pm$\,44.46 \\
SEED                 & 59.00\,$\pm$\,0.77 & 59.06\,$\pm$\,0.79 & +0.06\,$\pm$\,0.16 \\
LLaVA$^{\mathrm{W}}$ & 58.83\,$\pm$\,3.72 & 59.30\,$\pm$\,3.14 & +0.47\,$\pm$\,0.85 \\
MM-Vet               & 16.97\,$\pm$\,2.35 & 19.40\,$\pm$\,0.53 & +2.43\,$\pm$\,1.94 \\
\bottomrule
\end{tabular}
\end{table}

\subsection{\texorpdfstring{Controlled multi-seed effect of the isotropic loss}{Controlled multi-seed effect of the isotropic loss}}
\label{sec:cr_liso}
Tab.~\ref{tab:liso_seeds} isolates $\mathcal{L}_{\mathrm{iso}}$ by comparing Stage-II with and without it across seeds, all else being equal. 
On short-answer QA benchmarks, many of which coincide with client domains, the effect is within seed noise. 
On open-ended generation benchmarks, in contrast, $\mathcal{L}_{\mathrm{iso}}$ helps consistently, and on MM-Vet it both improves every run and makes the outcome markedly more stable. 
We attribute this asymmetry to how answers are composed. 
A short answer is typically decided by a few tokens routed to a single in-domain expert, so the relative geometry of the other experts matters little. 
A long-form answer, instead, mixes many tokens routed to different experts, and comparable residual geometry becomes essential for these contributions to combine coherently. 
Together with Tab.~\ref{tab:residual}, where $\mathcal{L}_{\mathrm{iso}}$ consistently improves residual geometry, this positions $\mathcal{L}_{\mathrm{iso}}$ as a complementary regularizer whose benefit is task-dependent, rather than the main source of the gains of \method. 
With only three paired runs, we report these trends as consistent cross-seed evidence and do not claim statistical significance.

\subsection{\texorpdfstring{Are the gains explained by additional parameters?}{Are the gains explained by additional parameters?}}
\label{sec:cr_params}
\method adds trainable parameters per client relative to FlexOlmo (Tab.~\ref{tab:efficiency}), entirely due to the Stage-II adapters, which raises the question of whether the gains simply reflect more capacity. 
Two controlled comparisons in Tab.~\ref{tab:phi_components} suggest otherwise. 
Adding the adapter alone, which accounts for all the extra parameters, does not improve over Stage-I, whereas $\mathcal{L}_{\mathrm{iso}}$, which adds no parameters at all, is what makes the adapter effective. 
The improvements therefore stem from \emph{how} the added parameters are constrained, i.e., anchored to the public expert, regularized toward a common residual geometry, and paired with a recalibrated routing prior, rather than from capacity itself. 
This also implies that the adapter could be made considerably lighter, e.g., low-rank, without changing the underlying mechanism.

\section{Number of LoRA experts in \method LoRA}
\label{sec:lora}

\begin{wrapfigure}[10]{r}{0.4\textwidth}
\vspace{-0.5cm}
  \centering
    \includegraphics[width=0.9\linewidth]{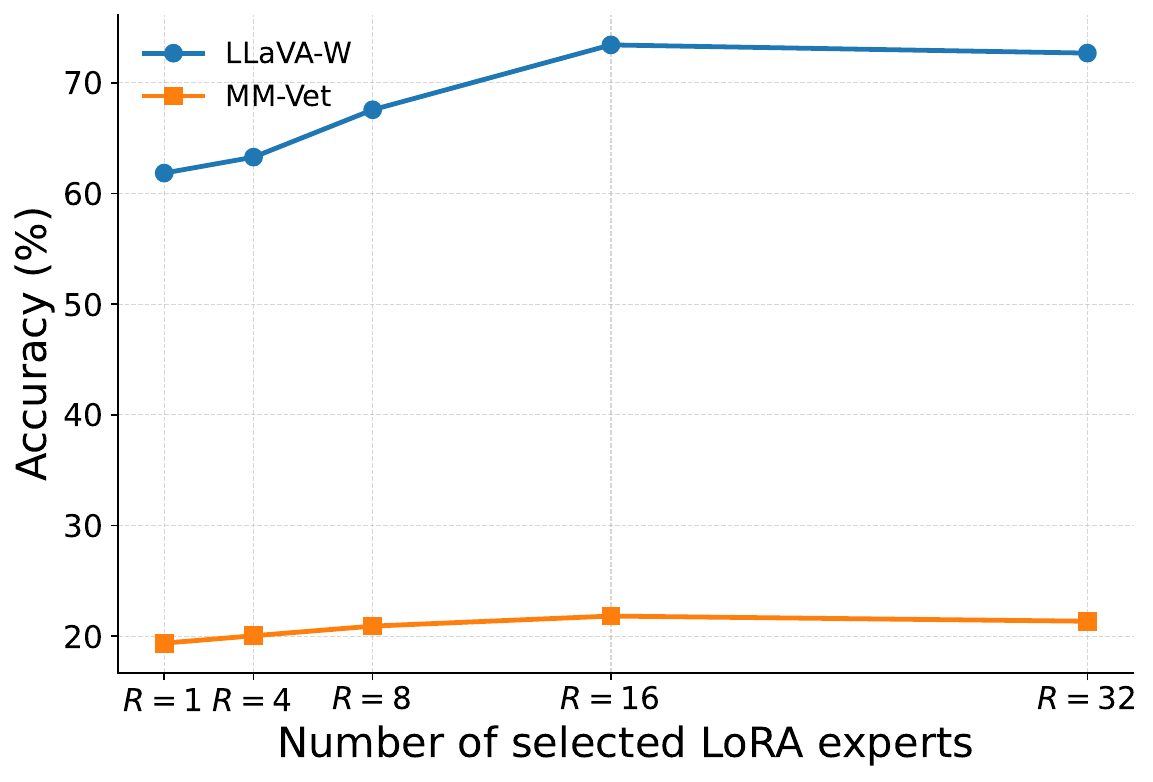}
    \vspace{-0.2cm}
    \caption{Number of LoRA experts.}
    \label{fig:lora_expert}
\end{wrapfigure}
We analyze the effect of using multiple LoRA experts per client in Fig.~\ref{fig:lora_expert}. Increasing the number of selected LoRA experts improves performance at first, showing that multiple lightweight experts provide complementary specialization paths and enable richer expert composition. The best performance is obtained with a moderate number of LoRA experts, while further increasing the number of experts slightly degrades performance. This suggests that too many selected LoRA experts may introduce redundant or noisy expert contributions, weakening sparse specialization. Therefore, we use $R=16$ as the default setting, which provides the best trade-off between adaptation capacity and routing stability.

\section{Routing Distributions}
\label{sec:routing_analysis}

We present the routing distributions of \method in Figs.~\ref{fig:routing_qwen}, \ref{fig:routing_phi}, and \ref{fig:routing_stablelm}, corresponding to Qwen-1.8B, Phi-2.7B, and StableLM-1.6B. 
For each backbone, we compute the top-$4$ routing mass on representative image-QA and instruction-following benchmarks, where routing mass is the aggregate normalized router weight assigned to each expert across tokens at each MoE layer. 
The plots show how the router distributes computation between the shared public expert and the client-specific private experts trained on the datasets described in Section~\ref{sec:dataset_details}. 

Across all three backbones, \method assigns non-trivial routing mass to the public expert while still activating private experts in a benchmark-dependent manner. 
This indicates that the router does not simply replace the public expert with domain experts; instead, private experts act as calibrated residual contributors around the shared public expert.  
In most of the datasets, the public route becomes more dominant in later layers, suggesting that later layers rely more on general visual-language and instruction-following knowledge. Conversely, middle layers often show stronger private-expert activity, reflecting the use of domain-specific visual grounding and reasoning features. 
The comparison among Figs.~\ref{fig:routing_qwen}, \ref{fig:routing_phi}, and \ref{fig:routing_stablelm} further shows that routing behavior is backbone-dependent. Nevertheless, the same qualitative trend holds: \method produces sparse but non-exclusive expert usage, preserves a stable public pathway, and composes multiple private experts without requiring explicit domain labels.

\begin{figure*}[t]
    \centering

    \begin{subfigure}[t]{0.48\textwidth}
        \centering
        \includegraphics[width=\linewidth]{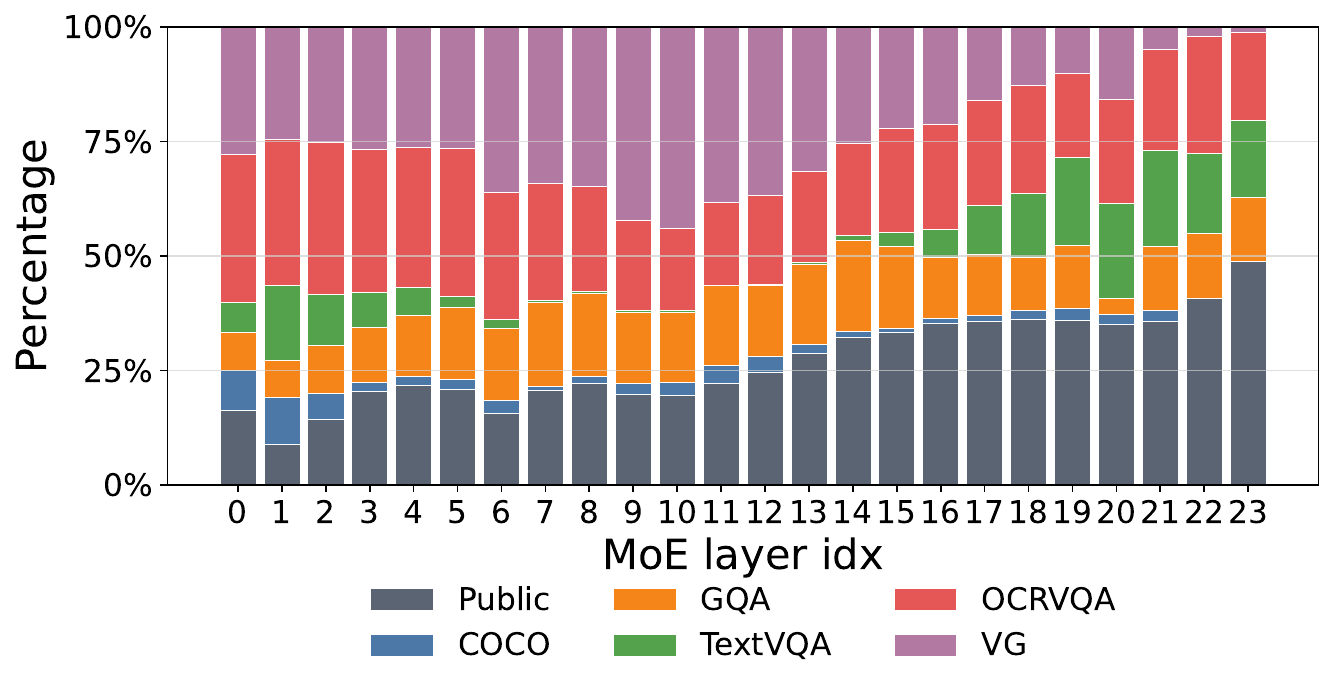}
        \caption{GQA}
        \label{fig:gqa_qwen_distmoe}
    \end{subfigure}
    \hfill
    \begin{subfigure}[t]{0.48\textwidth}
        \centering
        \includegraphics[width=\linewidth]{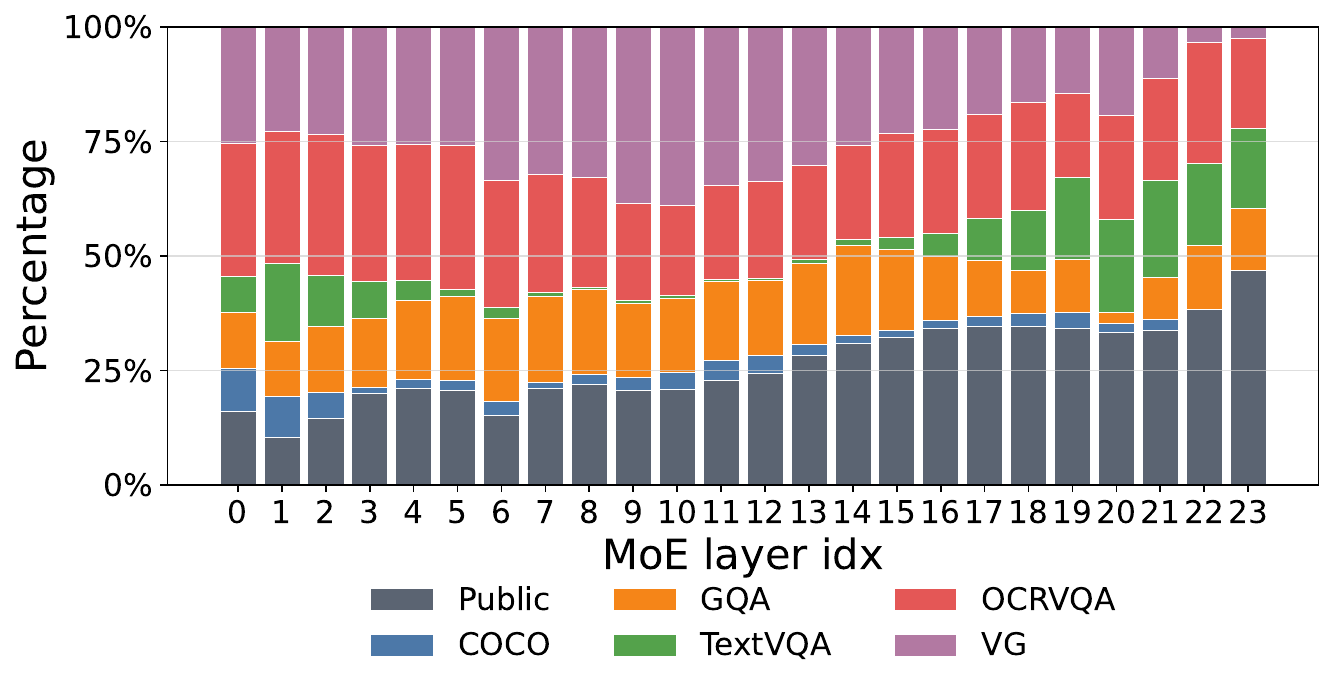}
        \caption{ScienceQA}
        \label{fig:sqa_qwen_distmoe}
    \end{subfigure}

    \vspace{0.5em}

    \begin{subfigure}[t]{0.48\textwidth}
        \centering
        \includegraphics[width=\linewidth]{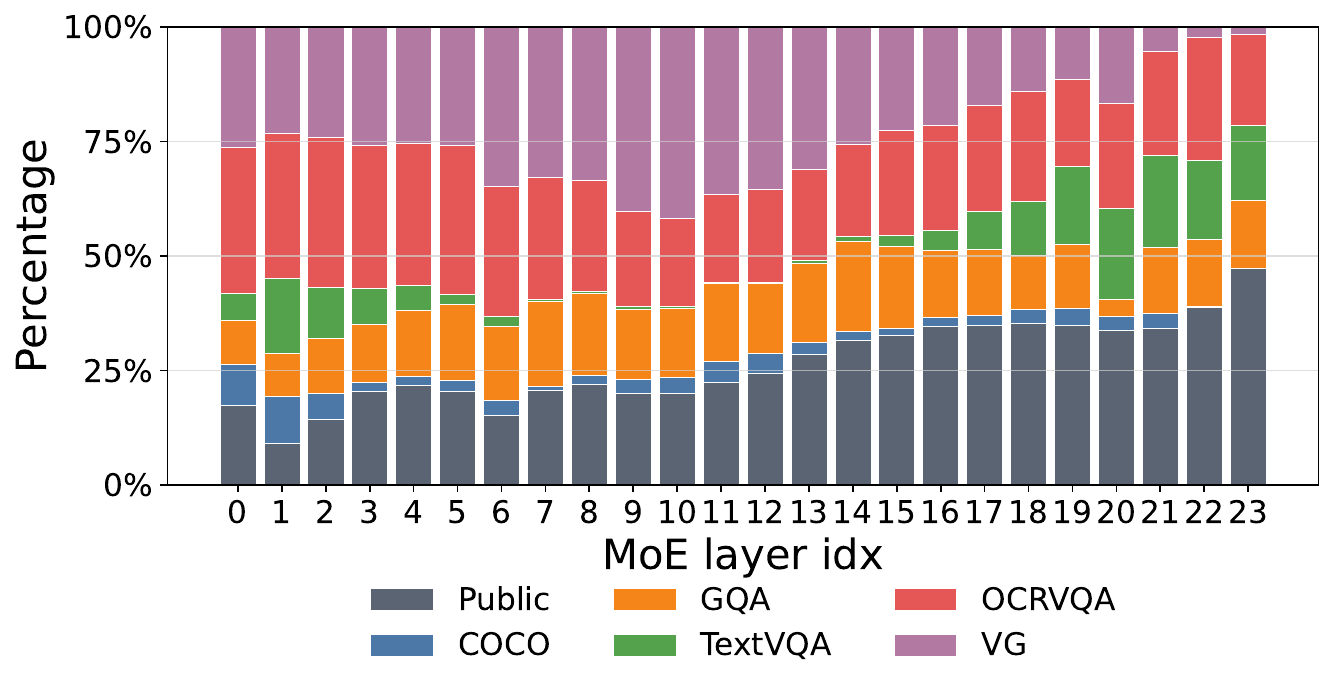}
        \caption{TextVQA}
        \label{fig:textvqa_qwen_distmoe}
    \end{subfigure}
    \hfill
    \begin{subfigure}[t]{0.48\textwidth}
        \centering
        \includegraphics[width=\linewidth]{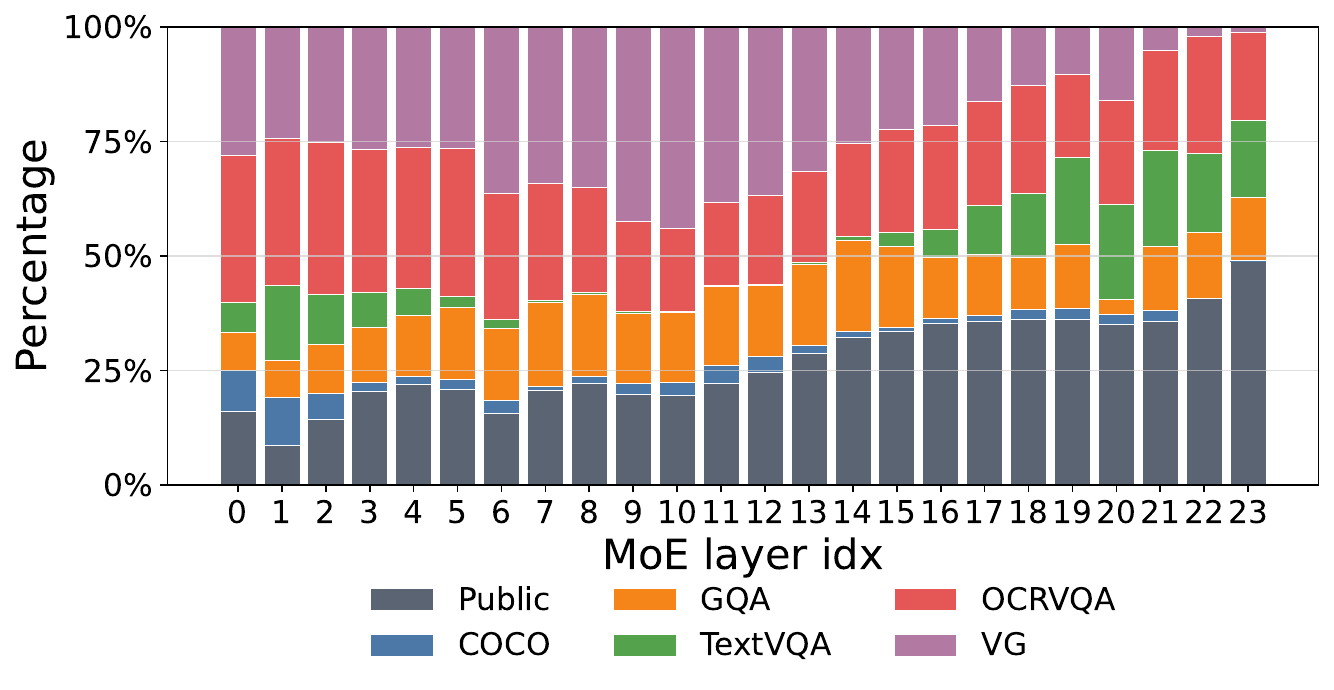}
        \caption{POPE}
        \label{fig:pope_qwen_distmoe}
    \end{subfigure}

    \vspace{0.5em}

    \begin{subfigure}[t]{0.48\textwidth}
        \centering
        \includegraphics[width=\linewidth]{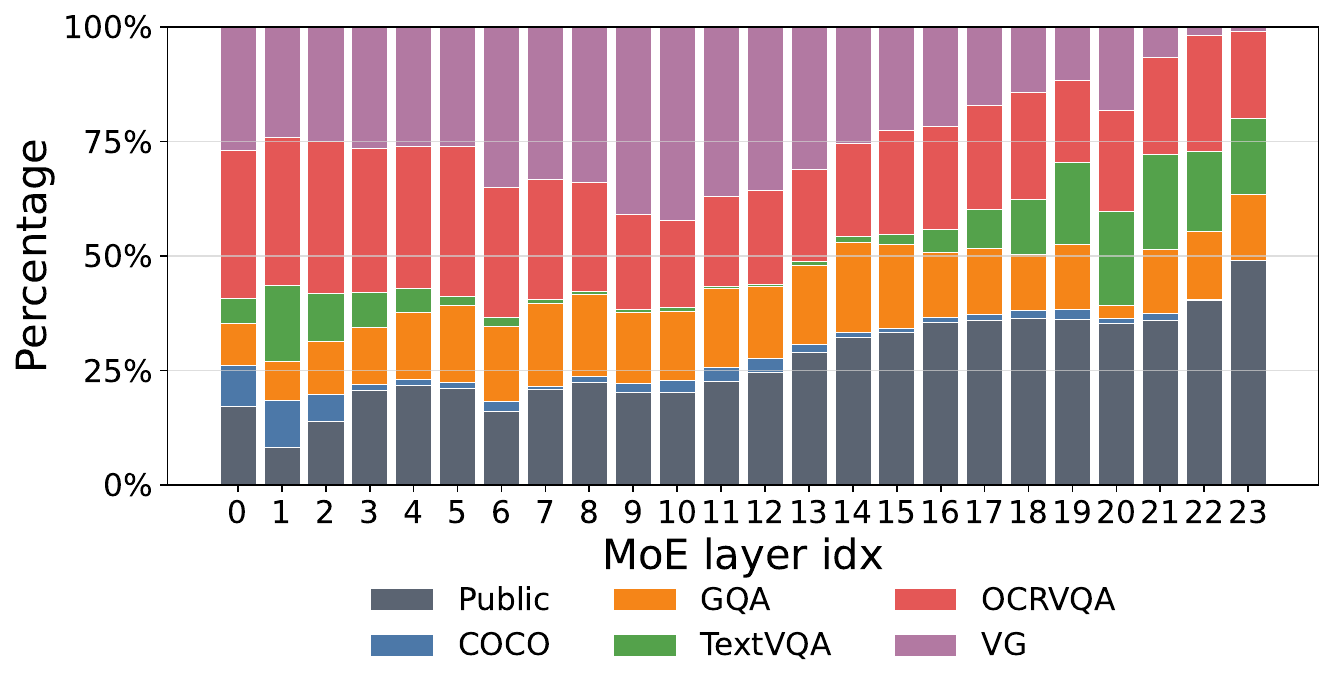}
        \caption{LLaVA-Bench}
        \label{fig:llavabench_qwen_distmoe}
    \end{subfigure}
    \hfill
    \begin{subfigure}[t]{0.48\textwidth}
        \centering
        \includegraphics[width=\linewidth]{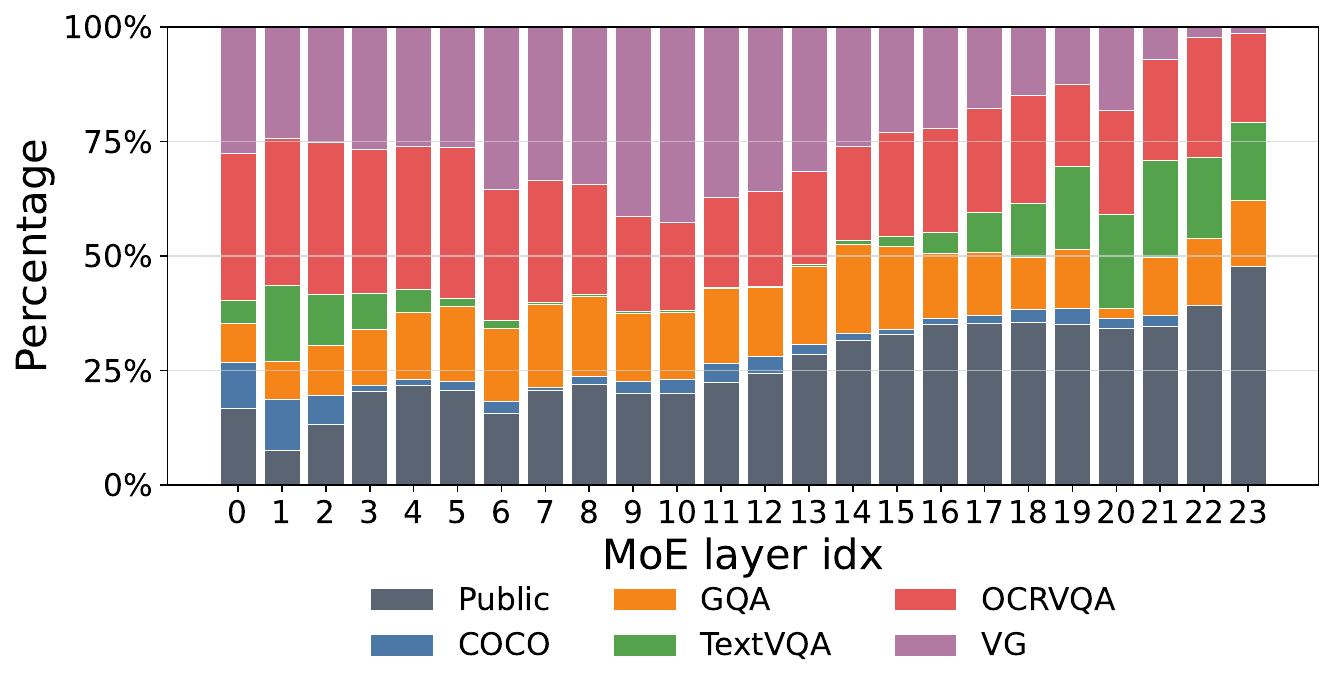}
        \caption{MM-Vet}
        \label{fig:mmvet_qwen_distmoe}
    \end{subfigure}

    \caption{Top-4 routing mass of \method with Qwen-1.8B.}
    \label{fig:routing_qwen}
\end{figure*}

\begin{figure*}[t]
    \centering

    \begin{subfigure}[t]{0.48\textwidth}
        \centering
        \includegraphics[width=\linewidth]{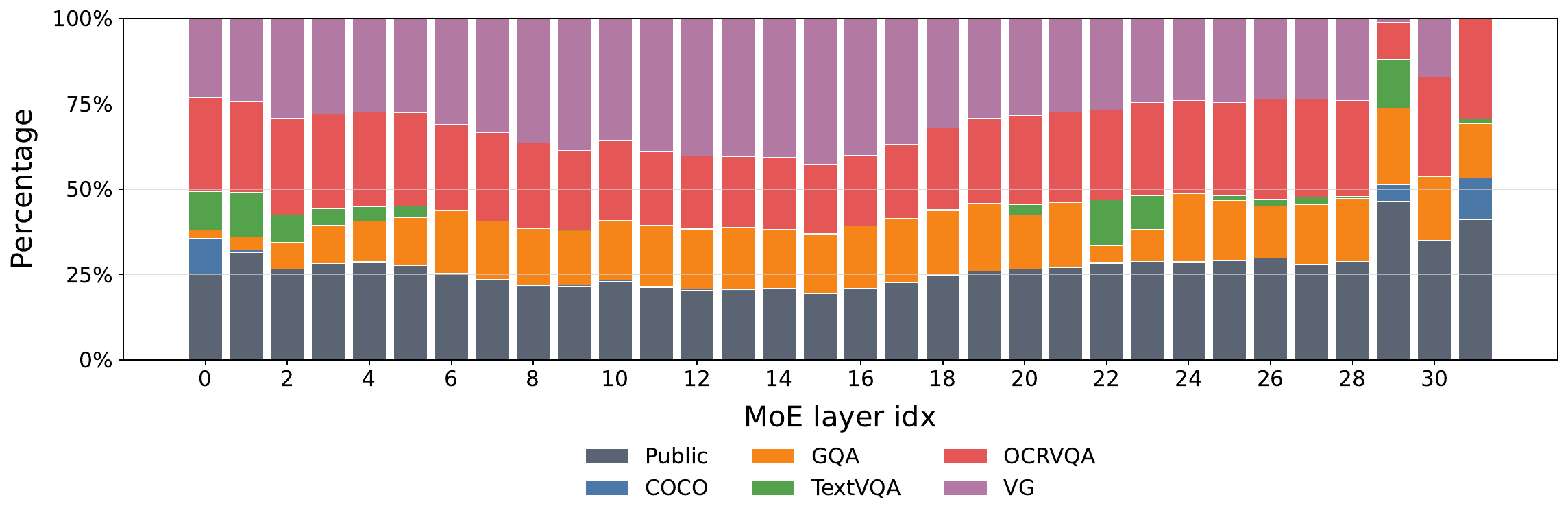}
        \caption{GQA}
        \label{fig:gqa_phi_distmoe}
    \end{subfigure}
    \hfill
    \begin{subfigure}[t]{0.48\textwidth}
        \centering
        \includegraphics[width=\linewidth]{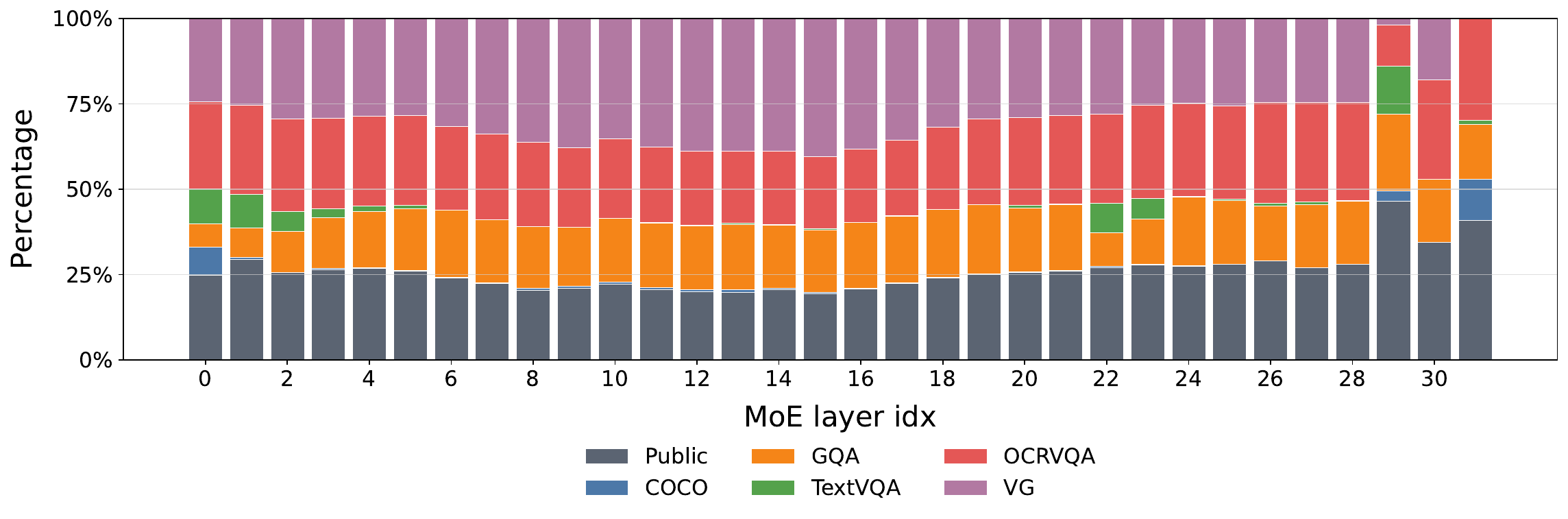}
        \caption{ScienceQA}
        \label{fig:sqa_phi_distmoe}
    \end{subfigure}

    \vspace{0.5em}

    \begin{subfigure}[t]{0.48\textwidth}
        \centering
        \includegraphics[width=\linewidth]{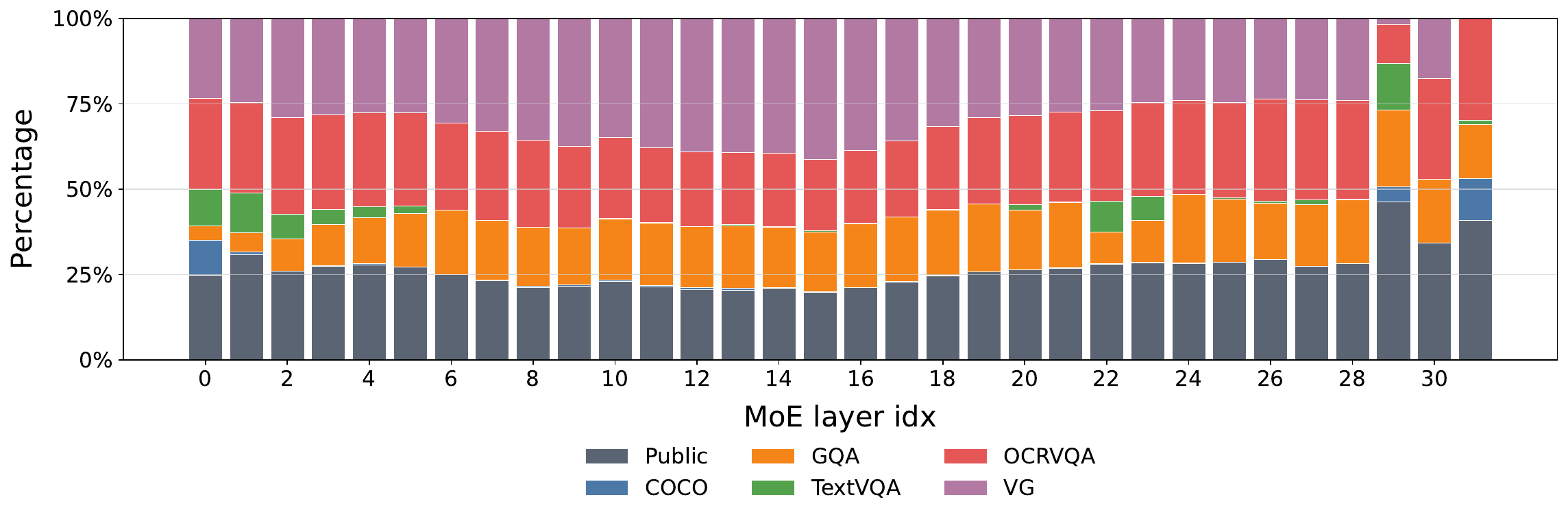}
        \caption{TextVQA}
        \label{fig:textvqa_phi_distmoe}
    \end{subfigure}
    \hfill
    \begin{subfigure}[t]{0.48\textwidth}
        \centering
        \includegraphics[width=\linewidth]{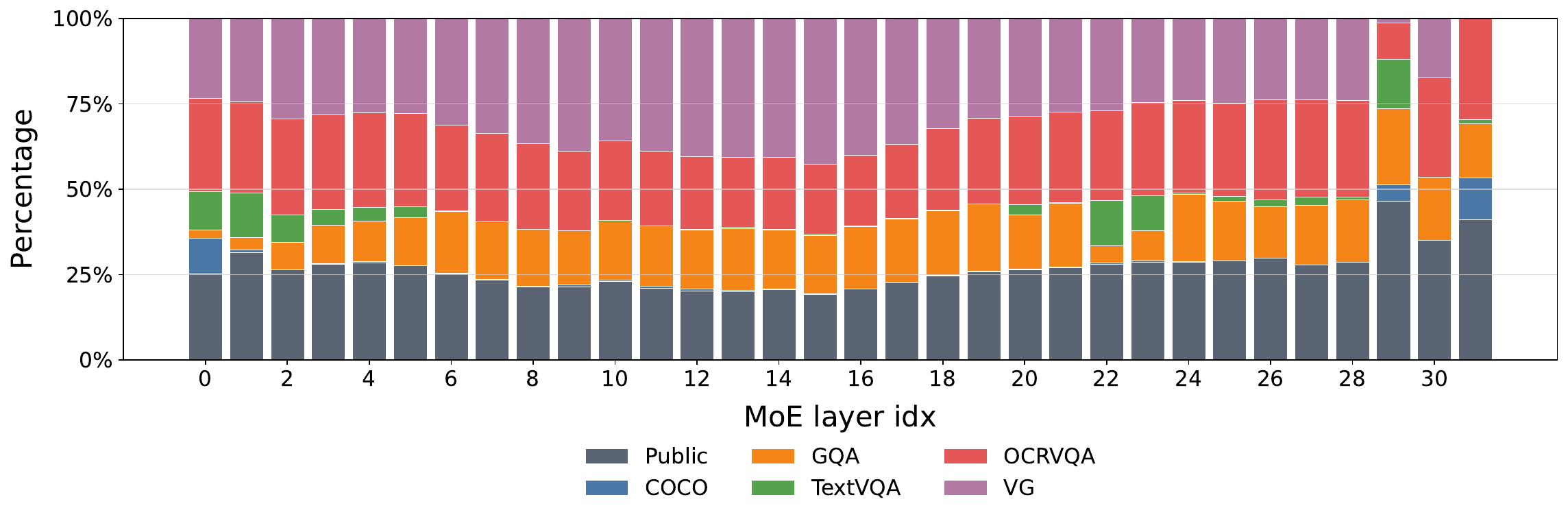}
        \caption{POPE}
        \label{fig:pope_phi_distmoe}
    \end{subfigure}

    \vspace{0.5em}

    \begin{subfigure}[t]{0.48\textwidth}
        \centering
        \includegraphics[width=\linewidth]{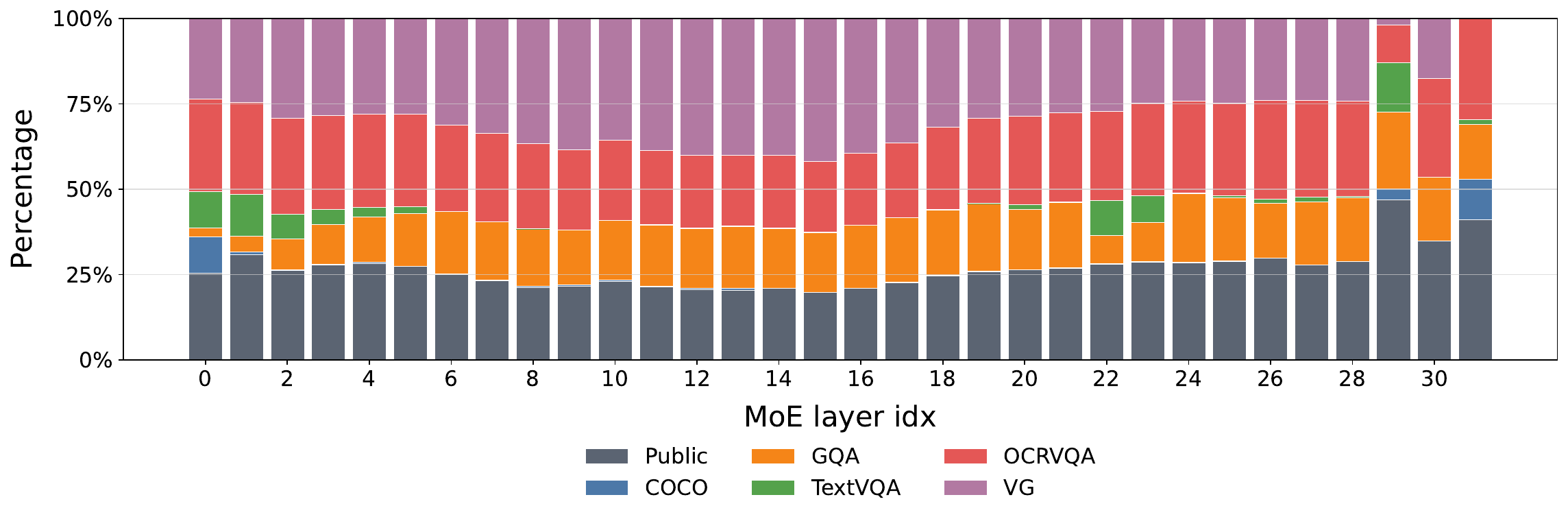}
        \caption{LLaVA-Bench}
        \label{fig:llavabench_phi_distmoe}
    \end{subfigure}
    \hfill
    \begin{subfigure}[t]{0.48\textwidth}
        \centering
        \includegraphics[width=\linewidth]{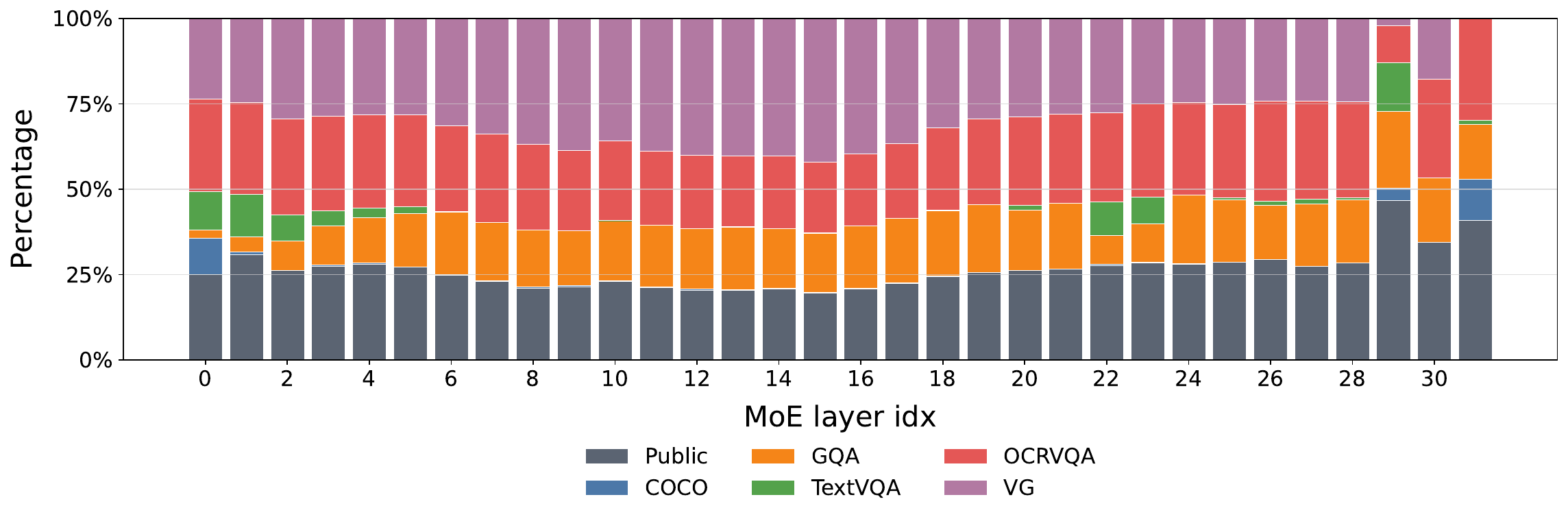}
        \caption{MM-Vet}
        \label{fig:mmvet_phi_distmoe}
    \end{subfigure}

    \caption{Top-4 routing mass of \method with Phi-2.7B.}
    \label{fig:routing_phi}
\end{figure*}

\begin{figure*}[t]
    \centering

    \begin{subfigure}[t]{0.48\textwidth}
        \centering
        \includegraphics[width=\linewidth]{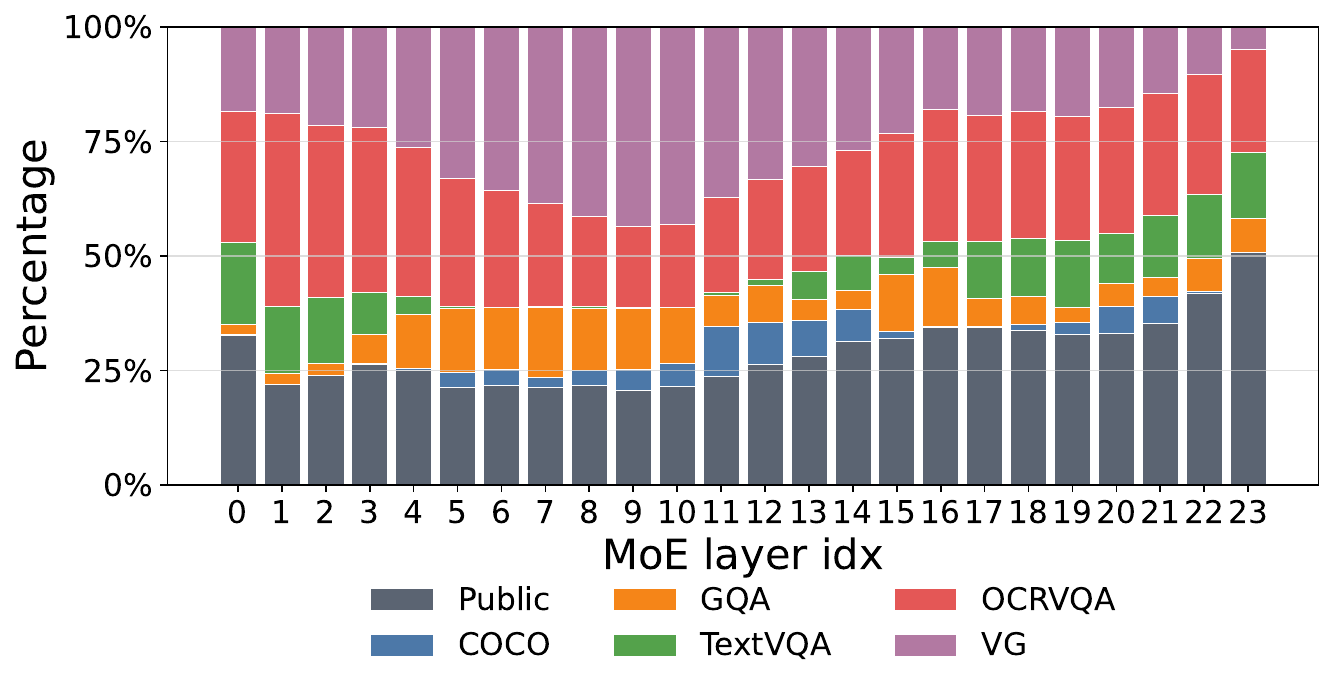}
        \caption{GQA}
        \label{fig:gqa_stablelm_distmoe}
    \end{subfigure}
    \hfill
    \begin{subfigure}[t]{0.48\textwidth}
        \centering
        \includegraphics[width=\linewidth]{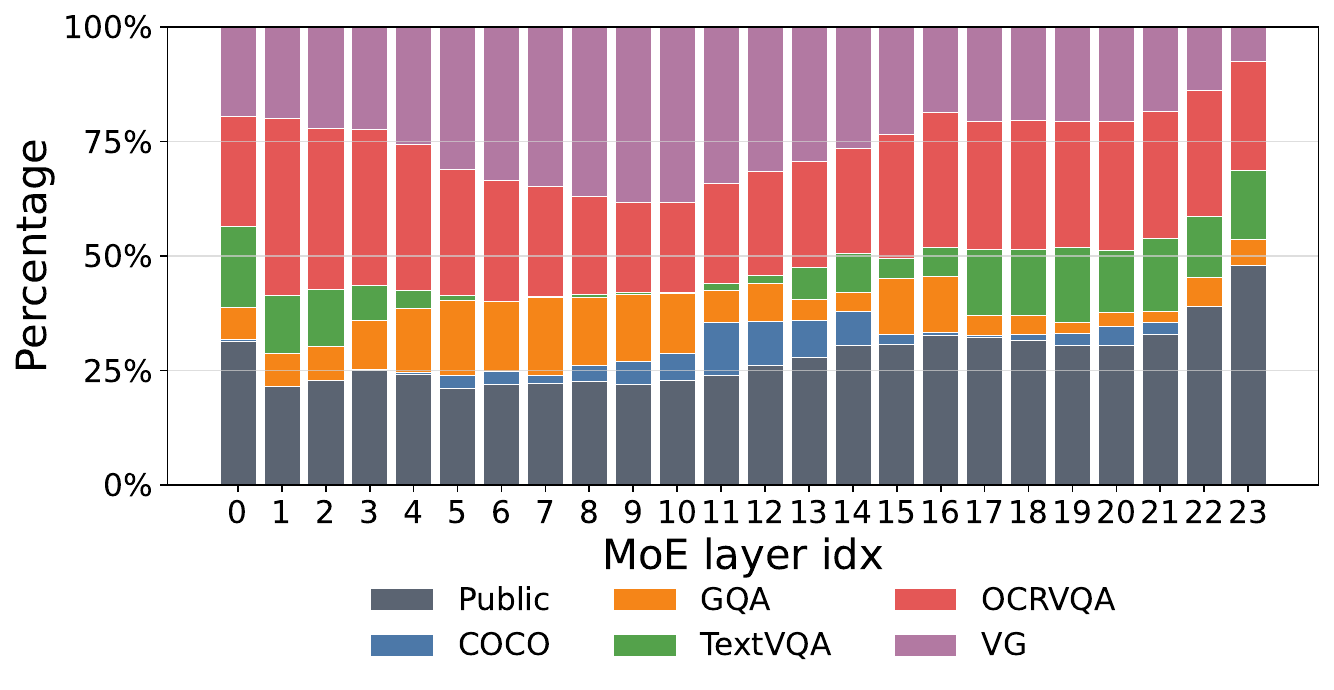}
        \caption{ScienceQA}
        \label{fig:sqa_stablelm_distmoe}
    \end{subfigure}

    \vspace{0.5em}

    \begin{subfigure}[t]{0.48\textwidth}
        \centering
        \includegraphics[width=\linewidth]{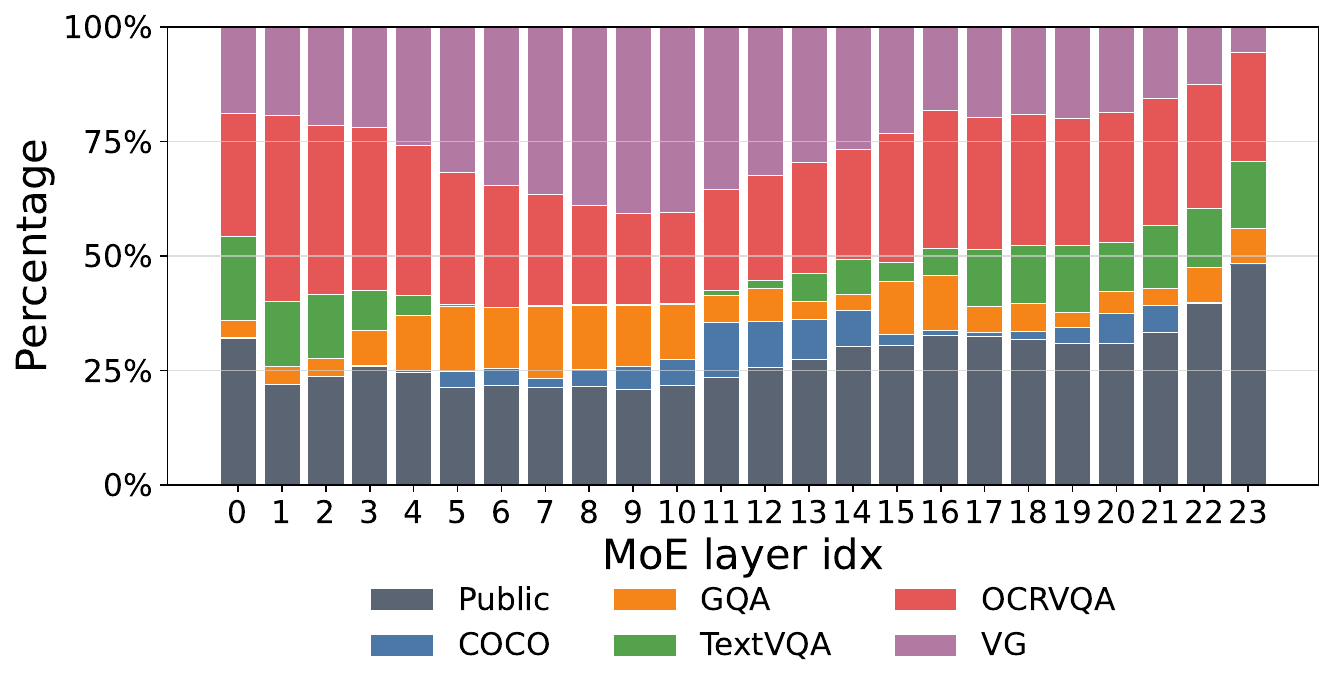}
        \caption{TextVQA}
        \label{fig:textvqa_stablelm_distmoe}
    \end{subfigure}
    \hfill
    \begin{subfigure}[t]{0.48\textwidth}
        \centering
        \includegraphics[width=\linewidth]{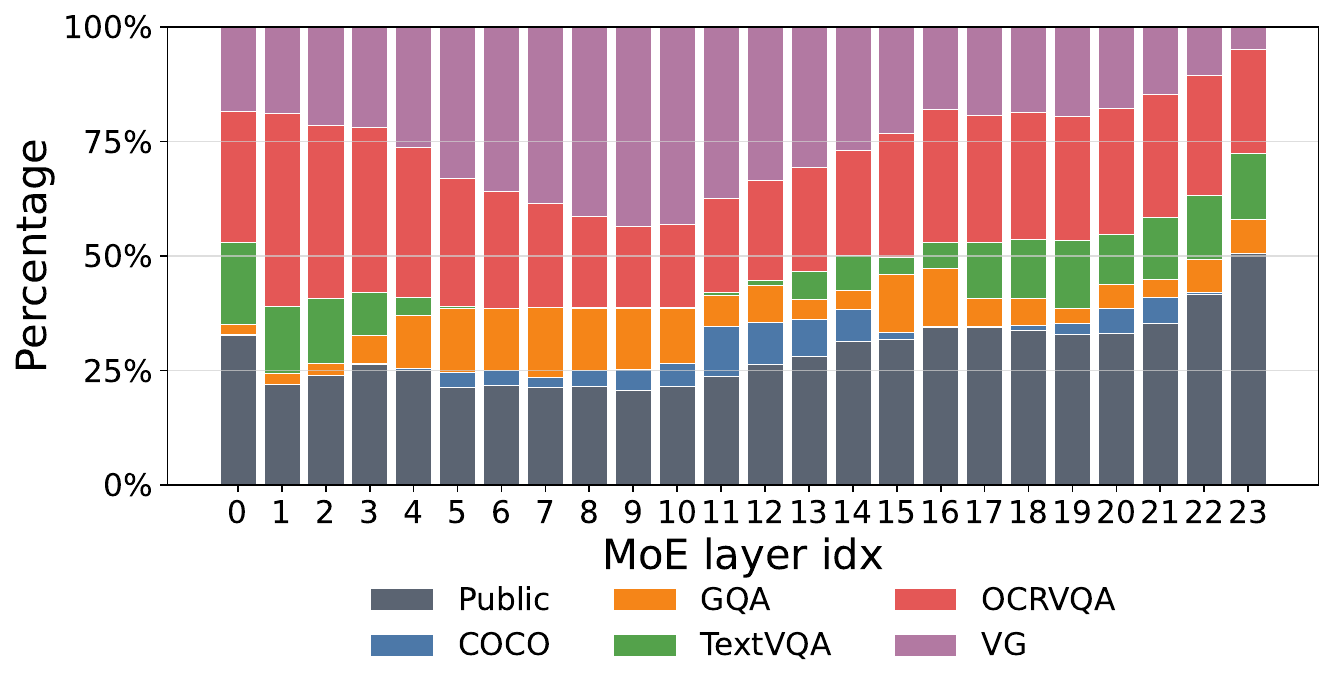}
        \caption{POPE}
        \label{fig:pope_stablelm_distmoe}
    \end{subfigure}

    \vspace{0.5em}

    \begin{subfigure}[t]{0.48\textwidth}
        \centering
        \includegraphics[width=\linewidth]{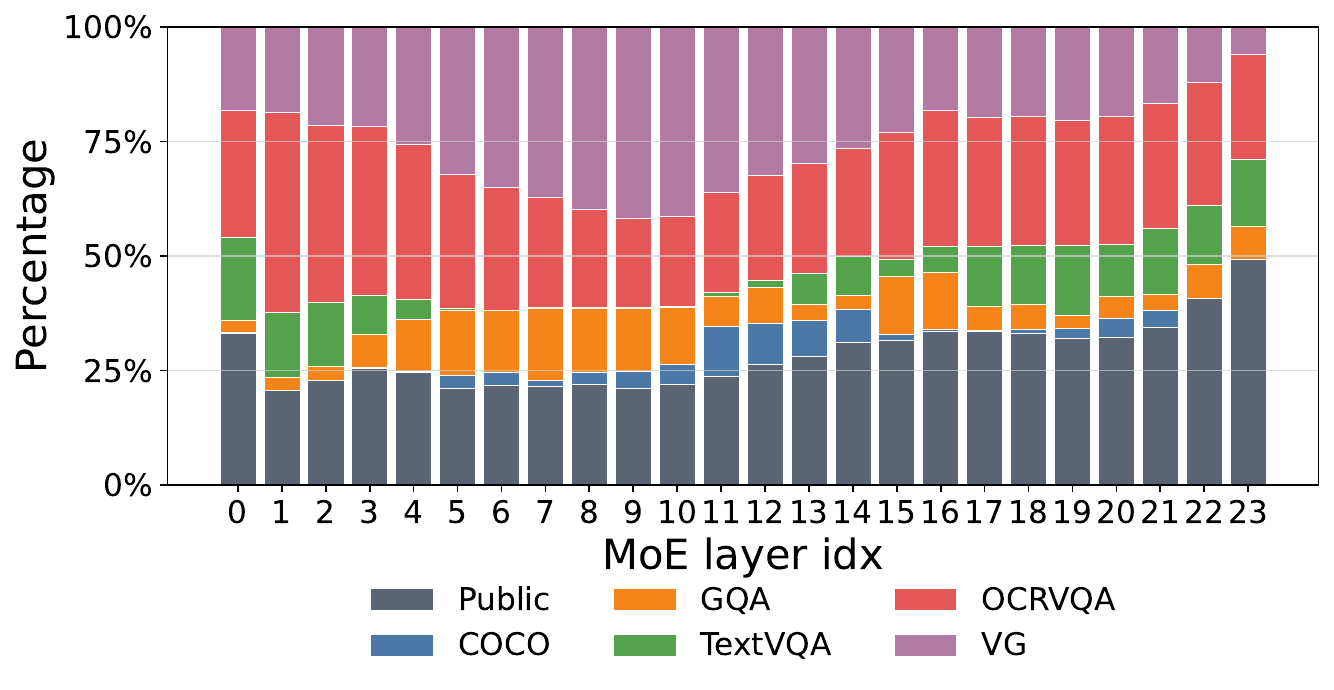}
        \caption{LLaVA-Bench}
        \label{fig:llavabench_stablelm_distmoe}
    \end{subfigure}
    \hfill
    \begin{subfigure}[t]{0.48\textwidth}
        \centering
        \includegraphics[width=\linewidth]{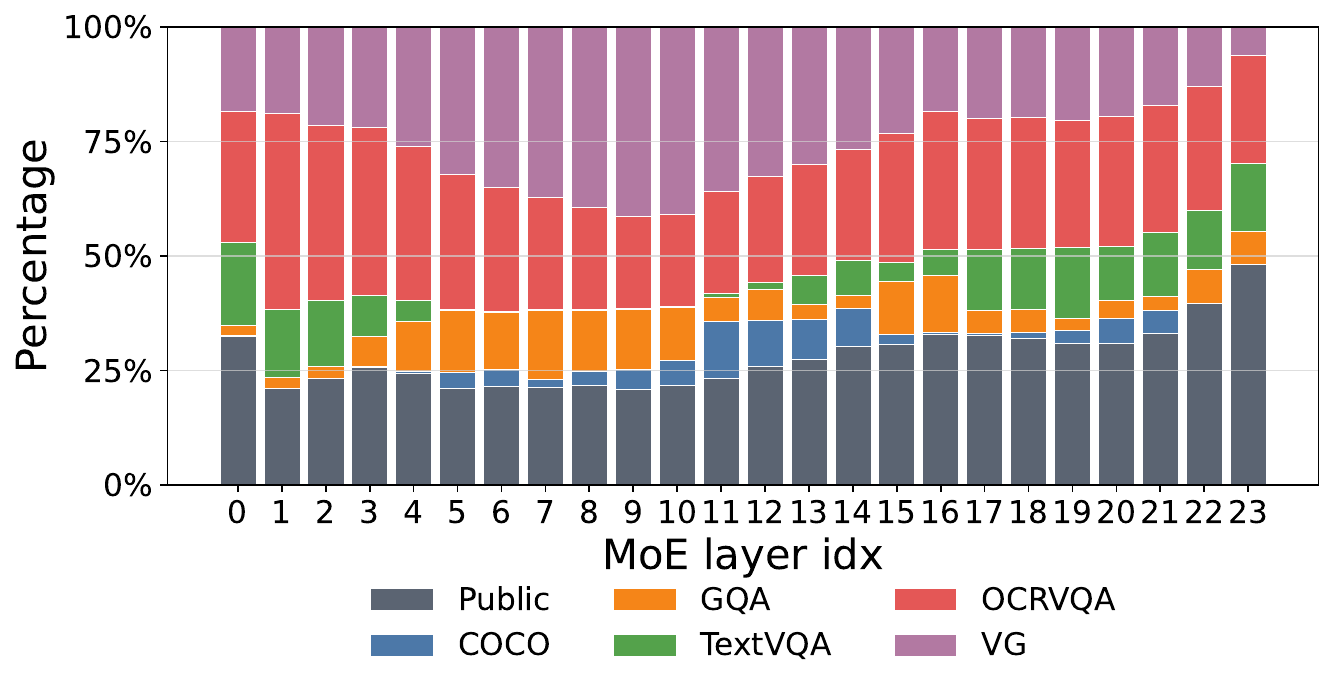}
        \caption{MM-Vet}
        \label{fig:mmvet_stablelm_distmoe}
    \end{subfigure}

    \caption{Top-4 routing mass of \method with StableLM-1.6B.}
    \label{fig:routing_stablelm}
\end{figure*}

\clearpage
\newpage
\section*{NeurIPS Paper Checklist}

\begin{enumerate}

\item {\bf Claims}
    \item[] Question: Do the main claims made in the abstract and introduction accurately reflect the paper's contributions and scope?
    \item[] Answer: \answerYes{} 
    \item[] Justification: We state the claims and contributions in the abstract and introduction sections.

\item {\bf Limitations}
    \item[] Question: Does the paper discuss the limitations of the work performed by the authors?
    \item[] Answer: \answerYes{} 
    \item[] Justification: We discuss the Limitations in Section~\ref{sec:conclusion}.

\item {\bf Theory assumptions and proofs}
    \item[] Question: For each theoretical result, does the paper provide the full set of assumptions and a complete (and correct) proof?
    \item[] Answer: \answerYes{} 
    \item[] Justification: We provide the theoretical justifications in Proposition~\ref{prop:residual_bound} and Appendix~\ref{sec:proof_prop}.

    \item {\bf Experimental result reproducibility}
    \item[] Question: Does the paper fully disclose all the information needed to reproduce the main experimental results of the paper to the extent that it affects the main claims and/or conclusions of the paper (regardless of whether the code and data are provided or not)?
    \item[] Answer: \answerYes{} 
    \item[] Justification: For the details to reproduce the experimental results, please refer to Section~\ref{sec:experiments} and for further details refer Appendix~\ref{sec:additional_experimental_details}, and ~\ref{sec:additional_results}.

\item {\bf Open access to data and code}
    \item[] Question: Does the paper provide open access to the data and code, with sufficient instructions to faithfully reproduce the main experimental results, as described in supplemental material?
    \item[] Answer: \answerYes{} 
    \item[] Justification: Please refer to Section~\ref{sec:experiments} and Appendix~\ref{sec:additional_experimental_details} for the datasets used , all are publicly available. Further code will be available on the project website.

\item {\bf Experimental setting/details}
    \item[] Question: Does the paper specify all the training and test details (e.g., data splits, hyperparameters, how they were chosen, type of optimizer) necessary to understand the results?
    \item[] Answer: \answerYes{} 
    \item[] Justification: Please refer to the Section~\ref{sec:experiments} and Appendix~\ref{sec:additional_experimental_details}.

\item {\bf Experiment statistical significance}
    \item[] Question: Does the paper report error bars suitably and correctly defined or other appropriate information about the statistical significance of the experiments?
    \item[] Answer: \answerYes{} 
    \item[] Justification: Tabs. \ref{tab:seeds} and \ref{tab:liso_seeds} report mean $\pm$ std over three training seeds.

\item {\bf Experiments compute resources}
    \item[] Question: For each experiment, does the paper provide sufficient information on the computer resources (type of compute workers, memory, time of execution) needed to reproduce the experiments?
    \item[] Answer: \answerYes{} 
    \item[] Justification: We use $4$ NVIDIA A100 GPUs for our experiments and discussed the details in Appendix~\ref{sec:implementation_details}.
    
\item {\bf Code of ethics}
    \item[] Question: Does the research conducted in the paper conform, in every respect, with the NeurIPS Code of Ethics \url{https://neurips.cc/public/EthicsGuidelines}?
    \item[] Answer: \answerYes{}{} 
    \item[] Justification: We confirm the research conducted in the paper conform, in every respect, with the NeurIPS Code of Ethics.

\item {\bf Broader impacts}
    \item[] Question: Does the paper discuss both potential positive societal impacts and negative societal impacts of the work performed?
    \item[] Answer: \answerYes{} 
    \item[] Justification: We discuss the impacts of the work in Section~\ref{sec:conclusion}.
    
\item {\bf Safeguards}
    \item[] Question: Does the paper describe safeguards that have been put in place for responsible release of data or models that have a high risk for misuse (e.g., pre-trained language models, image generators, or scraped datasets)?
    \item[] Answer: \answerNA{} 
    \item[] Justification: This work poses no such risks.

\item {\bf Licenses for existing assets}
    \item[] Question: Are the creators or original owners of assets (e.g., code, data, models), used in the paper, properly credited and are the license and terms of use explicitly mentioned and properly respected?
    \item[] Answer: \answerYes{} 
    \item[] Justification: We cite all datasets and models utilized in our experiments.

\item {\bf New assets}
    \item[] Question: Are new assets introduced in the paper well documented and is the documentation provided alongside the assets?
    \item[] Answer: \answerNA{} 
    \item[] Justification: This work does not release new assets.

\item {\bf Crowdsourcing and research with human subjects}
    \item[] Question: For crowdsourcing experiments and research with human subjects, does the paper include the full text of instructions given to participants and screenshots, if applicable, as well as details about compensation (if any)? 
    \item[] Answer: \answerNA{} 
    \item[] Justification: This work does not involve crowdsourcing nor research with human subjects.

\item {\bf Institutional review board (IRB) approvals or equivalent for research with human subjects}
    \item[] Question: Does the paper describe potential risks incurred by study participants, whether such risks were disclosed to the subjects, and whether Institutional Review Board (IRB) approvals (or an equivalent approval/review based on the requirements of your country or institution) were obtained?
    \item[] Answer: \answerNA{} 
    \item[] Justification: This work does not involve crowdsourcing nor research with human subjects.

\item {\bf Declaration of LLM usage}
    \item[] Question: Does the paper describe the usage of LLMs if it is an important, original, or non-standard component of the core methods in this research? Note that if the LLM is used only for writing, editing, or formatting purposes and does \emph{not} impact the core methodology, scientific rigor, or originality of the research, declaration is not required.
    \item[] Answer: \answerYes{} 
    \item[] Justification: LLMs are used as the language backbones of the evaluated MLLMs, including Qwen-1.8B, Phi-2.7B, and StableLM-1.6B. We also use GPT-4 as an LLM judge, requiring an OpenAI API key, for two datasets i.e. LLaVA$^{\mathrm{W}}$ and MM-Vet following the standard benchmark protocols. Please note that GPT-4 is used only for evaluation and not for training, data generation, or method development.

\end{enumerate}

\end{document}